\newif\ifarxiv
\arxivtrue

\ifdefined\useorstyle
\else

\documentclass[11pt]{article}
\usepackage[numbers]{natbib}
\usepackage{fullpage}
\usepackage{hyperref}
\hypersetup{
    colorlinks,
    linkcolor={blue!60!black},
    citecolor={blue!60!black},
    urlcolor={blue!60!black}
}

\usepackage{url}

\usepackage{pgfplotstable}
\usepackage{graphicx}
\usepackage{float}

\usepackage{color}

\definecolor{innerboxcolor}{rgb}{.9,.95,1}
\definecolor{outerlinecolor}{rgb}{.6,0,.2}

\providecommand{\comment}[1]{}

\usepackage{enumerate}
\usepackage{amssymb}
\usepackage{amsbsy}
\usepackage{amsmath}
\usepackage{amsthm}
\usepackage{amsfonts}
\usepackage{latexsym}
\usepackage{graphicx}
\usepackage{xifthen}
\usepackage{xspace}
\usepackage{bbm}
\usepackage{color}

\definecolor{innerboxcolor}{rgb}{.9,.95,1}
\definecolor{outerlinecolor}{rgb}{.6,0,.2}

\newcommand{\mc}[1]{\mathcal{#1}}
\newcommand{\mbf}[1]{\mathbf{#1}}

\newcommand{\norm}[1]{\left\|{#1}\right\|} %
\newcommand{\lone}[1]{\norm{#1}_1} %
\newcommand{\ltwo}[1]{\norm{#1}_2} %
\newcommand{\linf}[1]{\norm{#1}_\infty} %
\newcommand{\dnorm}[1]{\norm{#1}_*} %
\newcommand{\norms}[1]{\|{#1}\|} %
\newcommand{\ltwos}[1]{\norms{#1}_2} %
\newcommand{\opnorm}[1]{\norm{#1}_{\rm op}}

\newcommand{\defeq}{:=}
\newcommand{\eqdef}{=:}
\newcommand{\what}[1]{\widehat{#1}} %

\newcommand{\half}{\frac{1}{2}}
\newcommand{\indic}[1]{\mbf{1}\left\{#1\right\}} %

\newcommand{\R}{\mathbb{R}}
\newcommand{\N}{\mathbb{N}}

\makeatletter
\long\def\@makecaption#1#2{
  \vskip 0.8ex
  \setbox\@tempboxa\hbox{\small {\bf #1:} #2}
  \parindent 1.5em  %
  \dimen0=\hsize
  \advance\dimen0 by -3em
  \ifdim \wd\@tempboxa >\dimen0
  \hbox to \hsize{
    \parindent 0em
    \hfil
    \parbox{\dimen0}{\def\baselinestretch{0.96}\small
      {\bf #1.} #2
    }
    \hfil}
  \else \hbox to \hsize{\hfil \box\@tempboxa \hfil}
  \fi
}
\makeatother

\newcommand{\<}{\left\langle} %
\renewcommand{\>}{\right\rangle}

\newcommand{\openright}[2]{\left[{#1},{#2}\right)} %

\newcommand{\E}{\mathbb{E}} %
\renewcommand{\P}{\mathbb{P}} %
\newcommand{\simiid}{\stackrel{\rm iid}{\sim}}

\newcommand{\cd}{\stackrel{d}{\rightarrow}} %
\providecommand{\argmax}{\mathop{\rm argmax}} %
\providecommand{\argmin}{\mathop{\rm argmin}}

\providecommand{\diag}{\mathop{\rm diag}}

\providecommand{\sign}{\mathop{\rm sign}}

\newcommand{\hinge}[1]{\left[{#1}\right]_+} %

\providecommand{\minimize}{\mathop{\rm minimize}}
\providecommand{\maximize}{\mathop{\rm maximize}}

\newtheorem{theorem}{Theorem}
\newtheorem{corollary}{Corollary}
\newtheorem{proposition}[theorem]{Proposition}

\newtheorem{lemma}{Lemma}

\newtheorem{assumption}{Assumption}

\renewenvironment{proof}{\noindent{\bf Proof}\hspace*{1em}}{\qed\bigskip\\}

\newenvironment{proof-sketch}{\noindent{\bf Sketch of Proof}
  \hspace*{1em}}{\qed\bigskip\\}
\newenvironment{proof-idea}{\noindent{\bf Proof Idea}
  \hspace*{1em}}{\qed\bigskip\\}
\newenvironment{proof-of-claim}{\noindent{\bf Proof of Claim}
  \hspace*{1em}}{\qed\bigskip\\}
\newenvironment{proof-of-lemma}[1][{}]{\noindent{\bf Proof of Lemma {#1}}
  \hspace*{1em}}{\qed\bigskip\\}
\newenvironment{proof-of-proposition}[1][{}]{\noindent{\bf
    Proof of Proposition {#1}}
  \hspace*{1em}}{\qed\bigskip\\}
\newenvironment{proof-of-theorem}[1][{}]{\noindent{\bf Proof of Theorem {#1}}
  \hspace*{1em}}{\qed\bigskip\\}
\newenvironment{inner-proof}{\noindent{\bf Proof}\hspace{1em}}{
  $\bigtriangledown$\medskip\\}
\newenvironment{proof-attempt}{\noindent{\bf Proof Attempt}
  \hspace*{1em}}{\qed\bigskip\\}

\newcounter{example}

\newenvironment{example}[1][]{
  \refstepcounter{example}
  \ifthenelse{\isempty{#1}}{%
    \noindent \textbf{Example \theexample:}\hspace*{.05em}
  }{%
    \noindent \textbf{Example \theexample} ({#1})\textbf{:}\hspace*{.05em}
  }
}{%
  $\diamondsuit$ \bigskip
}

\newenvironment{example*}[1][]{
  \ifthenelse{\isempty{#1}}{%
    \noindent \textbf{Example:}\hspace*{.05em}
  }{%
    \noindent \textbf{Example} ({#1})\textbf{:}\hspace*{.05em}
  }
}{%
  $\diamondsuit$ \bigskip
}

\newcounter{remark}

\newenvironment{remark*}[1][]{
  \ifthenelse{\isempty{#1}}{%
    \noindent \textbf{Remark:}\hspace*{.05em}
  }{%
    \noindent \textbf{Remark} ({#1})\textbf{:}\hspace*{.05em}
  }
}{%
  $\diamondsuit$ \bigskip
}

\newcommand{\loss}{\ell}  %
\newcommand{\tol}{\rho}  %

\newcommand{\statrv}[1][]{%
  \ifthenelse{\isempty{#1}}{%
    Z
  }{%
    Z_{#1}
  }
}

\newcommand{\statval}{z}  %
\newcommand{\statdomain}{\mc{Z}}  %
\newcommand{\normal}{\mathsf{N}}  %

\newcommand{\numclass}{K}

\usepackage{subfigure}
\usepackage{color}
\usepackage{algorithm}
\usepackage{algpseudocode}
\usepackage{tabularx}

\usepackage{manyfoot}

\DeclareNewFootnote{A}
\DeclareNewFootnote{B}

\let\footnoteR\footnoteB
\let\footnote\footnoteA

\usepackage{caption}
\usepackage{graphicx}

\definecolor{matlab_blue}{rgb}{0,    0.4470,   0.7410}
\definecolor{matlab_red}{rgb}{0.8500,    0.3250,    0.0980}
\definecolor{matlab_yellow}{rgb}{0.9290,    0.6940,    0.1250}
\definecolor{matlab_purple}{rgb}{0.4940,    0.1840,    0.5560}
\definecolor{matlab_green}{rgb}{0.4660,    0.6740,    0.1880}

\newcommand{\emp}{{\what{P}_n}}
\newcommand{\lipc}{L_{\mathsf{c}}}

\newcommand{\gradx}{\mathsf{g}_{\mathsf{\theta}}}
\newcommand{\grady}{\mathsf{g}_{\mathsf{z}}}
\newcommand{\lipx}{L_{\mathsf{\theta \theta}}}
\newcommand{\lipy}{L_{\mathsf{zz}}}
\newcommand{\lipxy}{L_{\mathsf{\theta z}}}
\newcommand{\lipyx}{L_{\mathsf{z \theta}}}
\newcommand{\joints}{\Pi}
\newcommand{\proj}{\mathsf{Proj}}
\newcommand{\stepsize}{\alpha}

\newcommand{\fclass}{\mathcal{F}}
\newcommand{\covnum}{N}
\newcommand{\diam}{\mathop{\rm diam}}
\newcommand{\linfstatnorm}[1]{\left\|{#1}\right\|_{L^\infty(\statdomain)}}

\providecommand{\loss}{\ell}

\newcommand{\opt}{^\star}

\makeatletter
\long\def\@makecaption#1#2{
  \vskip 0.8ex
  \setbox\@tempboxa\hbox{\small {\bf #1:} #2}
  \parindent 1.5em  %
  \dimen0=\hsize
  \advance\dimen0 by -3em
  \ifdim \wd\@tempboxa >\dimen0
  \hbox to \hsize{
    \parindent 0em
    \hfil 
    \parbox{\dimen0}{\def\baselinestretch{0.96}\small
      {\bf #1.} #2
    } 
    \hfil}
  \else \hbox to \hsize{\hfil \box\@tempboxa \hfil}
  \fi
}
\makeatother

\makeatletter%
\begin{document}
\abovedisplayskip=8pt plus0pt minus3pt
\belowdisplayskip=8pt plus0pt minus3pt

\begin{center}
  {\LARGE Certifying Some Distributional Robustness \\ \vspace{5pt}with Principled
    Adversarial Training} \\
  \vspace{.5cm} {\Large Aman Sinha\footnoteR{Equal contribution}$^{1}$
    Hongseok Namkoong$^{*2}$ 
    Riccardo Volpi$^{3}$
    John Duchi$^{1,4}$} \\
  \vspace{.2cm}
  Departments of $^1$Electrical Engineering, $^{2}$Management Science \& Engineering,\\$^3$Pattern Analysis \& Computer Vision, and $^{4}$Statistics\\
  {\large $^{1, 2, 4}$Stanford University, $^{3}$Istituto Italiano di Tecnologia} \\
  \vspace{.2cm} \texttt{\{amans,hnamk,jduchi\}@stanford.edu, riccardo.volpi@iit.it}
\end{center}

\begin{abstract}
Neural networks are vulnerable to adversarial examples and researchers have
proposed many heuristic attack and defense mechanisms. We address this problem
through the principled lens of distributionally robust optimization, which
guarantees performance under adversarial input perturbations.  By considering
a Lagrangian penalty formulation of perturbing the underlying data
distribution in a Wasserstein ball, we provide a training procedure that
augments model parameter updates with worst-case perturbations of training
data. For smooth losses, our procedure provably achieves moderate levels of
robustness with little computational or statistical cost relative to empirical
risk minimization. Furthermore, our statistical guarantees allow us to
efficiently certify robustness for the population loss. For imperceptible
perturbations, our method matches or outperforms heuristic approaches.

 \end{abstract}

\fi

\section{Introduction}\label{sec:intro}

Consider the classical stochastic optimization problem, in which we minimize an expected loss $\E_{P_0}[\loss(\theta; Z)]$
over a parameter $\theta \in \Theta$, where $Z \sim P_0$, $P_0$ is a
distribution on a space $\mathcal{Z}$, and $\loss$ is a loss function. In many
systems, robustness to changes in the data-generating distribution $P_0$ is
desirable, whether they be from covariate shifts, changes in the underlying
domain~\citep{Ben-DavidBlCrPeVa10}, or adversarial
attacks~\citep{ASgoodfellow2015explaining,ASkurakin2016adversarial}. As deep
networks become prevalent in modern performance-critical systems---prominent
examples include perception systems for self-driving cars, and automated
detection of tumors---model failure is increasingly costly. In these
situations, it is irresponsible to deploy models whose robustness and failure
modes we do not understand or cannot certify.

Recent work shows that neural networks are vulnerable to adversarial examples;
seemingly imperceptible perturbations to data can lead to misbehavior of the
model, such as misclassification of the
output~\citep{ASgoodfellow2015explaining, NguyenYoCl15,
  ASkurakin2016adversarial, Moosavi-DezfooliFaFr16}. Consequently, researchers
have proposed adversarial attack and defense
mechanisms~\citep{PapernotMcGoJhCeSw16, ASpapernot2016limitations,
  PapernotMcWuJhSw16, RozsaGuBo16, AScarlini2017towards, HeWeChCaSo17,
  MadryMaScTsVl17, TramerKuPa17}. These works provide an initial foundation
for adversarial training, but it is challenging to rigorously identify the
classes of attacks against which they can defend (or if they
exist). Alternative approaches that provide formal verification of deep
networks~\citep{HuangKwWaWu17, ASkatz2017reluplex, KatzBaDiJuKo17} are NP-hard
in general; they require prohibitive computational expense even on small
networks. Recently, researchers have proposed convex relaxations of the
NP-hard verification problem with some
success~\citep{KolterWo17,RagunathanStLi18}, though they may be difficult to
scale to large networks. Our work is situated between these agendas: we
develop efficient procedures with rigorous guarantees for \emph{small to
  moderate} amounts of robustness.

We take the perspective of distributionally robust
optimization and provide an adversarial training procedure with provable
guarantees on its computational and statistical performance. 
Postulating a class $\mathcal{P}$ of distributions around the
data-generating distribution $P_0$, we consider 
\begin{equation}
  \label{eqn:OG-problem}
  \minimize_{\theta \in \Theta}
  \sup_{P \in \mathcal{P}} \E_P[\loss(\theta; Z)].
\end{equation}
The choice of $\mathcal{P}$ influences robustness guarantees and computability; we
develop robustness sets $\mathcal{P}$ with computationally efficient relaxations
that apply even when the loss $\loss$ is non-convex. We provide an adversarial
training procedure that, for smooth $\loss$, enjoys convergence guarantees
similar to non-robust approaches while \emph{certifying} performance even for
the worst-case population loss $\sup_{P \in \mathcal{P}} \E_{P}[\loss(\theta;
Z)]$. On a simple implementation in Tensorflow, our method takes
$5$--$10\times$ as long as stochastic gradient methods for empirical risk
minimization (ERM), matching runtimes for other adversarial training
procedures~\citep{ASgoodfellow2015explaining, ASkurakin2016adversarial,
  MadryMaScTsVl17}. We show that our procedure---which learns to protect
against adversarial perturbations in the training dataset---generalizes,
allowing us to train a model that prevents attacks to the test dataset.

We briefly overview our approach. Let
$c: \mathcal{Z} \times \mathcal{Z} \to \R_+ \cup \{\infty\}$, where $c(z, z_0)$
is the ``cost'' for an adversary to perturb $z_0$ to $z$ (we typically use $c(z, z_0) = \norm{z - z_0}_p^2$ with $p \ge 1$).  We consider the robustness region $\mathcal{P} = \{ P: W_c(P, P_0) \le \tol\}$, a $\tol$-neighborhood of the distribution $P_0$ under the Wasserstein metric $W_c(\cdot, \cdot)$
(see Section \ref{sec:approach} for a formal definition). For deep networks and other complex models, this formulation of
problem~\eqref{eqn:OG-problem} is intractable with arbitrary $\rho$. Instead, we consider its Lagrangian relaxation for a fixed penalty parameter $\gamma \ge 0$, resulting in the reformulation
\begin{subequations}\label{eqn:lagrangian-duality}
  \begin{equation}
    \minimize_{\theta \in \Theta}
    \left\{ F(\theta) \defeq
      \sup_{P} \left \{ \E_P[\loss(\theta; Z) ] - \gamma W_c(P, P_0)
        \right\}
        = \E_{P_0}[\phi_\gamma(\theta; Z)]
        \right\}
  \end{equation}
  \begin{equation}
    \label{eqn:inner-sup}
    ~~\mbox{where}~~
    \phi_\gamma(\theta; z_0)
    \defeq \sup_{z \in \mathcal{Z}}
    \left\{\loss(\theta; z) - \gamma c(z, z_0) \right\}.
  \end{equation}
\end{subequations}
(See Proposition~\ref{prop:duality} for a rigorous statement of these
equalities.) Here, we replaced the usual loss $\loss(\theta; Z)$ by the robust
surrogate $\phi_\gamma(\theta; Z)$; this surrogate~\eqref{eqn:inner-sup}
allows adversarial perturbations of the data $z$, modulated by the penalty
$\gamma$. As $P_0$ is unknown, we solve the penalty
problem~\eqref{eqn:lagrangian-duality} with $P_0$ replaced by the empirical
distribution $\emp$; we refer to this as the penalty problem below.

The key feature of the penalty problem~\eqref{eqn:lagrangian-duality} is that
moderate levels of robustness---in particular, defense against imperceptible
adversarial perturbations---are achievable at essentially no computational or
statistical cost for \emph{smooth losses $\loss$}. Specifically, for large
enough penalty $\gamma$ (by duality, small enough robustness $\rho$), the
function $z \mapsto \loss(\theta; z) - \gamma c(z, z_0)$ in the robust
surrogate~\eqref{eqn:inner-sup} is strongly concave and hence easy to
optimize if $\loss(\theta, z)$ is smooth in $z$. Consequently, stochastic
gradient methods applied to problem~\eqref{eqn:lagrangian-duality} have
similar convergence guarantees as for non-robust methods (ERM). In
Section~\ref{sec:generalization-robustness}, we provide a \emph{certificate of
  robustness} for any $\tol$; we give an efficiently computable data-dependent
upper bound on the worst-case loss
$\sup_{P: W_c(P, P_0) \le \tol} \E_P[\loss(\theta; Z)]$. That is, the
worst-case performance of the output of our principled adversarial training
procedure is guaranteed to be no worse than this certificate.  Our bound is
tight when $\tol = \what{\tol}_n$, the achieved
robustness for the empirical objective. These results suggest advantages of
networks with smooth activations rather than ReLUs. In
Section~\ref{section:examples}, we subtantiate our optimization guarantees
(Section~\ref{sec:approach}) and certificate of robustness
(Section~\ref{sec:generalization-robustness}) by providing concrete bounds on
the smoothness levels of neural networks.  We experimentally verify our
results in Section~\ref{sec:experiments} and show that we match or achieve
state-of-the-art performance on a variety of adversarial attacks.

Compared to the conference version of our manuscript~\citep{SinhaNaDu18}, we
have made substantial progress in our theoretical and empirical
development. First, we provide additional theoretical results that allow
instantiating our previous abstract computational and statistical guarantees
in concrete learning scenarios involving neural networks
(Section~\ref{section:examples}). Secondly, we conduct extensive experiments
to provide i) evaluations of adversarial training methods on more realistic
large-scale classification scenarios (Sections~\ref{section:dogs}) than what
the current literature provides and ii) an empirical analysis of the gap
between our theoretical bounds and empirical performance under various
adversarial attacks (Section~\ref{section:three-choose-two}).

\paragraph{Robust optimization and adversarial training}
The standard robust-optimization approach minimizes worst-case losses of the
form $\sup_{u \in \mathcal{U}} \loss(\theta; z + u)$ for some uncertainty set
$\mathcal{U}$~\citep{RatliffBaZi06, Ben-TalGhNe09, XuCaMa09}.  Unfortunately,
this approach is intractable except for specially structured losses, such as
the composition of a linear and simple convex function~\citep{Ben-TalGhNe09,
  XuCaMa09,XuCaMa12}. Nevertheless, this robust approach underlies recent
advances in adversarial training~\citep{SzegedyZaSuBrErGoFe14,
  ASgoodfellow2015explaining, ASpapernot2016limitations,
  MadryMaScTsVl17}, which heuristically perturb data during a stochastic
optimization procedure.

One such heuristic uses a locally linearized loss function (proposed with
$p=\infty$ as the ``fast gradient sign method''
\citep{ASgoodfellow2015explaining}):
\begin{equation}\label{eqn:fgsm}
  \Delta_{x_i}(\theta) := 
  \argmax_{\norm{\eta}_p \le \epsilon}
  \{\nabla_x \loss(\theta; (x_i, y_i))^T\eta\}
  ~~ \mbox{and perturb} ~~
  x_i \to x_i + \Delta_{x_i}(\theta).
\end{equation}
One form of adversarial training trains on the losses
$\loss(\theta; (x_i + \Delta_{x_i}(\theta), y_i))$
\citep{ ASgoodfellow2015explaining, ASkurakin2016adversarial}, while
others perform iterated variants \citep{ASpapernot2016limitations,
  AScarlini2017towards, MadryMaScTsVl17, TramerKuPa17}.
\citet{MadryMaScTsVl17} observe that these procedures attempt to
optimize the objective $\E_{P_0}[\sup_{\norm{u}_p \le \epsilon}
  \loss(\theta; Z + u)]$, a constrained version of the penalty
problem~\eqref{eqn:lagrangian-duality}. This notion of robustness is
typically intractable: the inner supremum is generally non-concave in $u$,
so it is unclear whether model-fitting with these techniques converges, and
there are possibly worst-case perturbations these techniques do not
find. Indeed, it is NP-hard to find worst-case perturbations when deep networks use ReLU activations, suggesting difficulties for fast and iterated
heuristics (see Lemma~\ref{thm:nphard} in Section
\ref{sec:nphard}). Smoothness, which can be obtained in standard deep
architectures with exponential linear units (ELUs)
\citep{ASclevert2015fast}, allows us to find Lagrangian worst-case
perturbations with low computational cost.

\paragraph{Distributionally robust optimization}
To situate the current work, we review some of the substantial body of work on
robustness and learning. The choice of $\mathcal{P}$ in the robust
objective~\eqref{eqn:OG-problem} affects both the richness of the uncertainty
set we wish to consider as well as the tractability of the resulting
optimization problem. Previous approaches to distributional robustness have
considered finite-dimensional parametrizations for $\mathcal{P}$, such as
constraint sets for moments, support, or directional
deviations~\citep{ASchen2007robust, DelageYe10, ASgoh2010distributionally}, as
well as non-parametric distances for probability measures such as
$f$-divergences~\citep{Ben-TalHeWaMeRe13, BertsimasGuKa13, LamZh15,
  MiyatoMaKoNaIs15, DuchiGlNa16, NamkoongDu16},
and Wasserstein distances~\citep{EsfahaniKu15,
  Shafieezadeh-AbadehEsKu15,ASblanchet2016robust, GaoKl16, BlanchetKaZhMu17, GaoChKl17,
  KuhnEsNgSh19}. In constrast to $f$-divergences (e.g.\ $\chi^2$- or
Kullback-Leibler divergences) which are effective when the support of the
distribution $P_0$ is fixed, a Wasserstein ball around $P_0$ includes
distributions $Q$ with different support and allows (in a sense) robustness to
unseen data.

Many authors have studied tractable classes of uncertainty sets $\mathcal{P}$
and losses $\loss$. For example, \citet{Ben-TalHeWaMeRe13}, \citet{Lam18} and
\citet{NamkoongDu17} use convex optimization approaches for $f$-divergence
balls. For worst-case regions $\mathcal{P}$ formed by Wasserstein balls,
\citet{EsfahaniKu15, Shafieezadeh-AbadehEsKu15, ASblanchet2016robust,
  KuhnEsNgSh19} propose tractable approaches for solving the saddle-point
problem~\eqref{eqn:OG-problem}, such as converting it into a regularized ERM
problem; but this is possible only for a limited class of convex losses
$\loss$ and costs $c$. As we are interested in machine learning applications,
we treat a larger class of losses and costs and provide direct solution
methods for a Lagrangian relaxation of the saddle-point
problem~\eqref{eqn:OG-problem}.

One natural application is domain adaptation~\citep{LeeRa17};
\citeauthor{LeeRa17} provide guarantees similar to ours for the empirical
minimizer of the robust saddle-point problem~\eqref{eqn:OG-problem} and give
specialized bounds for domain adaptation problems. Their bounds rely on
concentration of the empirical distribution to its population counterpart in
Wasserstein distance, which may be prohibitively slow even in moderate
dimensional problems. In contrast, our statistical guarantees have the usual
dependence on the dimension based on covering numbers, and we develop
efficient optimization procedures for our distributionally robust
approach. Through an extensive set of experiments, we  provide
empirical evidence that our algorithm defends against imperceptible
adversarial perturbations. Since the conference version of this manuscript was
submitted,~\citet{BlanchetMuZh18} studied efficient solution methods for the
worst case problem~\eqref{eqn:OG-problem} with
$\mathcal{P} = \{P: W_c(P, P_0) \le \tol\}$ for affine models,
and~\citet{GaoChKl17, VolpiNaDuMuSa18} studied connections between robust
formulations and novel regularization schemes that regularize by
$\norm{\nabla_z \loss(\theta; Z)}$.

\section{Proposed approach}
\label{sec:approach}

\newcommand{\couplings}{\Pi}

Our approach is based on the following simple insight: assume that the
function $z \mapsto \loss(\theta; z)$ is smooth, meaning there is some $L$ for
which $\nabla_z \loss(\theta; \cdot)$ is $L$-Lipschitz.  Then for any
$c : \mathcal{Z} \times \mathcal{Z} \to \R_+ \cup \{\infty\}$ $1$-strongly convex in its
first argument, a Taylor expansion yields
\begin{equation}
  \loss(\theta; z') - \gamma c(z', z_0)  \le \loss(\theta; z) - \gamma c(z, z_0)
  + \<\nabla_z (\loss(\theta; z) - \gamma c(z, z_0)),
  z' - z\>
  + \frac{L - \gamma}{2} \ltwo{z - z'}^2.
  \label{eqn:trivial-concavity}
\end{equation}
For $\gamma \ge L$ this is the first-order condition for
$(\gamma - L)$-strong concavity of
$z \mapsto (\loss(\theta; z) - \gamma c(z, z_0))$.  Thus, whenever the loss is
smooth enough in $z$ and the penalty $\gamma$ is large enough (corresponding
to less robustness), computing the surrogate~\eqref{eqn:inner-sup} is a
strongly-concave optimization problem.

We leverage the insight~\eqref{eqn:trivial-concavity} to show that as long
as we do not require \emph{too much} robustness, this strong concavity
approach~\eqref{eqn:trivial-concavity} provides a computationally efficient
and principled approach for robust optimization problems~\eqref{eqn:OG-problem}.
Our starting point is a duality result for the minimax
problem~\eqref{eqn:OG-problem} and its Lagrangian relaxation for
Wasserstein-based uncertainty sets, which makes the connections between
distributional robustness and the ``lazy'' surrogate~\eqref{eqn:inner-sup}
clear. We then show (Section~\ref{sec:optimization}) how stochastic gradient
descent methods can efficiently find minimizers (in the convex case) or
approximate stationary points (when $\loss$ is non-convex) for our relaxed
robust problems.

\paragraph{Wasserstein robustness and duality}
Wasserstein distances define a notion of closeness between distributions. Let
$\mathcal{Z} \subset \R^m$, and let $(\mathcal{Z}, \mathcal{A}, P_0)$ be a probability
space. Let the transportation cost
$c: \mathcal Z \times \mathcal Z \to \openright{0}{\infty}$ be nonnegative, lower
semi-continuous, and satisfy $c(z, z) = 0$.  For example, for a differentiable
convex $h : \mathcal{Z} \to \R$, the Bregman divergence
$c(z, z_0) = h(z) - h(z_0) - \<\nabla h(z_0), z - z_0\>$ satisfies these
conditions.  For probability measures $P$ and $Q$ supported on $\mathcal{Z}$, let
$\couplings(P, Q)$ denote their couplings, meaning measures $M$ on $\mathcal{Z}^2$
with $M(A, \mathcal{Z}) = P(A)$ and $M(\mathcal{Z}, A) = Q(A)$. The Wasserstein distance
between $P$ and $Q$ is
\begin{equation*}
  W_c(P,Q) \defeq \inf_{M \in \couplings(P,Q)}\E_M[c(Z,Z')].
\end{equation*}
For $\tol \ge 0$ and distribution $P_0$, we let $\mathcal{P} = \{ P: W_c(P, P_0)
\le \tol \}$, considering the Wasserstein form of the robust
problem~\eqref{eqn:OG-problem} and its Lagrangian
relaxation~\eqref{eqn:lagrangian-duality} with $\gamma \ge 0$. The following
duality result~\citep{BlanchetMu16, GaoKl16} gives the
equality~\eqref{eqn:lagrangian-duality} for the relaxation and an analogous
result for the problem~\eqref{eqn:OG-problem}.  We give an alternative proof
in Appendix~\ref{sec:proof-duality} for convex, continuous cost functions.
\begin{proposition}
  \label{prop:duality}
  Let $\loss : \Theta \times \mathcal{Z} \to \R$ and $c : \mathcal{Z}
  \times \mathcal{Z} \to \R_+$ be continuous.
  Let $\phi_\gamma(\theta; z_0) = \sup_{z \in \mathcal{Z}} \left\{\loss(\theta;
  z) - \gamma c(z, z_0) \right\}$ be the robust
  surrogate~\eqref{eqn:inner-sup}. For any distribution $Q$ and
  any $\rho > 0$,
  \begin{equation}
    \label{eqn:constrained}
    \sup_{P : W_c(P, Q) \le \rho} \E_P[\loss(\theta; Z)]
    = \inf_{\gamma \ge 0} \big\{ \gamma \rho
      + \E_{Q}[\phi_{\gamma}(\theta; Z)] \big\},
  \end{equation}
  and for any $\gamma \ge 0$, we have
  \begin{equation}
    \label{eqn:lagrangian}
    \sup_{P} \left \{ \E_P[\loss(\theta; Z) ] - \gamma W_c(P, Q) \right \}
    = \E_Q [\phi_{\gamma}(\theta; Z)].
  \end{equation}
\end{proposition}
\noindent 
Leveraging the insight~\eqref{eqn:trivial-concavity}, we give up the
requirement that we wish a prescribed amount $\tol$ of robustness (solving the
worst-case problem~\eqref{eqn:OG-problem} for
$\mathcal{P} = \{P: W_c(P, P_0) \le \tol \}$) and focus instead on the Lagrangian
penalty problem~\eqref{eqn:lagrangian-duality} and its empirical counterpart
\begin{equation}
  \label{eqn:empproblem}
  \minimize_{\theta \in \Theta}
  \left \{ F_n(\theta)
    \defeq 
    \sup_{P} \left\{ \E[\loss(\theta; Z)] - \gamma W_c(P, \emp) \right\}
    = \E_{\emp}[\phi_\gamma(\theta; Z)] \right \}.
\end{equation}

\subsection{Optimizing the robust loss by stochastic gradient descent}
\label{sec:optimization}

We now develop stochastic gradient-type methods for the relaxed robust
problem~\eqref{eqn:empproblem}, making clear the computational benefits of
relaxing the strict robustness requirements of
formulation~\eqref{eqn:constrained}.  We begin with assumptions we require,
which quantify the amount of robustness we can provide.

\begin{assumption}
  \label{assumption:strong-convexity}
  The function $c: \mathcal{Z} \times \mathcal{Z} \to \R_+$ is continuous. For each
  $z_0 \in \mathcal{Z}$, $c(\cdot, z_0)$ is 1-strongly convex with respect to the
  norm $\norm{\cdot}$.
\end{assumption}
\noindent
To guarantee that the robust surrogate~\eqref{eqn:inner-sup} is tractably
computable, we also require a few smoothness assumptions.  Let $\dnorm{\cdot}$
be the dual norm to $\norm{\cdot}$; we abuse notation by using the same norm
$\norm{\cdot}$ on $\Theta$ and $\mathcal{Z}$, though the specific norm is clear
from context.
\begin{assumption}
  \label{assumption:smoothness}
  The loss $\loss: {\Theta} \times \mathcal{Z} \to \R$ satisfies
  the Lipschitzian smoothness conditions
  \begin{equation*}
    \begin{split}
      \dnorm{\nabla_{\theta}\loss(\theta; z) - \nabla_{\theta}\loss(\theta'; z)}
      \le \lipx \norm{\theta - \theta'},
      & ~~ \dnorm{\nabla_{z}\loss(\theta; z)
        - \nabla_{z}\loss(\theta; z')}
      \le \lipy \norm{z - z'},\\
      \dnorm{\nabla_{\theta}\loss(\theta; z) - \nabla_{\theta}\loss(\theta; z')}
      \le \lipxy \norm{z - z'}, & ~~
      \dnorm{\nabla_{z}\loss(\theta; z) - \nabla_{z}\loss(\theta'; z)}
      \le \lipyx \norm{\theta - \theta'}.
    \end{split}
  \end{equation*}
\end{assumption}
\noindent
These properties guarantee both (i) the well-behavedness of the robust
surrogate $\phi_\gamma$ and (ii) its efficient computability.  Making point
(i) precise, Lemma \ref{lemma:smoothness} shows that if $\gamma$ is large
enough and Assumption \ref{assumption:smoothness} holds, the surrogate
$\phi_\gamma$ is still smooth. Throughout, we assume $\Theta \subseteq \R^d$.

\begin{lemma}
  \label{lemma:smoothness}
    Let $f : {\Theta} \times \mathcal{Z} \to \R$ be differentiable and
  $\lambda$-strongly concave in $z$ with respect to the norm $\norm{\cdot}$,
  and define $\bar{f}(\theta) = \sup_{z \in \mathcal{Z}} f(\theta, z)$.  Let
  $\gradx(\theta, z) = \nabla_{\theta} f(\theta, z)$ and
  $\grady(\theta, z) = \nabla_{z} f(\theta, z)$, and assume $\gradx$ and
  $\grady$ satisfy Assumption~\ref{assumption:smoothness} with $\loss(\theta; z)$ replaced with $f(\theta, z)$. Then $\bar{f}$ is differentiable, and letting
  $z^{\star}(\theta) = \argmax_{z \in \mathcal{Z}} f(\theta, z)$, we have
  $\nabla \bar{f}(\theta) = \gradx(\theta, z^{\star}(\theta))$. Moreover,
  \begin{equation*}
    \norm{z^{\star}(\theta_1) - z^{\star}(\theta_2)}
    \le \frac{\lipyx}{\lambda} \norm{\theta_1 - \theta_2}
    ~~~\mbox{and}~~~ \norm{\nabla \bar{f}(\theta) - \nabla \bar{f}(\theta')}_{\star}
    \le \left(\lipx + \frac{\lipxy\lipyx}{\lambda}\right) \norm{\theta - \theta'}.
  \end{equation*}
\end{lemma}
\noindent See Section~\ref{sec:proof-of-smoothness} for the proof.  Fix
$z_0 \in \mathcal{Z}$ and focus on the $\ell_2$-norm case where $c(z, z_0)$
satisfies Assumption~\ref{assumption:strong-convexity} with
$\ltwo{\cdot}$. Noting that
$f(\theta, z) \defeq \loss(\theta, z) - \gamma c(z, z_0)$ is
$(\gamma - \lipy)$-strongly concave from the
insight~\eqref{eqn:trivial-concavity} (with $L \defeq \lipy$), let us apply
Lemma~\ref{lemma:smoothness}. Under
Assumptions~\ref{assumption:strong-convexity},~\ref{assumption:smoothness},
$\phi_\gamma(\cdot; z_0)$ then has
$L = \lipx + \frac{\lipxy \lipyx}{\hinge{\gamma - \lipy}}$-Lipschitz
gradients, and
\begin{equation*}
  \nabla_\theta \phi_\gamma(\theta; z_0)
  = \nabla_\theta \loss(\theta; z^\star(z_0, \theta))
  ~~ \mbox{where} ~~
  z^\star(z_0, \theta) = \argmax_{z \in \mathcal{Z}} \{\loss(\theta; z)
  - \gamma c(z, z_0)\}.
\end{equation*}
This motivates Algorithm \ref{alg:thealg}, a stochastic-gradient approach for
the penalty problem~\eqref{eqn:empproblem}. The benefits of Lagrangian
relaxation become clear here: for $\loss(\theta;z)$ smooth in $z$ and $\gamma$
large enough, gradient ascent on $\loss(\theta^t; z) - \gamma c(z, z^t)$ in
$z$ converges linearly and we can compute (approximate) $\what{z}^t$
efficiently (we initialize our inner gradient ascent iterations with the
sampled natural example $z^t$).

\begin{algorithm}[t]
  \caption{\label{alg:thealg}
    Distributionally robust optimization with adversarial training}
  \begin{algorithmic}[]
    \State \textsc{Input:} Sampling distribution $P_0$, constraint sets
    $\Theta$ and $\mathcal Z$, stepsize sequence $\{\alpha_t > 0\}_{t=0}^{T-1}$
    \State \textbf{for} {$t=0, \ldots, T-1$} \textbf{do}
    \State~~~ Sample $z^t \sim P_0$ and
    find an $\epsilon$-approximate maximizer $\what{z}^t$
    of $\loss(\theta^t; z) - \gamma c(z, z^t)$
    \State~~~ $\theta^{t+1} \gets
    \proj_\Theta(\theta^t - \alpha_t \nabla_{\theta} \loss(\theta^t; \what{z}^t))$
  \end{algorithmic}
\end{algorithm}

Convergence properties of Algorithm \ref{alg:thealg} depend on the loss
$\loss$.  When $\loss$ is convex in $\theta$ and $\gamma$ is large enough that
$z \mapsto (\loss(\theta; z)- \gamma c(z, z_0))$ is concave for all
$(\theta, z_0) \in \Theta \times \mathcal{Z}$, we have a stochastic monotone
variational inequality, which is efficiently solvable~\citep{JuditskyNeTa11,
  ASchen2014accelerated} with convergence rate $1 / \sqrt{T}$. When the loss
$\loss$ is nonconvex in $\theta$, the following theorem guarantees convergence
to a stationary point of problem~\eqref{eqn:empproblem} at the same rate when
$\gamma \ge \lipy$. Recall that
$F(\theta) = \E_{P_0}[\phi_{\gamma}(\theta; Z)]$ is the robust surrogate
objective for the Lagrangian relaxation~\eqref{eqn:lagrangian-duality}.
\begin{theorem}[Convergence of Nonconvex SGD]\label{thm:convergence}
  Let Assumptions \ref{assumption:strong-convexity} and
  \ref{assumption:smoothness} hold with the $\ell_2$-norm and let
  $\Theta=\R^d$. Let $\Delta_F \ge F(\theta^0) - \inf_\theta F(\theta)$.
  Assume
  $\E[\ltwo{\nabla F(\theta) - \nabla_{\theta} \phi_\gamma(\theta, Z)}^2] \le
  \sigma^2$ and take constant stepsizes
  $\stepsize = \sqrt{\frac{\Delta_F}{L_{\phi} T\sigma^2}}$ where
  $L_{\phi} \defeq \lipx + \frac{\lipxy\lipyx}{\gamma - \lipy}$. For
  $T \ge \frac{L_{\phi}\Delta_{F}}{\sigma^2}$, Algorithm~\ref{alg:thealg}
  satisfies
  \begin{equation*}
    \frac{1}{T}
    \sum_{t = 0}^{T-1} \E\left[\ltwo{\nabla F(\theta^t)}^2\right]
    - \frac{4\lipxy^2}{\gamma - \lipy} \epsilon
    \le 4 \sigma \sqrt{\frac{L_{\phi} \Delta_F}{T}}.
  \end{equation*}
\end{theorem}
\noindent See Section~\ref{sec:proof-of-convergence} for the
proof. We make a few remarks. First, the condition $\E[\ltwo{\nabla F(\theta)-\nabla_\theta \phi_\gamma(\theta, Z)}^2] \le \sigma^2$ holds (to
within a constant factor) whenever $\ltwo{\nabla_\theta \loss(\theta, z)}
\le \sigma$ for all $\theta, z$. Theorem~\ref{thm:convergence} shows that
the stochastic gradient method achieves the rates of convergence on the
penalty problem~\eqref{eqn:empproblem} achievable in standard smooth
non-convex optimization~\citep{GhadimiLa13}. The accuracy parameter
$\epsilon$ has a \emph{fixed} effect on optimization accuracy, independent
of $T$: approximate maximization has limited effects.

Key to the convergence guarantee of Theorem~\ref{thm:convergence} %
is that the loss $\loss$ is smooth in $z$: the inner
supremum~\eqref{eqn:inner-sup} is NP-hard to compute for non-smooth deep
networks (see Lemma~\ref{thm:nphard} in Section~\ref{sec:nphard} below for a proof
of this for networks with ReLUs). The smoothness of $\loss$ is essential so that a
penalized version $\loss(\theta, z) - \gamma c(z, z_0)$ is concave in $z$
(which can be approximately verified by computing Hessians
$\nabla^2_{zz} \loss(\theta, z)$ for each training datapoint), allowing
computation and our coming certificates of optimality. Replacing ReLUs with
sigmoids or ELUs~\citep{ASclevert2015fast} allows us to apply
Theorem~\ref{thm:convergence}, making distributionally robust optimization
tractable for deep learning.

Our distributionally robust framework~\eqref{eqn:lagrangian-duality} is
general enough to consider adversarial perturbations to an arbitrary subset of
coordinates in $Z$.  For example, it is appropriate in certain applications to
hedge against adversarial perturbations to a small fixed region of an
image~\citep{BrownMaRoAbGi17}. By modifying the cost function $c(z, z')$ to
take value $\infty$ outside this small region, our general formulation covers
such variants. In Section \ref{sec:supervised}, we illustrate this
modification for supervised-learning scenarios, where we adversarially perturb
feature vectors of datapoints but not their labels.

\subsubsection{Finding worst-case perturbations with ReLUs is NP-hard}
\label{sec:nphard}
To emphasize the importance of smoothness in efficiently finding solutions to the inner supremem~\eqref{eqn:inner-sup}, we show that computing worst-case perturbations
$\sup_{u \in \mathcal{U}} \loss(\theta; z + u)$ is NP-hard for a large class
of feedforward neural networks with (non-smooth) ReLU activations. This result is
essentially due to \citet{ASkatz2017reluplex}. In the following, we use
polynomial time to mean polynomial growth with respect to $m$, the dimension
of the inputs $z$.

\ifdefined\useorstyle
An optimization problem is \emph{NPO} (NP-Optimization) if (i) the
dimensionality of the solution grows polynomially, (ii) the language
$\{u \in \mathcal{U} \}$ can be recognized in polynomial time (i.e.\ a
deterministic algorithm can decide in polynomial time whether
$u \in \mathcal{U}$), and (iii) $\loss$ can be evaluated in polynomial
time. We restrict analysis to feedforward neural networks with ReLU
activations such that the corresponding worst-case perturbation problem is
NPO (note that $z,u \in \R^m$, so trivially the dimensionality of the
  solution grows polynomially). We also impose separable structure on
$\mathcal{U}$, that is, $\mathcal{U}:=\{v \le u \le w \}$ for 
$v < w \in \R^m$. See Section~\ref{section:proof-of-lemma-nphard} for the
proof of the following result.
\else
An optimization problem is \emph{NPO} (NP-Optimization) if (i) the
dimensionality of the solution grows polynomially, (ii) the language
$\{u \in \mathcal{U} \}$ can be recognized in polynomial time (i.e.\ a
deterministic algorithm can decide in polynomial time whether
$u \in \mathcal{U}$), and (iii) $\loss$ can be evaluated in polynomial
time. We restrict analysis to feedforward neural networks with ReLU
activations such that the corresponding worst-case perturbation problem is
NPO.\footnote{Note that $z,u \in \R^m$, so trivially the dimensionality of the
  solution grows polynomially.} Furthermore, we impose separable structure on
$\mathcal{U}$, that is, $\mathcal{U}:=\{v \le u \le w \}$ for some
$v < w \in \R^m$. See Section~\ref{section:proof-of-lemma-nphard} for the
proof of the following result.
\fi
\begin{lemma}
  \label{thm:nphard}
  Consider feedforward neural networks with ReLUs and let $\mathcal{U}:=\{v \le
  u \le w\}$, where $v < w$ such that the optimization problem
  $\mathop{\max}_{u \in \mathcal{U}} \loss(\theta; z + u)$ is NPO. Then there exists
  $\theta$ such that this optimization problem is also NP-hard.
\end{lemma}

\subsection{Supervised learning}
\label{sec:supervised}

\renewcommand{\gradx}{\mathsf{g}_{\mathsf{\theta}}}
\renewcommand{\grady}{\mathsf{g}_{\mathsf{x}}}
\renewcommand{\lipx}{L_{\mathsf{\theta \theta}}}
\renewcommand{\lipy}{L_{\mathsf{xx}}}
\renewcommand{\lipxy}{L_{\mathsf{\theta x}}}
\renewcommand{\lipyx}{L_{\mathsf{x \theta}}}

\ifdefined\useorstyle
In supervised learning settings such as classification, it is often natural to only consider adversarial perturbations to the feature
vectors (covariates). In this section, we give an adaptation of the results in
Section~\ref{sec:approach} to such
scenarios. Let $Z = (X, Y) \in \mathcal{X} \times \R$ where $X \in \mathcal{X}$ is a
feature vector (where we assume that $\mathcal{X}$ is a subset of normed vector
  space) and $Y \in \R$ is a label. In classification settings, we have
$Y \in \{1, \ldots, K\}$. We consider an adversary that can only perturb the
feature vector $X$~\citep{ASgoodfellow2015explaining}, which can be easily
represented in our robust formulation~\eqref{eqn:lagrangian-duality} by
defining the Wasserstein cost function $c: \mathcal{Z} \times \mathcal{Z} \to \R_+ \cup \{\infty\}$
as follows: for $z = (x, y)$ and $z' = (x', y')$ define the
covariate-shift cost function as
\begin{equation}
  \label{eqn:supervised-cost}
  c(z, z') \defeq c_x(x, x') + \infty \cdot \indic{y \neq y'}
\end{equation}
where $c_x: \mathcal{X} \times \mathcal{X} \to \R_+$ is the
transportation cost for the feature vector $X$. As before, we assume that
$c_x$ is nonnegative, continuous, convex in its first argument and satisfies
$c_x(x, x) = 0$.
\else
In supervised learning settings such as classification, it is often natural to only consider adversarial perturbations to the feature
vectors (covariates). In this section, we give an adaptation of the results in
Section~\ref{sec:approach} to such
scenarios. Let $Z = (X, Y) \in \mathcal{X} \times \R$ where $X \in \mathcal{X}$ is a
feature vector\footnote{We assume that $\mathcal{X}$ is a subset of normed vector
  space.}  and $Y \in \R$ is a label. In classification settings, we have
$Y \in \{1, \ldots, K\}$. We consider an adversary that can only perturb the
feature vector $X$~\citep{ASgoodfellow2015explaining}, which can be easily
represented in our robust formulation~\eqref{eqn:lagrangian-duality} by
defining the Wasserstein cost function $c: \mathcal{Z} \times \mathcal{Z} \to \R_+ \cup \{\infty\}$
as follows: for $z = (x, y)$ and $z' = (x', y')$ define the
covariate-shift cost function as
\begin{equation}
  \label{eqn:supervised-cost}
  c(z, z') \defeq c_x(x, x') + \infty \cdot \indic{y \neq y'}
\end{equation}
where $c_x: \mathcal{X} \times \mathcal{X} \to \R_+$ is the
transportation cost for the feature vector $X$. As before, we assume that
$c_x$ is nonnegative, continuous, convex in its first argument and satisfies
$c_x(x, x) = 0$.
\fi

Under the cost function~\eqref{eqn:supervised-cost}, the robust surrogate loss
in the penalty problem~\eqref{eqn:lagrangian-duality} and its empirical
counterpart~\eqref{eqn:empproblem} become
\begin{equation*}
  \phi_{\gamma}(\theta; (x_0, y_0))
  = \sup_{x \in \mathcal{X}}
      \left\{\loss(\theta; (x, y_0)) - \gamma c_x(x, x_0) \right\}.
\end{equation*}
Similarly as in Section~\ref{sec:optimization}, we require the following two
assumptions that guarantee efficient computability of the robust surrogate
$\phi_{\gamma}$.
\begin{assumption}
  \label{assumption:strong-convexity-supervised}
  The function $c_x: \mathcal{X} \times \mathcal{X} \to \R_+$ is
  continuous. For each $x_0 \in \mathcal{X}$, $c_x(\cdot, x_0)$ is 1-strongly convex
  with respect to the norm $\norm{\cdot}$.
\end{assumption}
\noindent
Let $\dnorm{\cdot}$ be the dual norm to $\norm{\cdot}$; we again abuse
notation by using the same norm $\norm{\cdot}$ on $\Theta$ and $\mathcal{X}$,
though the specific norm is clear from context.
\begin{assumption}
  \label{assumption:smoothness-supervised}
  The loss $\loss: {\Theta} \times \mathcal{Z} \to \R$ satisfies
  the Lipschitzian smoothness conditions
  \begin{equation*}
    \begin{split}
      \dnorm{\nabla_{\theta}\loss(\theta; (x, y))
        - \nabla_{\theta}\loss(\theta'; (x, y))}
      \le \lipx \norm{\theta - \theta'},
      & ~~ \dnorm{\nabla_{x}\loss(\theta; (x, y))
        - \nabla_{x}\loss(\theta; (x', y))}
      \le \lipy \norm{x - x'},\\
      \dnorm{\nabla_{\theta}\loss(\theta; (x, y))
        - \nabla_{\theta}\loss(\theta; (x', y))}
      \le \lipxy \norm{x - x'}, & ~~
      \dnorm{\nabla_{x}\loss(\theta; (x, y)) - \nabla_{x}\loss(\theta'; (x, y))}
      \le \lipyx \norm{\theta - \theta'}.
    \end{split}
  \end{equation*}
\end{assumption}

\noindent Under Assumptions~\ref{assumption:strong-convexity-supervised}
and~\ref{assumption:smoothness-supervised}, an analogue of
Lemma~\ref{lemma:smoothness} still holds. The proof of the following result is
nearly identical to that of Lemma~\ref{lemma:smoothness}; we state the full
result for completeness.
\begin{lemma}
  \label{lemma:smoothness-supervised}
  Let $f : {\Theta} \times \mathcal{X} \to \R$ be differentiable and
  $\lambda$-strongly concave in $x$ with respect to the norm $\norm{\cdot}$,
  and define $\bar{f}(\theta) = \sup_{x \in \mathcal{X}} f(\theta, x)$.  Let
  $\gradx(\theta, x) = \nabla_{\theta} f(\theta, x)$ and
  $\grady(\theta, x) = \nabla_{x} f(\theta, x)$, and assume $\gradx$ and
  $\grady$ satisfy the Lipschitz conditions of
  Assumption~\ref{assumption:smoothness}.  Then $\bar{f}$ is differentiable,
  and letting $x^{\star}(\theta) = \argmax_{x \in \mathcal{X}} f(\theta, x)$, we
  have $\nabla \bar{f}(\theta) = \gradx(\theta, x^{\star}(\theta))$. Moreover,
  \begin{equation*}
    \norm{x^{\star}(\theta_1) - x^{\star}(\theta_2)}
    \le \frac{\lipyx}{\lambda} \norm{\theta_1 - \theta_2}
    ~~~\mbox{and}~~~ \norm{\nabla \bar{f}(\theta) - \nabla \bar{f}(\theta')}_{\star}
    \le \left(\lipx + \frac{\lipxy\lipyx}{\lambda}\right) \norm{\theta - \theta'}.
    \end{equation*}
\end{lemma}
\noindent From Lemma~\ref{lemma:smoothness-supervised}, we immediately get an analogue to Theorem~\ref{thm:convergence} for the cost function~\eqref{eqn:supervised-cost}.
\begin{corollary}[Convergence of Nonconvex SGD]\label{cor:convergence-supervised}
  Let Assumptions \ref{assumption:strong-convexity-supervised} and
  \ref{assumption:smoothness-supervised} hold with the $\ell_2$-norm and let
  $\Theta=\R^d$. Let $\Delta_F \ge F(\theta^0) - \inf_\theta F(\theta)$.
  Assume
  $\E[\ltwo{\nabla F(\theta) - \nabla_{\theta} \phi_\gamma(\theta, Z)}^2] \le
  \sigma^2$, and take constant stepsizes
  $\stepsize = \sqrt{\frac{2 \Delta_F}{L \sigma^2 T}}$ where
  $L=\lipx + \frac{\lipxy\lipyx}{\gamma - \lipy}$. Then
  Algorithm~\ref{alg:thealg} satisfies
  \begin{equation*}
    \frac{1}{T}
    \sum_{t = 1}^T \E\left[\ltwo{\nabla F(\theta^t)}^2\right]
    - \frac{2 \lipxy^2}{\gamma - \lipy} \epsilon
    \le \sigma \sqrt{8 \frac{L \Delta_F}{T}}.
  \end{equation*}
\end{corollary}
\noindent Proof of Corollary~\ref{cor:convergence-supervised} is identical
to that of Theorem~\ref{thm:convergence}, but applies
Lemma~\ref{lemma:smoothness-supervised} instead of
Lemma~\ref{lemma:smoothness}.

\renewcommand{\grady}{\mathsf{g}_{\mathsf{z}}}
\renewcommand{\lipy}{L_{\mathsf{zz}}}
\renewcommand{\lipxy}{L_{\mathsf{\theta z}}}
\renewcommand{\lipyx}{L_{\mathsf{z \theta}}}

\section{Certificate of robustness and generalization}
\label{sec:generalization-robustness}

\newcommand{\zbd}{M_{\mathsf{z}}}
\newcommand{\objbd}{M_{\ell}}

From results in the previous section, Algorithm~\ref{alg:thealg} provably
learns to protect against adversarial perturbations of the
form~\eqref{eqn:empproblem} on the training dataset. Now we show that such
procedures generalize, allowing us to prevent attacks on the test set. Our
subsequent results hold uniformly over the space of parameters
$\theta \in \Theta$, including $\theta_{\rm WRM}$, the output of the
stochastic gradient descent procedure in Section~\ref{sec:optimization}. Our
first main result, presented in Section~\ref{sec:rob-cert}, gives a
data-dependent upper bound on the population worst-case objective
$\sup_{P: W_c(P, P_0) \le \tol} \E_{P}[\loss(\theta; Z)]$ for any arbitrary
level of robustness $\tol$; this bound is optimal for $\tol = \what{\tol}_n$,
the level of robustness achieved for the empirical distribution by
solving~\eqref{eqn:empproblem}. Our bound is efficiently computable and hence
\emph{certifies} a level of robustness for the worst-case population
objective. Second, we show in Section~\ref{sec:adv-gen} that adversarial
perturbations on the training set (in a sense) generalize: solving the
empirical penalty problem~\eqref{eqn:empproblem} guarantees a similar level of
robustness as directly solving its population
counterpart~\eqref{eqn:lagrangian-duality}.

\subsection{Certificate of robustness}
\label{sec:rob-cert}

Our main result in this section is a data-dependent upper bound for the
worst-case population objective: $\sup_{P: W_c(P, P_0) \le \tol}
\E_P[\loss(\theta; Z)] \le \gamma \tol + \E_{\emp}[\phi_{\gamma}(\theta; Z)]
+ O(1/\sqrt{n})$ for all $\theta \in \Theta$, with high probability. To make
this rigorous, fix $\gamma > 0$, and consider the worst-case perturbation,
typically called the \emph{transportation map} or Monge
map~\citep{Villani09},
\begin{equation}
  \label{eqn:transport-map}
  T_\gamma(\theta; z_0) \defeq
  \argmax_{z \in \mathcal{Z}} \{ \loss(\theta; z) - \gamma c(z, z_0)\}.
\end{equation}
Under our assumptions, $T_\gamma$ is easily computable when $\gamma \ge
\lipy$. Letting $\delta_z$ denote the point mass at $z$,
Proposition~\ref{prop:duality}
shows the empirical maximizers of the
Lagrangian formulation~\eqref{eqn:lagrangian} are attained by
\begin{equation}
  \label{eqn:tol-n}
  \begin{split}
    P_n^*(\theta) & \defeq \argmax_{P} \left\{
    \E_P[\loss(\theta; Z)] - \gamma W_c(P, \emp) \right\}
    = \frac{1}{n} \sum_{i = 1}^n \delta_{T_\gamma(\theta, Z_i)}
      ~~~\mbox{and} \\
    \what{\tol}_n (\theta) & \defeq W_c(P_n^*(\theta), \emp)
    = \E_\emp[c(T_\gamma(\theta; Z), Z)].
  \end{split}
\end{equation}
Our results imply, in particular, that the empirical worst-case loss
$\E_{P_n^*}[\loss(\theta; Z)]$ gives a \emph{certificate of robustness} to
(population) Wasserstein perturbations up to level
$\what{\tol}_n$. $\E_{P_n^*(\theta)}[\loss(\theta; Z)]$ is efficiently
computable via \eqref{eqn:tol-n}, providing a data-dependent guarantee for the
worst-case population loss.

Our bound relies on the usual covering numbers for the model class
$\{\loss(\theta; \cdot): \theta \in \Theta\}$ as the notion of
complexity~\citep[e.g.][]{VanDerVaartWe96}, so, despite the
infinite-dimensional problem~\eqref{eqn:empproblem},
we retain the same uniform convergence guarantees typical of empirical risk
minimization. Recall that for a set $V$, a collection $v_1, \ldots, v_N$
 is an \emph{$\epsilon$-cover} of $V$ in norm
$\norm{\cdot}$ if for each $v \in \mathcal{V}$, there exists $v_i$ such that
$\norm{v - v_i} \le \epsilon$. The \emph{covering number} of $V$ with
respect to $\norm{\cdot}$ is
\begin{equation*}
  \covnum(V, \epsilon, \norm{\cdot}) \defeq
  \inf\left\{\covnum \in \N \mid \mbox{there~is~an~}
    \epsilon \mbox{-cover~of~} V
    ~ \mbox{with~respect~to~} \norm{\cdot} \right\}.
\end{equation*}
For $\fclass \defeq \left\{ \loss(\theta, \cdot): \theta \in \Theta \right\}$
equipped with the $L^\infty(\statdomain)$ norm
$ \linfstatnorm{f} \defeq \sup_{\statval \in \statdomain}|f(\statval)|$, we
state our results in terms of $\linfstatnorm{\cdot}$-covering numbers of
$\fclass$. To ease notation, we let
\begin{align*}
  \epsilon_{n} (t)
  & \defeq \gamma b_1 \sqrt{\frac{\objbd}{n}}
    \int_0^1 \sqrt{\log N(\fclass, \objbd\epsilon, \linfstatnorm{\cdot})}
    d\epsilon
    + b_2 \objbd \sqrt{\frac{t}{n}} 
\end{align*}
where $b_1, b_2$ are numerical constants.

We are now ready to state the main result of this section. We first show from
the duality result~\eqref{eqn:lagrangian} that we can provide an upper bound
for the worst-case population performance for any level of robustness $\tol$.
For $\tol = \what{\tol}_n(\theta)$ and $\theta = \theta_{\rm WRM}$, this certificate
is (in a sense) tight as we see below.
\begin{theorem}
  \label{theorem:robustness}
  Assume $|\loss(\theta; z)| \le \objbd$ for all $\theta \in \Theta$ and
  $z \in \mathcal{Z}$. Then, for a fixed $t > 0$ and numerical constants
  $b_1, b_2 > 0$, with probability at least $1-e^{-t}$, simultaneously for all
  $\theta \in \Theta$, $\tol \ge 0$, $\gamma \ge 0$,
  \begin{align}
    \label{eqn:robustness-any-rho}
    \sup_{P: W_c(P, P_0) \le \tol} \E_P[\loss(\theta; Z)]
    \le \gamma \tol + \E_{\emp}[\phi_{\gamma}(\theta; Z)]
    + \epsilon_n(t).
  \end{align}
  In particular, if $\tol = \what{\tol}_n(\theta)$ then with probability at
  least $1-e^{-t}$, for all $\theta \in \Theta$
  \begin{align}
    \label{eqn:robustness-good-rho}
    \sup_{P: W_c(P, P_0) \le \what{\tol}_n(\theta)}
    \E_{P} [\loss(\theta; Z)]
    & \le \gamma \what{\tol}_n(\theta) + \E_{\emp}[\phi_{\gamma}(\theta; Z)] 
    + \epsilon_n(t)
     \nonumber \\
    & = 
    \sup_{P: W_c(P, \emp) \le \what{\tol}_n(\theta)}
      \E_{P}[\loss(\theta; Z)] + \epsilon_n(t).
  \end{align}
\end{theorem}
\noindent See Section~\ref{section:proof-robustness} for its proof. We now
give a concrete variant for Lipschitz functions. When $\Theta$ is
finite-dimensional ($\Theta \subset \R^d$), Theorem~\ref{theorem:robustness}
provides a robustness guarantee scaling linearly with $d$ despite the
infinite-dimensional Wasserstein penalty.  Assuming there exist
$\theta_0 \in \Theta$, $M_{\theta_0} < \infty$ such that
$|\loss(\theta_0; z)| \le M_{\theta_0}$ for all $z \in \mathcal{Z}$, we have
the following corollary (see proof in Section~\ref{section:proof-lipschitz}).
\begin{corollary}
  \label{corollary:lipschitz}
  Let $\loss(\cdot; \statval)$ be $L$-Lipschitz with respect to some norm
  $\norm{\cdot}$ for all $\statval \in \statdomain$. Assume that
  $\Theta \subset \R^d$ satisfies
  $\diam(\Theta) = \sup_{\theta, \theta' \in \Theta} \norm{\theta - \theta'} <
  \infty$. Then, the bounds~\eqref{eqn:robustness-any-rho}
  and~\eqref{eqn:robustness-good-rho} hold with
  \begin{equation}
    \epsilon_{n}(t) = b_1 \sqrt{\frac{d (L \diam(\Theta) + M_{\theta_0})}{n}}
    + b_2 (L \diam(\Theta) + M_{\theta_0}) \sqrt{\frac{t}{n}}
    \label{eqn:lipschitz-bound}
  \end{equation}
  for some numerical constants $b_1, b_2 > 0$.
\end{corollary}

A key consequence of the bound~\eqref{eqn:robustness-any-rho} is that
$\gamma \tol + \E_{\emp}[\phi_{\gamma}(\theta; Z)]$ \emph{certifies}
robustness for the worst-case population objective for any $\rho$ and
$\theta$.  For a given $\theta$, this certificate is tightest at the achieved
level of robustness $\what{\tol}_n(\theta)$, as noted in the refined
bound~\eqref{eqn:robustness-good-rho} which follows from the duality result
\begin{align}
  \underbrace{\E_{\emp}[\phi_\gamma(\theta; Z)]}_{\rm surrogate~loss}
  +  \underbrace{\gamma\what{\tol}_n(\theta)}_{\rm robustness}
  & = \sup_{P: W_c(P, \emp) \le \what{\tol}_n(\theta)} \E_{P}[\loss(\theta; Z)]
  = \E_{P_n^*(\theta)}[\loss(\theta; Z)].
  \label{eqn:certificate}
\end{align}
(See Section~\ref{section:proof-robustness} for a proof of these equalities.)
We expect $\theta_{\rm WRM}$, the output of Algorithm \ref{alg:thealg}, to be
close to the minimizer of the surrogate loss
$\E_{\emp}[\phi_{\gamma}(\theta; Z)]$ and therefore have the best
guarantees. Most importantly, the certificate~\eqref{eqn:certificate} is easy
to compute via expression~\eqref{eqn:tol-n}: as noted in
Section~\ref{sec:optimization}, the mappings $T(\theta, Z_i)$ are efficiently
computable for large enough $\gamma$, and
$\what{\rho}_n = \E_{\emp}[c(T(\theta, Z), Z)]$.

The bounds~\eqref{eqn:robustness-any-rho}--\eqref{eqn:lipschitz-bound} may
be too large---because of their dependence on covering numbers and
dimension---for practical use in security-critical applications.  With that
said, the strong duality result, Proposition~\ref{prop:duality}, still
applies to any distribution. Given a collection of test
examples $Z^{\rm test}_i$, we may interrogate possible losses under
perturbations for the \emph{test} examples by noting that, if
$\what{P}_{\rm test}$ denotes the empirical distribution on
the test set (say, with putative assigned labels), then
\begin{equation}
  \label{eqn:certify-tests}
  \frac{1}{n_{\rm test}}
  \sum_{i = 1}^n \sup_{z : c(z, Z_i^{\rm test}) \le \tol}
  \left\{\loss(\theta; z)\right\}
  \le \sup_{P : W_c(P, \what{P}_{\rm test})
    \le \tol} \E_P[\loss(\theta; Z)]
  \le \gamma \tol + \E_{\what{P}_{\rm test}}[
    \phi_\gamma(\theta; Z)]
\end{equation}
for all $\gamma, \tol \ge 0$. Whenever $\gamma$ is large enough
(so that this is tight for small $\tol$),
we may efficiently compute the Monge-map~\eqref{eqn:transport-map}
and the test loss~\eqref{eqn:certify-tests} to guarantee bounds on the
sensitivity of a parameter $\theta$ to a particular sample and
predicted labeling based on the sample.

\subsection{Generalization of adversarial examples}
\label{sec:adv-gen}

We can also show that the level of robustness on the training set
generalizes. Our starting point is
Lemma~\ref{lemma:smoothness}, which shows that $T_{\gamma}(\cdot; z)$ is
smooth under Assumptions~\ref{assumption:strong-convexity}
and~\ref{assumption:smoothness}:
\begin{equation}
  \norm{T_\gamma(\theta_1; z)
    - T_\gamma(\theta_2; z)}
  \le \frac{\lipyx}{\hinge{\gamma-\lipy}} \norm{\theta_1 - \theta_2}
  \label{eqn:transport-smoothness}
\end{equation}
for all $\theta_1, \theta_2$, where $\lipy$ was the Lipschitz
constant of $\nabla_z \loss(\theta; z)$. We provide explicit bounds on the smoothness constant $\lipy$ for neural networks with smooth
activation functions in Section \ref{section:examples}. Leveraging this smoothness, we now show that
$\what{\tol}_n(\theta) = \E_{\emp}[c(T_{\gamma}(\theta; Z), Z)]$, the level of
robustness achieved for the empirical problem, concentrates uniformly around
its population counterpart.
\begin{theorem}
  \label{theorem:dist-concentration}
  Let $\mathcal{Z} \subset \{z \in \R^m: \norm{z} \le \zbd\}$ so that
  $\norm{Z} \le \zbd$ almost surely and assume either that (i)
  $c(\cdot, \cdot)$ is $\lipc$-Lipschitz over $\mathcal{Z}$ with respect to the
  norm $\norm{\cdot}$ in each argument, or (ii) that
  $\loss(\theta, z) \in [0, \objbd]$ and $z \mapsto \loss(\theta, z)$ is
  $\gamma \lipc$-Lipschitz for all $\theta \in \Theta$.

  If Assumptions~\ref{assumption:strong-convexity}
  and~\ref{assumption:smoothness} hold, then with probability at least
  $1-e^{-t}$,
  \begin{equation}
    \label{eqn:dist-concentration-supervised}
    \sup_{\theta \in \Theta}
    | \E_\emp[c(T_\gamma(\theta; Z), Z)]
    - \E_{P_0}[c(T_\gamma(\theta; Z), Z)] |
    \le 4 B \sqrt{\frac{1}{n}
      \left(t + \log N \left( \Theta,
          \frac{\hinge{\gamma-\lipy} t}{4\lipc \lipyx},
          \norm{\cdot} \right) \right)}.
  \end{equation}
  where $B = \lipc \zbd$ under assumption $(i)$ and $B = \objbd/\gamma$ under
  assumption $(ii)$.
\end{theorem}
\noindent See Section~\ref{sec:proof-of-dist-concentration} for the proof.
For $\Theta \subset \R^d$, we have
$\log \covnum(\Theta, \epsilon, \norm{\cdot}) \le d \log (1 +
  \frac{\diam(\Theta)}{\epsilon})$
so that the bound~\eqref{eqn:dist-concentration} gives the usual $\sqrt{d/n}$
generalization rate for the distance between adversarial perturbations and
natural examples. Another consequence of Theorem~\ref{theorem:dist-concentration} is that
$\what{\tol}_n(\theta_{\rm WRM})$ in the
certificate~\eqref{eqn:robustness-good-rho} is positive as long as the loss
$\loss$ is not completely invariant to data. To see this, note from the
optimality conditions for $T_{\gamma}(\theta; Z)$ that
$\E_{P_0}[c(T_{\gamma}(\theta; Z), Z)] = 0$ iff
$\nabla_{z} \loss(\theta; z) = 0$ almost surely, and hence for large enough
$n$, we have $\what{\tol}_n (\theta) > 0$ by the
bound~\eqref{eqn:dist-concentration}.

\renewcommand{\grady}{\mathsf{g}_{\mathsf{x}}}
\renewcommand{\lipy}{L_{\mathsf{xx}}}
\renewcommand{\lipxy}{L_{\mathsf{\theta x}}}
\renewcommand{\lipyx}{L_{\mathsf{x \theta}}}

An analogous result to Theorem~\ref{theorem:dist-concentration}
holds in the supervised-learning setting with the modified cost
function~\eqref{eqn:supervised-cost}. Abusing notation,  redefine the
transport map with the modified cost function~\eqref{eqn:supervised-cost}
\begin{equation*}
  \label{eqn:transport-map-supervised}
  T_\gamma(\theta; (x_0, y_0)) \defeq
  \argmax_{x \in \mathcal{X}} \{ \loss(\theta; (x, y_0)) - \gamma c_x(x, x_0)\}.
\end{equation*}
\begin{corollary}
  \label{corollary:dist-concentration-supervised}
  Let $\mathcal{Z} \subset \{z \in \R^m: \norm{z} \le \zbd\}$ so that
  $\norm{Z} \le \zbd$ almost surely and assume either that (i)
  $c_x(\cdot, \cdot)$ is $\lipc$-Lipschitz over $\mathcal{X}$ with respect to the
  norm $\norm{\cdot}$ in each argument, or (ii) that
  $\loss(\theta, z) \in [0, \objbd]$ and $x \mapsto \loss(\theta, (x, y))$ is
  $\gamma \lipc$-Lipschitz for all $\theta \in \Theta$.  If
  Assumptions~\ref{assumption:strong-convexity-supervised}
  and~\ref{assumption:smoothness-supervised} hold, then with probability at
  least $1-e^{-t}$,
  \begin{equation}
    \label{eqn:dist-concentration}
    \sup_{\theta \in \Theta}
    | \E_\emp[c(T_\gamma(\theta; Z), Z)]
    - \E_{P_0}[c(T_\gamma(\theta; Z), Z)] |
    \le 4 D \sqrt{\frac{1}{n}
      \left(t + \log N \left( \Theta,
          \frac{\hinge{\gamma-\lipy} t}{4\lipc \lipyx},
          \norm{\cdot} \right) \right)}.
  \end{equation}
  where $B = \lipc \zbd$ under assumption $(i)$ and $B = \objbd/\gamma$ under
  assumption $(ii)$.
\end{corollary}
\noindent Proof of Corollary~\ref{corollary:dist-concentration-supervised}
is identical to that of Theorem~\ref{theorem:dist-concentration}, but applies Lemma~\ref{lemma:smoothness-supervised} instead of
Lemma~\ref{lemma:smoothness}.

\renewcommand{\grady}{\mathsf{g}_{\mathsf{z}}}
\renewcommand{\lipy}{L_{\mathsf{zz}}}
\renewcommand{\lipxy}{L_{\mathsf{\theta z}}}
\renewcommand{\lipyx}{L_{\mathsf{z \theta}}}

\section{Bounds on smoothness of neural networks}
\label{section:examples}

In this section, we give upper bounds on the Lipschitz constant $L_{\mathsf{xx}}$ of
the loss $x \mapsto \loss(\theta; (x, y))$. We focus on the supervised-learning setting presented in Section~\ref{sec:supervised} for ease of
notation.
Since our optimization and generalization guarantees apply only for
$\gamma \ge L_{\mathsf{xx}}$, our below examples serve as concrete numbers at
which we can provide theoretical guarantees on adversarial training. We first
provide bounds on deep neural networks with smooth activation functions, and
apply them in the classification setting. Our bounds are based on worst-case
matrix norm bounds that can be prohibitively loose for deep architectures but
often provide reasonable estimates in moderate scales. Due to the conservative
nature of the bound, choosing $\gamma$ larger than this value---so that
theoretical results in previous sections apply---may not yield appreciable
adversarial robustness in practice. In Section~\ref{sec:experiments}, we
provide empirical discussion of the bounds, and further provide some negative
results even in the MNIST setting in Section~\ref{section:three-choose-two}.

Before we present our examples, we first set some notation. For a parameter
$\theta$ and a feature vector $x \in \mathcal{X} \subset \R^p$, we denote a
deep network with $L$ layers by $x \mapsto F_L(\theta; x)$ which maps inputs
$x$ to an ouput $y = F(\theta; X) \in \R^{\numclass}$. In the classification
setting, $\numclass$ denotes the number of classes, and $F_{L, k}(\theta; x)$
is the $k$-th element of $F_{L}(\theta; x)$; the final loss is computed using
the softmax loss
\begin{equation}
  \label{eqn:softmax}
  \loss(\theta; (x, y)) = - \log p_y(\theta; x)
  ~~\mbox{where}~~
  p_{y}(\theta; x) \defeq \frac{\exp(F_{L, y}(\theta; x))}
  {\sum_{k=1}^K \exp(F_{L, k}(\theta; x))}
\end{equation}
where $p_y(\theta; x)$ is the softmax probability for class
$y \in \{1, \ldots, \numclass\}$.

We let $\theta = (\theta_1, \ldots, \theta_L)$, where
$\theta_l \in \R^{d_{l, I} \times d_{l-1, O}}$ is the weight matrix at the
$l$-th layer of the deep network, where $d_{0, O} = p$ and
$d_{L, O} = \numclass$. We denote a nonlinear operation function after the
$l$-th linear layer $\sigma_{l}: \R^{d_{l, I}} \to \R^{d_{l, O}}$, so that
output after the operation is given by
\begin{equation}
  \label{eqn:deep-net}
  F_l(\theta; x) \defeq \sigma_l(\theta_l \cdot \sigma_{l-1}(\theta_{l-1} \cdots
  \sigma_1(\theta_1 \cdot x) \cdots )).
\end{equation}
For convolutional neural networks, our notation $\sigma_l$ includes both
pooling operations and activation functions. We also denote the Jacobian of $x \mapsto F_{l}(\theta; x)$ as $J_xF_l(\theta; x)$.  

To proceed with providing an upper bound on the Lipschitz constant of a neural network, we assume a given level of smoothness for each respective layer of the network.
\begin{assumption}
  \label{assumption:activation}
  For all $l =1, \ldots, L$, $\sigma_l: \R^{d_{l, I}} \to \R^{d_{l, O}}$ is
  $L_l^0$-Lipschitz w.r.t. the $\ell_2$-norm $\ltwo{\cdot}$, and its Jacobian
  $J\sigma_l: \R^{d_{l, I}} \to \R^{d_{l, O} \times d_{l, I}}$ is
  $L_l^1$-Lipschitz w.r.t. $(\opnorm{\cdot}, \ltwo{\cdot})$, where
  $\opnorm{\cdot}$ is the $\ell_2$-operator (spectral) norm. Furthermore, $L_l^0 > 0$ for all $l$.
\end{assumption}
\noindent The last part of the assumption that $L_l^0 > 0$ for all layers in the network assures that the network is not degenerate, as $L^0=0$ defines a layer whose output is constant and independent of the input. We now provide a few examples of layer-wise Lipschitz constants for smooth operations common in convolutional neural networks.

\begin{example}[Average pooling]
  Let $\sigma: \R^{d_I} \to \R^{d_O}$ be given by
  $\sigma(x)_j =\frac{1}{|I_j|} \sum_{a \in I_j} x_a $, where
  $I_j \subseteq \{1, \ldots, d_I\}$ with $\min_{j} |I_j| \ge m_l$ and $\max_{j} |I_j| \le m_u$. Further, let $N$ be the maximum number of times an index appears in $I_j$ for $j=1, \ldots, d_O$ so that $N \defeq \max_{1 \le a \le d_{I}} \sum_{j=1}^{d_{O}} \indic{a \in I_j} \le d_O$. Then, $\sigma$ satisfies Assumption~\ref{assumption:activation} with $L^0 = \frac{1}{m_l}\sqrt{{N m_u}}$ and $L^1 = 0$. To see this, note that
  \begin{align*}
    | \sigma(x)_j - \sigma(x')_j|
    = \left| \frac{1}{|I_j|} \sum_{a \in I_j} (x_a - x_a') \right|
    \le \frac{1}{m_l} \left | \sum_{a \in I_j} (x_a - x_a') \right | \le \frac{\sqrt{m_u}}{m_l}\ltwo{x_{I_j} -x'_{I_j}},
  \end{align*}
  where $x_{I_j}= \sum_{a \in I_j} x_a \mathbf{e}_a$. Then, 
  \begin{equation*}
    \ltwo{\sigma(x) - \sigma(x')} = 
    \left( \sum_{j=1}^{d_O} (\sigma(x)_j - \sigma(x')_j)^2 \right)^{1/2}
    \le \frac{\sqrt{m_u}}{m_l}\left( \sum_{j=1}^{d_O} \ltwo{x_{I_j} -x'_{I_j}}^2 \right)^{1/2}
    \le \frac{\sqrt{Nm_u}}{m_l} \ltwo{x - x'}.
  \end{equation*}  
  Since $\sigma$ is linear, we have $L^1 = 0$.
\end{example}

\begin{example}[Sigmoid] \label{example:sigmoid}
  Let $\sigma: \R^d \to \R^d$ be the elementwise sigmoid activation
  $\sigma(x)_j = \frac{1}{1+e^{-x_j}}$. Then, $\sigma$ satisfies
  Assumption~\ref{assumption:activation} with $L^0=1/4$ and $L^1=1/10$. To see this, note that $J_x\sigma(x)$, the Jacobian of $x \mapsto \sigma(x)$
  is the diagonal matrix with $i$-th diagonal entry
  $\sigma'(x)_i := \frac{d}{dx_i}\sigma(x)_i=
  \sigma(x)_i(1-\sigma(x)_i)$. $L_0$ is given by the spectral norm of this
  matrix, which is $\max_i \sigma'(x)_i \le \max_i \sup_x \sigma'(x)_i =
  1/4$. For $L^1$, we first define
  $\sigma''(x)_i := \frac{d^2}{dx_i^2}\sigma(x)_i=
  \sigma(x)_i(1-\sigma(x)_i)(1-2\sigma(x)_i)$ and note that
    $\sup_x \sigma''(x)_i \le -\frac{(1-\sqrt{3})(2-\sqrt{3})}{(3-\sqrt{3})^3}
    \le \frac{1}{10}$.
Then, we have
\begin{align*}
  \opnorm{J_x\sigma(x) - J_x\sigma(x')}
  & = \max_i | \sigma'(x)_i - \sigma'(x')_i|
    \le \max_i \sup_{x,x'} | \sigma'(x)_i - \sigma'(x')_i| \\
  & \le \max_i \sup_{y}|{\sigma''(y)_i}| \linf{x-x'} 
  \le \frac{1}{10}\ltwo{x-x'}
  \end{align*}
 \end{example}

\begin{example}[ELU]
  Let $\sigma: \R^d \to \R^d$ be the ELU activation with scale parameter
  $\alpha = 1$:
  \begin{equation*}
    \sigma(x)_j = x_j \indic{x_j\ge 0}
    + (e^{x_j}-1) \indic{x_j < 0}.
  \end{equation*}
  Then, $\sigma$ satisfies Assumption~\ref{assumption:activation} with
  $L^0 = L^1 = 1$. %
\end{example}

The following lemma bounds the smoothness of the map
$x \mapsto F_l(\theta; x)$. To ease notation,  define
\begin{equation}
  \label{eq:lipschitzconsts}
  \alpha_l(\theta) \defeq \prod_{j=1}^l L_j^0 \opnorm{\theta_j}
   ~~~~~\mbox{and}~~~~~ \beta_l(\theta) \defeq \alpha_l(\theta)
    \sum_{j=1}^l \left\{ \frac{L_j^1}{(L_j^0)^2}\alpha_j(\theta)
  \right\}.
\end{equation}
Assumption~\ref{assumption:activation} guarantees $\beta_l(\theta)$ is well-defined, as we consider non-degenerate networks with $L_j^0>0$. 
\begin{proposition}
  \label{proposition:deep-net-smooth}
  Let Assumption~\ref{assumption:activation} hold. For all $l = 1, \ldots, L$,
  $F_l(\theta; \cdot): \R^{d_{0, I}} \to \R^{d_{l, O}}$ is
  $\alpha_l(\theta)$-Lipschitz w.r.t. the $\ell_2$-norm
  $\ltwo{\cdot}$, and its Jacobian
  $J_x F_l(\theta; \cdot): \R^{d_{0, I}} \to \R^{d_{l, O} \times d_{0, I}}$ is
  $\beta_l(\theta)$-Lipschitz w.r.t. $(\opnorm{\cdot}, \ltwo{\cdot})$, where $\opnorm{\cdot}$ is
  the $\ell_2$-operator (spectral) norm.
\end{proposition}
\noindent See Section~\ref{section:proof-of-deep-net-smooth} for the proof. We can use these bounds to derive our desired bound on the
Lipschitz constant of the final output of the network~\eqref{eqn:softmax}. We defer its proof
to Section~\ref{section:proof-of-softmax-smooth}.
\begin{corollary}
  \label{corollary:softmax-smooth}
Let Assumption~\ref{assumption:activation} hold. For the softmax loss
  defined by expression~\eqref{eqn:softmax},
  $x \mapsto \nabla_x \loss(\theta; (x, y))$ is
  $\left (\sqrt{2}\beta_L(\theta) +\alpha_L(\theta)^2 \right )$-Lipschitz w.r.t. $(\opnorm{\cdot}, \ltwo{\cdot})$.
\end{corollary}

\section{Experiments}
\label{sec:experiments}

Our technique for distributionally robust optimization with adversarial training
extends beyond supervised learning. To that end,
we present empirical evaluations on
supervised and reinforcement learning tasks where we compare performance with
empirical risk minimization (ERM) and, where appropriate, models trained with
the fast-gradient method~\eqref{eqn:fgsm}
(FGM)~\citep{ASgoodfellow2015explaining}, its iterated variant
(IFGM)~\citep{ASkurakin2016adversarial}, and the projected-gradient method
(PGM)~\citep{MadryMaScTsVl17}. PGM augments stochastic gradient steps
for the parameter $\theta$ with projected gradient ascent
over $x \mapsto \loss(\theta; x, y)$, iterating (for data point $x_i, y_i$)
\begin{equation}
  \label{eqn:madry}
  \Delta x_i^{t+1}(\theta) := 
  \argmax_{\norm{\eta}_p \le \epsilon}
  \{\nabla_x \loss(\theta; x_i^t, y_i)^T\eta\}
  ~~ \mbox{and} ~~
  x_i^{t+1} \defeq \Pi_{\mathcal{B}_{\epsilon, p}(x_i^t)}\left \{ x_i^t
    + \alpha_t\Delta x_i^t(\theta) \right \}
\end{equation}
for $t = 1, \ldots, T_{\rm adv}$, where $\Pi$ denotes projection onto
$\mathcal{B}_{\epsilon, p}(x_i):=\{x: \norm{x - x_i}_p \le \epsilon\}$.  

The adversarial training literature (e.g.~\citet{ASgoodfellow2015explaining})
usually considers $\linf{\cdot}$-norm attacks, which allow imperceptible
perturbations to all input features.  Since in most scenarios it is reasonable
to defend against weaker adversaries that instead perturb influential features
more, we consider training against $\ltwo{\cdot}$-norm attacks. Namely, we use
the squared Euclidean cost for the feature vectors
$c_x(x,x') \defeq \ltwo{x - x'}^2$ and define the overall cost as the
covariate-shift adversary~\eqref{eqn:supervised-cost} for WRM
(Algorithm~\ref{alg:thealg}), and we use $p = 2$ for FGM, IFGM, PGM training
in all experiments; we still test against adversarial perturbations with
respect to the norms $p = 2, \infty$. We use $T_{\rm adv} = 15$ iterations for
all iterative methods (IFGM, PGM, and WRM) in training and attacks.

In Section~\ref{sec:toy}, we visualize differences between our approach and
ad-hoc methods to illustrate the benefits of certified robustness. In
Section~\ref{sec:mnist} we consider supervised learning problems  where we adversarially perturb the test data for the MNIST and Stanford Dogs~\cite{khosla2011novel} datasets. Finally, we consider a reinforcement learning problem in Section~\ref{sec:cartpole}, where the Markov
decision process used for training differs from that for testing.

WRM enjoys the theoretical guarantees of Sections~\ref{sec:approach}
and~\ref{sec:generalization-robustness} for large $\gamma$, but for small $\gamma$ (large
adversarial budgets), WRM becomes a heuristic like other methods. In
Appendix~\ref{sec:large-adversary}, we compare WRM with other methods on
attacks with large adversarial budgets. In Appendix~\ref{section:sup-norm}, we
further compare WRM---which is trained to defend against
$\ltwo{\cdot}$-adversaries---with other heuristics trained to defend against
$\linf{\cdot}$-adversaries. WRM matches or outperforms other
heuristics against imperceptible attacks, while it
underperforms for attacks with large adversarial budgets.

\subsection{Visualizing the benefits of certified robustness}\label{sec:toy}
For our first experiment, we generate synthetic data $Z = (X, Y) \sim P_0$ by
$X_i \simiid \normal(0_2, I_2)$ with labels
$Y_i = \sign(\ltwo{x} - \sqrt{2})$, where $X \in \R^2$ and $I_2$ is the
identity matrix in $\R^2$. Furthermore, to create a wide margin separating the
classes, we remove data with $\ltwo{X} \in (\sqrt{2}/1.3, 1.3\sqrt{2})$. We
train a small neural network with 2 hidden layers of size 4 and 2 and either
all ReLU or all ELU activations between layers, comparing our approach (WRM)
with ERM and the 2-norm FGM. For our approach we use
$\gamma=2$, and to make fair comparisons with FGM we use
\begin{equation}
  \label{eqn:fgm-fair}
  \epsilon^2 = \what{\rho}_n(\theta_{\rm WRM})
  = W_c(P_n^*(\theta_{ \rm WRM}), \emp)
  = \E_{\emp}[c(T(\theta_{ \rm WRM}, Z),Z)],
\end{equation}
\ifdefined\useorstyle
for the fast-gradient perturbation magnitude $\epsilon$, where
$\theta_{\rm WRM}$ is the output of Algorithm~\ref{alg:thealg} (for ELU activations with scale parameter $1$, $\gamma=2$ makes problem
  \eqref{eqn:inner-sup} strongly concave over the training data. ReLUs have no guarantees, but we use 15 gradient steps
  with stepsize $1/\sqrt{t}$ for both activations). 
\else
for the fast-gradient perturbation magnitude $\epsilon$, where
$\theta_{\rm WRM}$ is the output of Algorithm~\ref{alg:thealg}.\footnote{For ELU activations with scale parameter $1$, $\gamma=2$ makes problem
  \eqref{eqn:inner-sup} strongly concave over the training data. ReLUs have no guarantees, but we use 15 gradient steps
  with stepsize $1/\sqrt{t}$ for both activations.}
\fi

Figure~\ref{fig:toy} illustrates the classification boundaries for the three
training procedures over the ReLU-activated (Figure~\ref{fig:toy}(a)) and
ELU-activated (Figure~\ref{fig:toy}(b)) models. Since $70\%$ of the data are
of the {\color{matlab_blue} blue} class ($\ltwo{X} \le \sqrt{2}/1.3$),
distributional robustness favors pushing the classification boundary outwards;
intuitively, adversarial examples will come from pushing {\color{matlab_blue}
  blue} points outwards across the boundary. ERM and FGM suffer from
sensitivities to various regions of the data, as evidenced by the lack of
symmetry in their classification boundaries. For both activations, WRM pushes
the classification boundaries further outwards than ERM or FGM.  However, WRM
with ReLUs still suffers from sensitivities (e.g.\ radial asymmetry in the
classification surface) due to the lack of robustness guarantees. WRM with
ELUs provides a certified level of robustness, yielding an axisymmetric
classification boundary that hedges against adversarial perturbations in all
directions.

\begin{figure}[!!t]
\begin{minipage}{0.45\columnwidth}
\centering
\subfigure[ReLU model]{\includegraphics[width=0.85\textwidth]{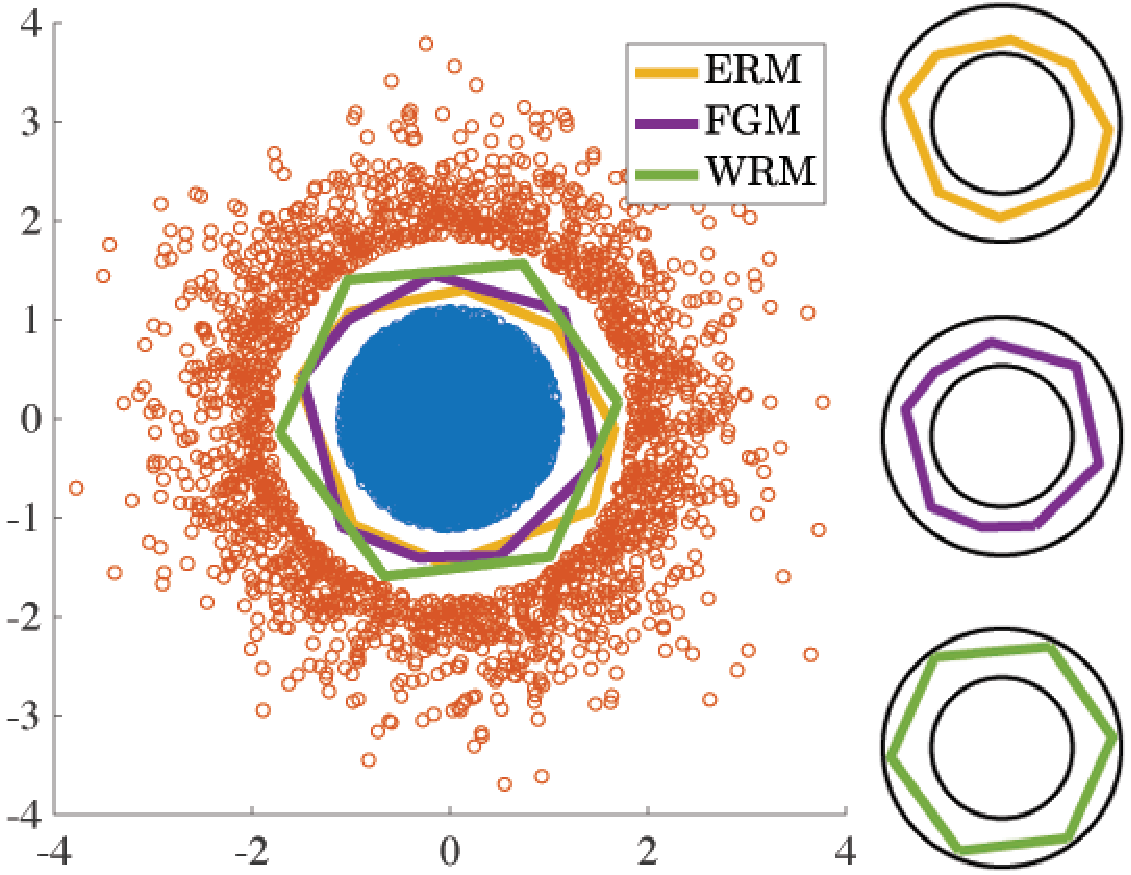}}%
\end{minipage}
\centering
\begin{minipage}{0.45\columnwidth}%
\centering
\subfigure[ELU model]{\includegraphics[width=0.85\textwidth]{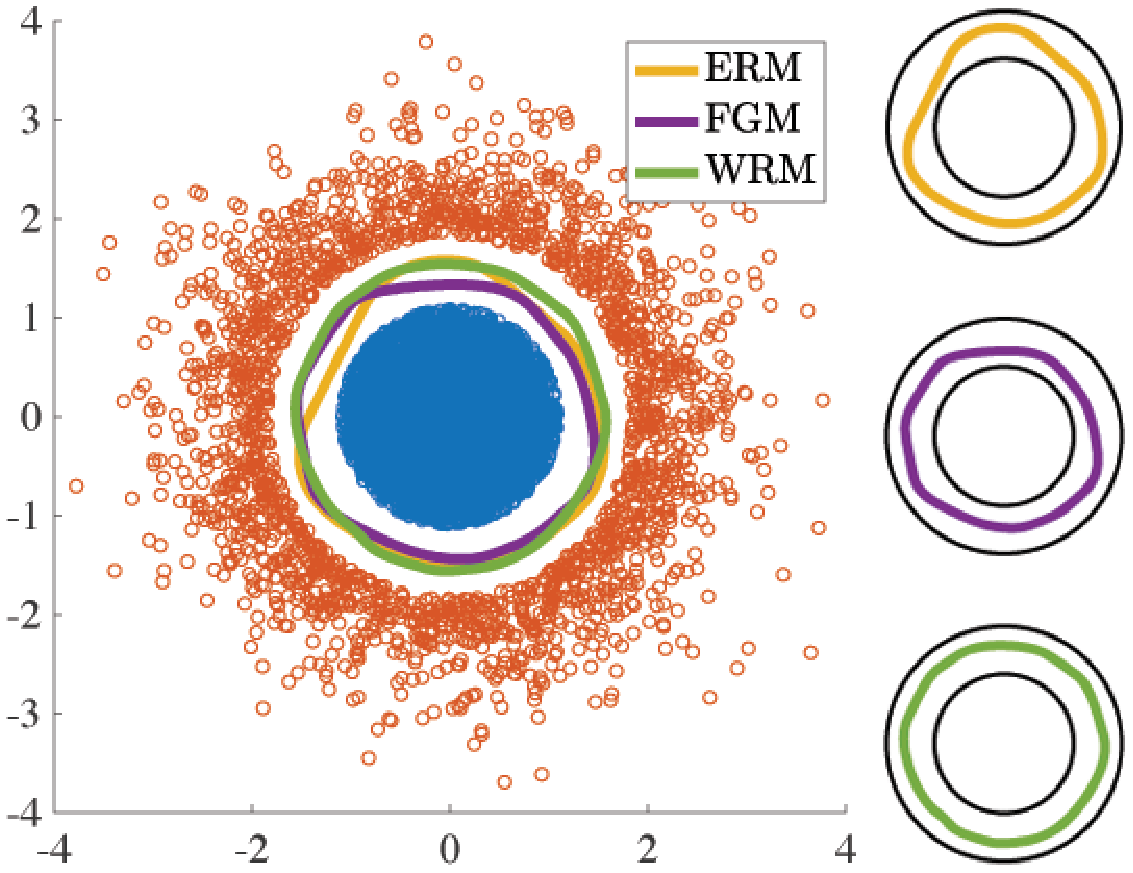}}%
\end{minipage}
\centering
\caption[]{\label{fig:toy}Experimental results on synthetic data. Training data are shown in {\color{matlab_blue} blue} and {\color{matlab_red} red}. Classification boundaries are shown in {\color{matlab_yellow} yellow}, {\color{matlab_purple} purple}, and {\color{matlab_green} green} for ERM, FGM, and WRM respectively. The boundaries are shown with the training data as well as separately with the true class boundaries.}
\end{figure}
\begin{figure}[!!t]
\begin{minipage}{0.45\columnwidth}
\centering
\subfigure[Synthetic data]{\includegraphics[width=0.85\textwidth]{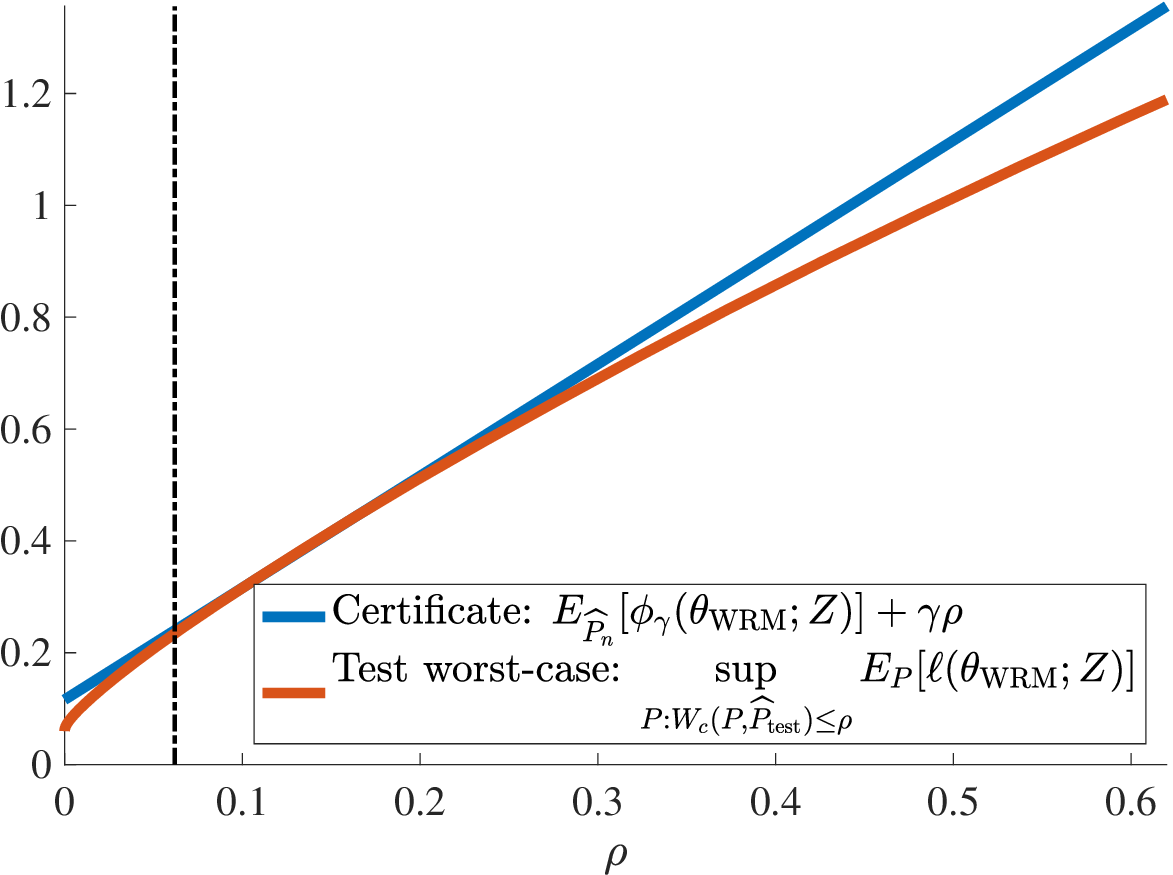}}%
\end{minipage}
\centering
\begin{minipage}{0.45\columnwidth}%
  \centering
  \subfigure[MNIST]{\includegraphics[width=0.85\textwidth]{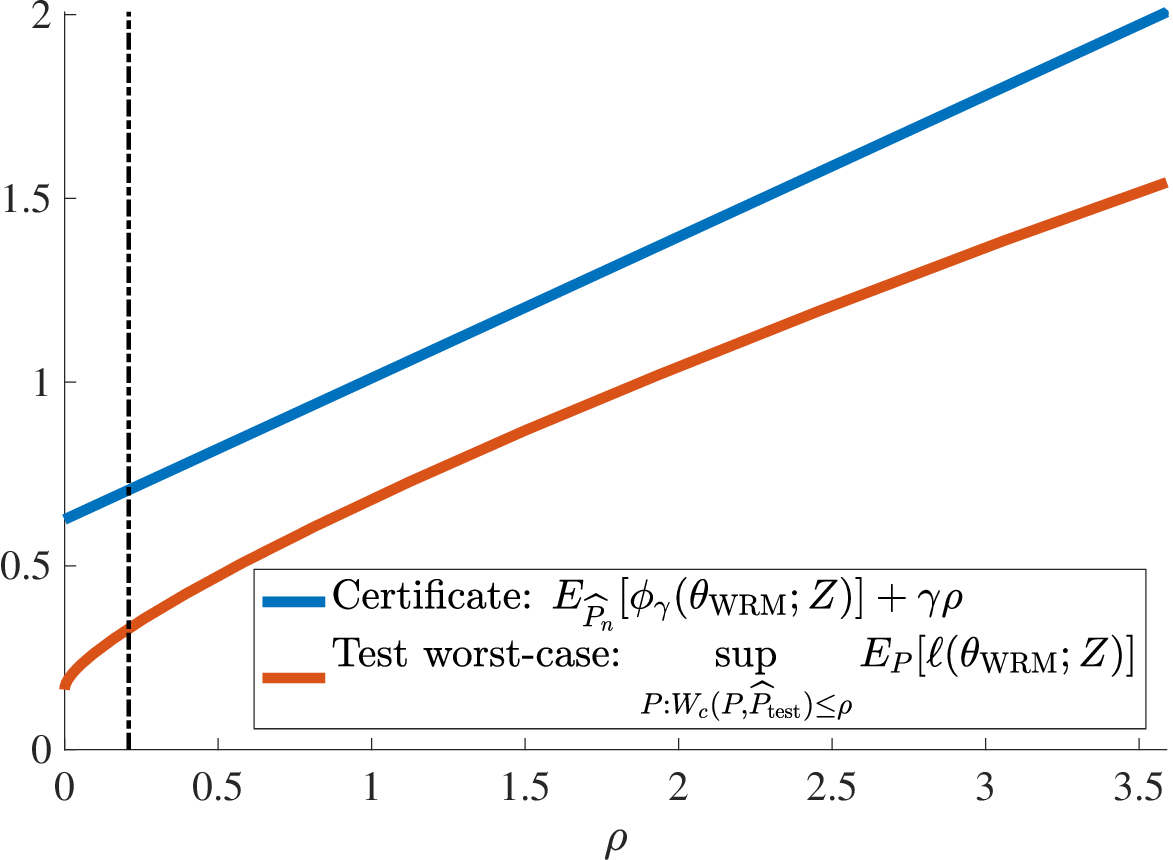}}%
\end{minipage}
\centering
\caption[]{\label{fig:loss-bound} Empirical comparison between certificate of robustness~\eqref{eqn:robustness-any-rho} ({\color{matlab_blue} blue}) and test
  worst-case performance ({\color{matlab_red} red}) for experiments with (a) synthetic data and (b) MNIST. We omit the certificate's error term $\epsilon_n(t)$. The vertical bar indicates the achieved level of robustness on the training set $\what{\tol}_n$($\theta_{\rm WRM}$).}
\end{figure}

Recall that our \emph{certificates of robustness} on the worst-case
performance given in Theorem~\ref{theorem:robustness} applies for any level of
robustness $\tol$. In Figure~\ref{fig:loss-bound}(a), we plot our
certificate~\eqref{eqn:robustness-any-rho} against the out-of-sample (test)
worst-case performance $\sup_{P: W_c(P, P_0) \le \tol}
\E_P[\loss(\theta;Z)]$ for WRM with ELUs. Since the worst-case loss is hard to evaluate
directly, we solve its Lagrangian relaxation~\eqref{eqn:lagrangian} for
different values of $\gamma_{\rm adv}$. For each $\gamma_{\rm adv}$, we
consider the distance to adversarial examples in the test dataset
\begin{equation}
  \label{eqn:rho-hat-test}
  \what{\rho}_{\rm test}(\theta) :=
  \E_{\what{P}_{\rm test}}[c({T_{\gamma_{\rm adv}}(\theta, Z), Z})],
\end{equation}
where $\what{P}_{\rm test}$ is the test distribution,
$c(z,z'):=\| x-x'\|_2^2 + \infty \cdot \indic{y \neq y'}$ as before, and
$T_{\gamma_{\rm adv}} (\theta, Z) = \argmax_{z} \{ \loss(\theta; z) -
\gamma_{\rm adv} c({z, Z}) \}$ is the adversarial perturbation of $Z$ (Monge
map) for the model $\theta$. The worst-case losses on the test dataset are
then given by
\begin{equation*}
  \E_{\what{P}_{\rm test}}[\phi_{\gamma_{\rm adv}}(\theta_{\rm WRM}; Z)] +
  \gamma_{\rm adv} \what{\rho}_{\rm test}(\theta_{\rm WRM}) = \sup_{P: W_c(P,
    P_{\rm test}) \le \what{\tol}_{\rm test}(\theta_{\rm WRM})}
  \E_{P}[\loss(\theta_{\rm WRM}; Z)].
\end{equation*}
As anticipated, our certificate is almost tight near the achieved level of
robustness $\what{\tol}_n(\theta_{\rm WRM})$ for WRM~\eqref{eqn:tol-n} and
provides a performance guarantee even for other values of $\tol$.

\subsection{Learning a more robust classifier}\label{sec:mnist}
We now consider two real-world supervised-learning benchmarks. In the first,
we use our adversarial training method on the MNIST dataset in an unprincipled
manner: we use a fixed level $\gamma$ but without prescribing this level to be
such that the robust surrogate~\eqref{eqn:inner-sup} is concave. For the more
realistic Stanford Dogs dataset~\cite{khosla2011novel}, we first present
results where we ran Algorithm~\ref{alg:thealg} on a semantic feature space
with $\gamma$ chosen large enough to satisfy bounds given in
Section~\ref{section:examples}. We complement these experiments with results
where we ran Algorithm~\ref{alg:thealg} in an unprincipled manner on raw pixel
perturbations, analogous to our MNIST results.

\subsubsection{The MNIST dataset}\label{sec:mnist-real}
\ifdefined\useorstyle
For the MNIST dataset, we train small neural network classifiers consisting of $8\times8$, $6\times6$, and
$5\times5$ convolutional filter layers with ELU activations followed by a
fully connected layer and softmax output. We train WRM with
$\gamma=0.04 \E_{\emp}[\norm{X}_2]$, and for the other methods we choose
$\epsilon$ as the level of robustness achieved by
WRM~\eqref{eqn:fgm-fair} (for this $\gamma$,
  $\phi_\gamma(\theta_{\rm WRM}; z)$ is strongly concave for $98\%$ of the
  training data). In the figures, we scale the budgets $1/\gamma_{\rm adv}$
and $\epsilon_{\rm adv}$ for the adversary with
$C_p:=\E_{\emp}[\norm{X}_p]$ (for the standard MNIST dataset,
  $C_2:=\E_{\emp}\norm{X}_2=9.21$ and
  $C_{\infty}:=\E_{\emp}\norm{X}_\infty=1.00$).
\else
For the MNIST dataset, we train a small neural network classifiers consisting of $8\times8$, $6\times6$, and
$5\times5$ convolutional filter layers with ELU activations followed by a
fully connected layer and softmax output. We train WRM with
$\gamma=0.04 \E_{\emp}[\norm{X}_2]$, and for the other methods we choose
$\epsilon$ as the level of robustness achieved by
WRM~\eqref{eqn:fgm-fair}.\footnote{For this $\gamma$,
  $\phi_\gamma(\theta_{\rm WRM}; z)$ is strongly concave for $98\%$ of the
  training data.} In the figures, we scale the budgets $1/\gamma_{\rm adv}$
and $\epsilon_{\rm adv}$ for the adversary with
$C_p:=\E_{\emp}[\norm{X}_p]$.\footnote{For the standard MNIST dataset,
  $C_2:=\E_{\emp}\norm{X}_2=9.21$ and
  $C_{\infty}:=\E_{\emp}\norm{X}_\infty=1.00$.}
 \fi

In Figure~\ref{fig:loss-bound}(b) we illustrate the validity of our
certificate of robustness~\eqref{eqn:robustness-any-rho} for the worst-case
test performance for arbitrary level of robustness $\rho$.  We see that our
certificate provides a performance guarantee for out-of-sample worst-case
performance. In Figure~\ref{fig:mnist}, we compare our method against
different adversarial training techniques; all methods achieve  $> 99\%$
test-set accuracy, implying there is little test-time penalty for the
robustness levels ($\epsilon$ and $\gamma$) used for training. It is thus
important to distinguish the methods' abilities to combat attacks.  We test
performance of the five methods (ERM, FGM, IFGM, PGM, WRM) under PGM
attacks~\eqref{eqn:madry} with respect to $2$- and $\infty$-norms. In
Figure~\ref{fig:mnist}(a) and (b), all adversarial methods outperform ERM, and
WRM offers more robustness even with respect to these PGM attacks. Training
with the Euclidean cost still provides robustness to $\infty$-norm fast
gradient attacks. We provide further evidence in Appendix
\ref{sec:more-attacks}.

In Figure~\ref{fig:mnist2}(a), we study stability of the loss surface with
respect to perturbations to inputs. We note that small values of
$\what{\tol}_{\rm test}(\theta)$, the distance to adversarial
examples~\eqref{eqn:rho-hat-test}, correspond to small magnitudes of
$\nabla_z \loss(\theta; z)$ in a neighborhood of the nominal input, which
ensures stability of the model. Figure~\ref{fig:mnist2}(a) shows that
$\what{\rho}_{\rm test}$ differs by orders of magnitude between the training
methods (models
$\theta = \theta_{\rm ERM}, \theta_{\rm FGM}, \theta_{\rm IFGM}, \theta_{\rm
  PGM}, \theta_{\rm WRM}$); the trend is nearly uniform over all
$\gamma_{\rm adv}$, with $\theta_{\rm WRM}$ being the most stable. Thus, we
see that our adversarial-training method defends against gradient-exploiting
attacks by reducing the magnitudes of gradients near the nominal input.

Figure~\ref{fig:mnist2}(b) presents qualitative examples that illustrate
the utility of our approach. For a single test datapoint, we adversarially
perturb the image until the model misclassifies it. We again
consider WRM attacks and we decrease $\gamma_{\rm adv}$ until each model
misclassifies the input. The original label is 8, whereas on the adversarial
examples IFGM predicts 2, PGM predicts 0, and the other models predict
3. WRM's ``misclassifications'' appear consistently reasonable to the human
eye (see Appendix \ref{sec:more-stability} for examples of other digits); WRM
defends against gradient-based exploits by learning a representation that
makes gradients point towards inputs of other classes. Together, Figures
~\ref{fig:mnist2}(a) and (b) depict our method's defense mechanisms to
gradient-based attacks: creating a  stable loss surface by reducing the
magnitude of gradients and improving their interpretability.

\begin{figure}[!!t]
\begin{minipage}{0.45\columnwidth}
\centering
\subfigure[Test error vs. $\epsilon_{\rm adv}$ for $\|\cdot\|_2$ attack]{\includegraphics[width=0.9\textwidth]{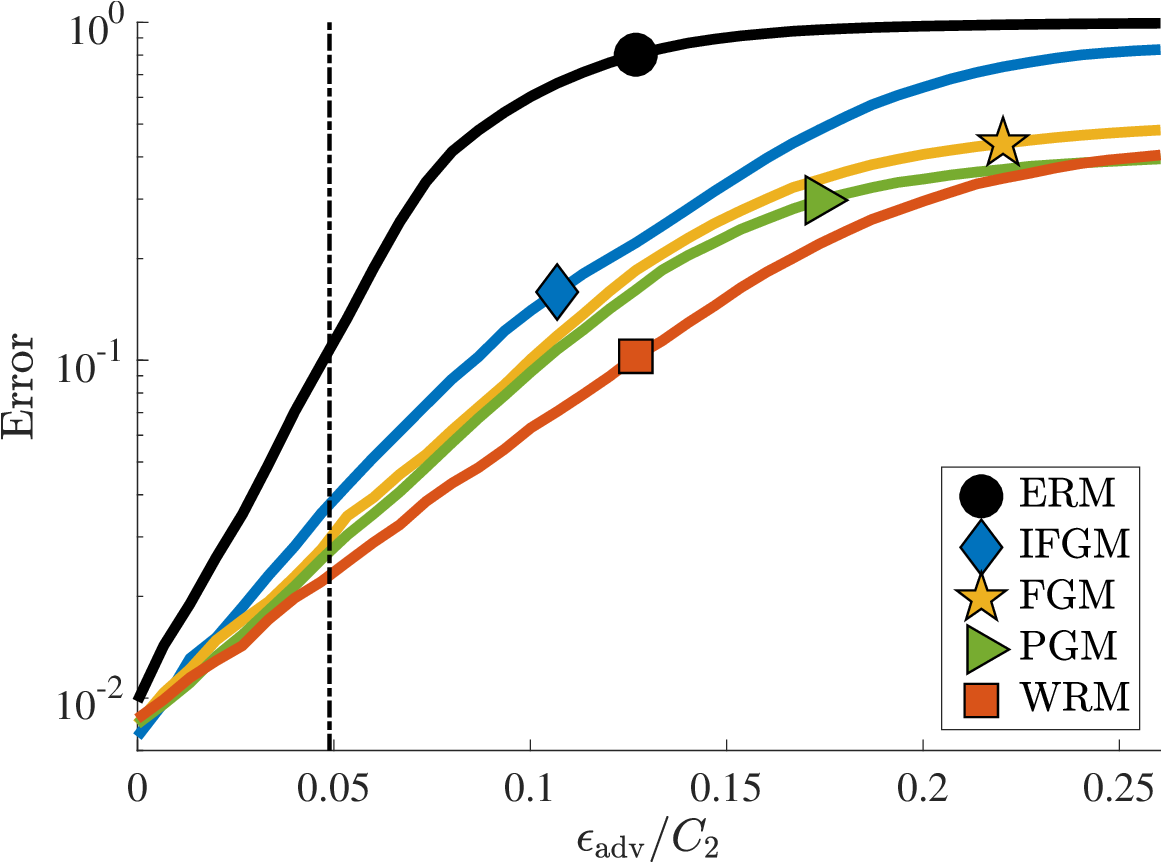}}%
\end{minipage}
\centering
\begin{minipage}{0.45\columnwidth}%
\centering
\subfigure[Test error vs. $\epsilon_{\rm adv}$ for $\|\cdot\|_{\infty}$ attack]{\includegraphics[width=0.9\textwidth]{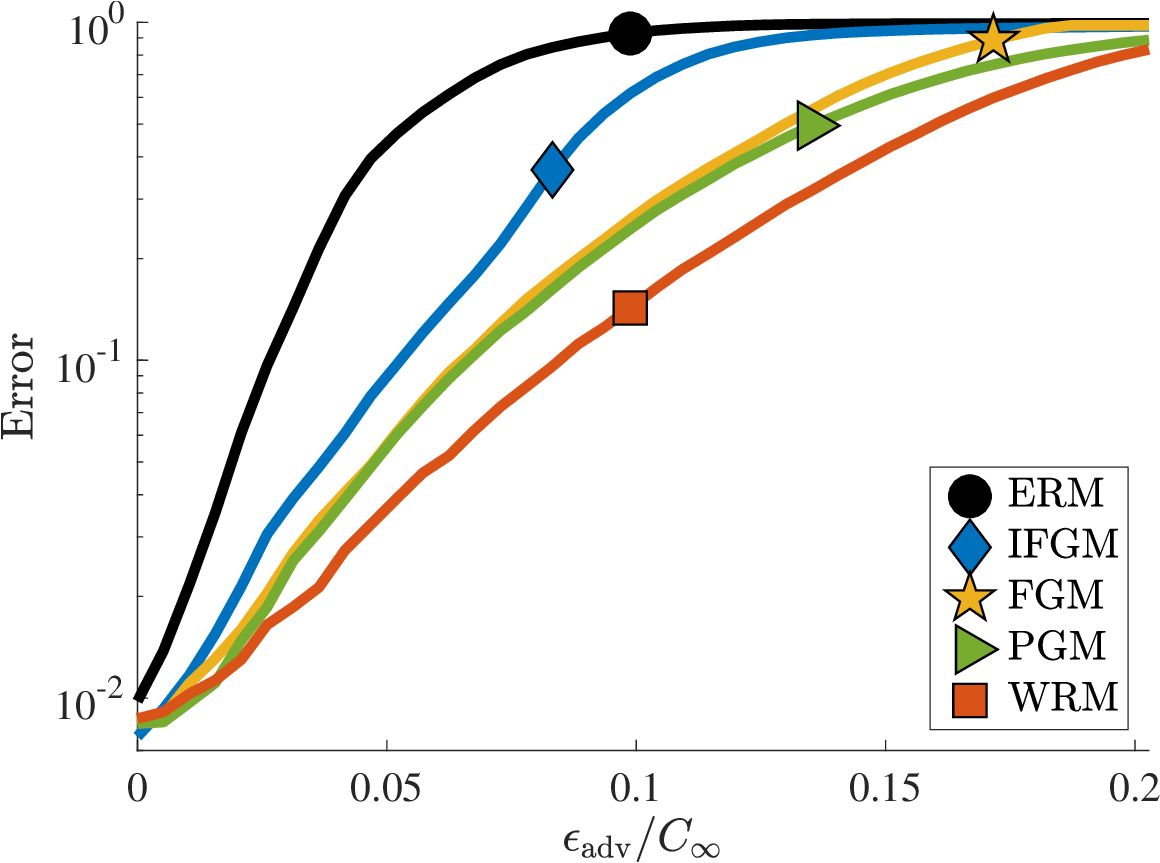}}%
\end{minipage}
\centering
\caption[]{\label{fig:mnist}PGM attacks on the MNIST
  dataset. (a) and (b) show test misclassification error vs. the adversarial
  perturbation level $\epsilon_{\rm adv}$ for the PGM attack with
  respect to Euclidean and $\infty$ norms respectively. The vertical bar
  in (a) indicates the perturbation level used for training the PGM, FGM, and IFGM models as well as the estimated radius
  $\sqrt{\what{\rho}_n(\theta_{\rm WRM})}$. For MNIST, $C_2=9.21$ and $C_{\infty}=1.00$. }
\end{figure}

\ifarxiv
\begin{figure}[!!t]
\begin{minipage}{0.45\columnwidth}
\centering
\subfigure[$\what{\rho}_{test}$ vs. $1/\gamma_{\rm adv}$]
{\includegraphics[width=0.9\textwidth]{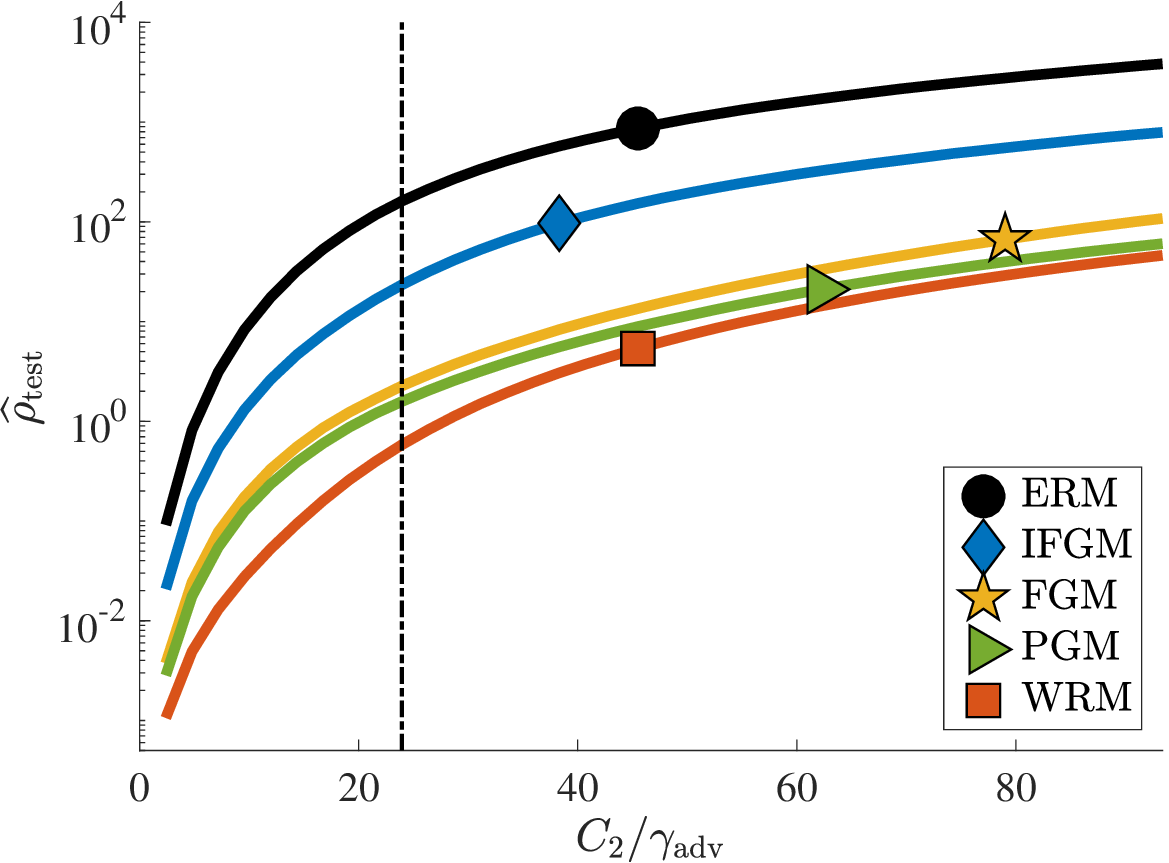}}%
\end{minipage}
\centering
\begin{minipage}{0.45\columnwidth}%
\centering
\subfigure[Perturbations on a test datapoint]{\includegraphics[width=0.9\textwidth]{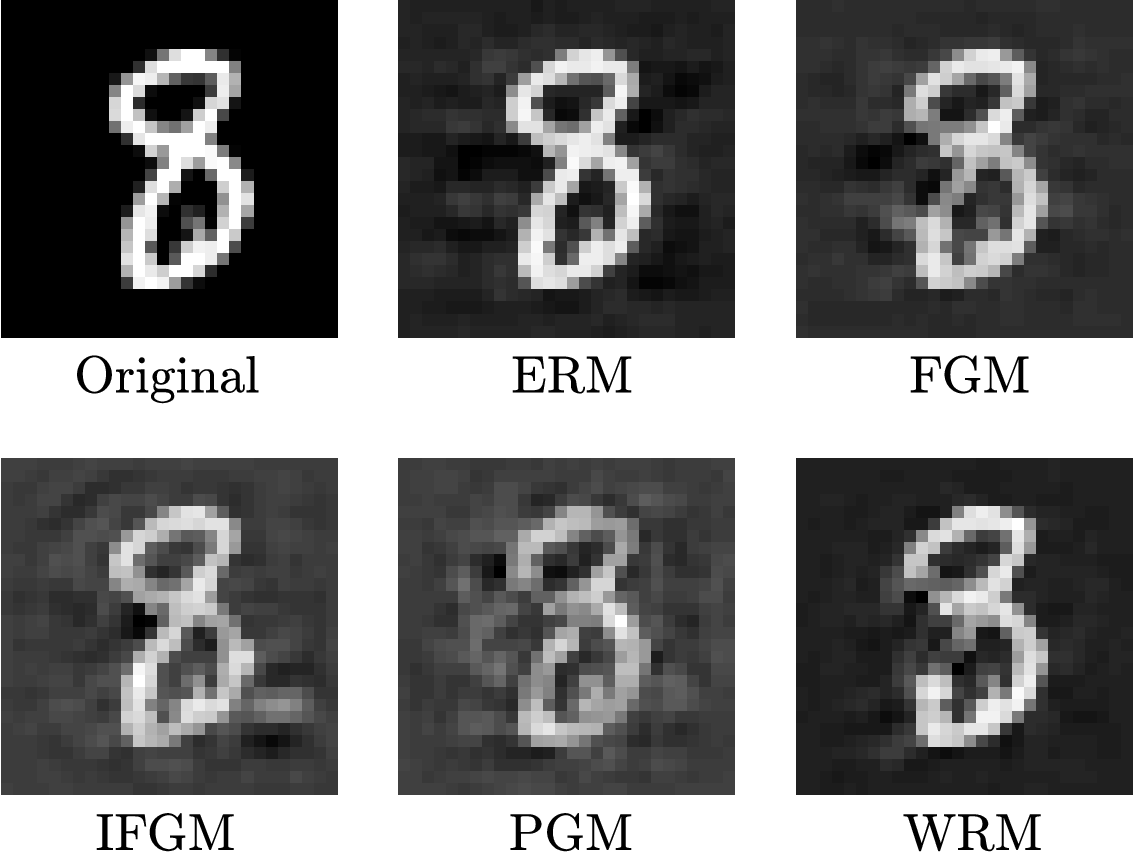}}%
\end{minipage}
\centering
\caption[]{\label{fig:mnist2}Stability of the loss surface. In (a), we show
  the average distance of the perturbed distribution $\what{\rho}_{\rm test}$
  for a given $\gamma_{\rm adv}$, an indicator of local stability to inputs
  for the decision surface. The vertical bar in (a) indicates the
  $\gamma$ we use for training WRM. In (b) we visualize the smallest WRM perturbation (largest $\gamma_{\rm adv}$) necessary to make a model misclassify a datapoint. More examples are in Appendix \ref{sec:more-stability}.}
\end{figure}
\else
\begin{figure}[!!t]
\begin{minipage}{0.45\columnwidth}
\centering
\subfigure[$\what{\rho}_{test}$ vs. $1/\gamma_{\rm adv}$]
{\includegraphics[width=0.85\textwidth]{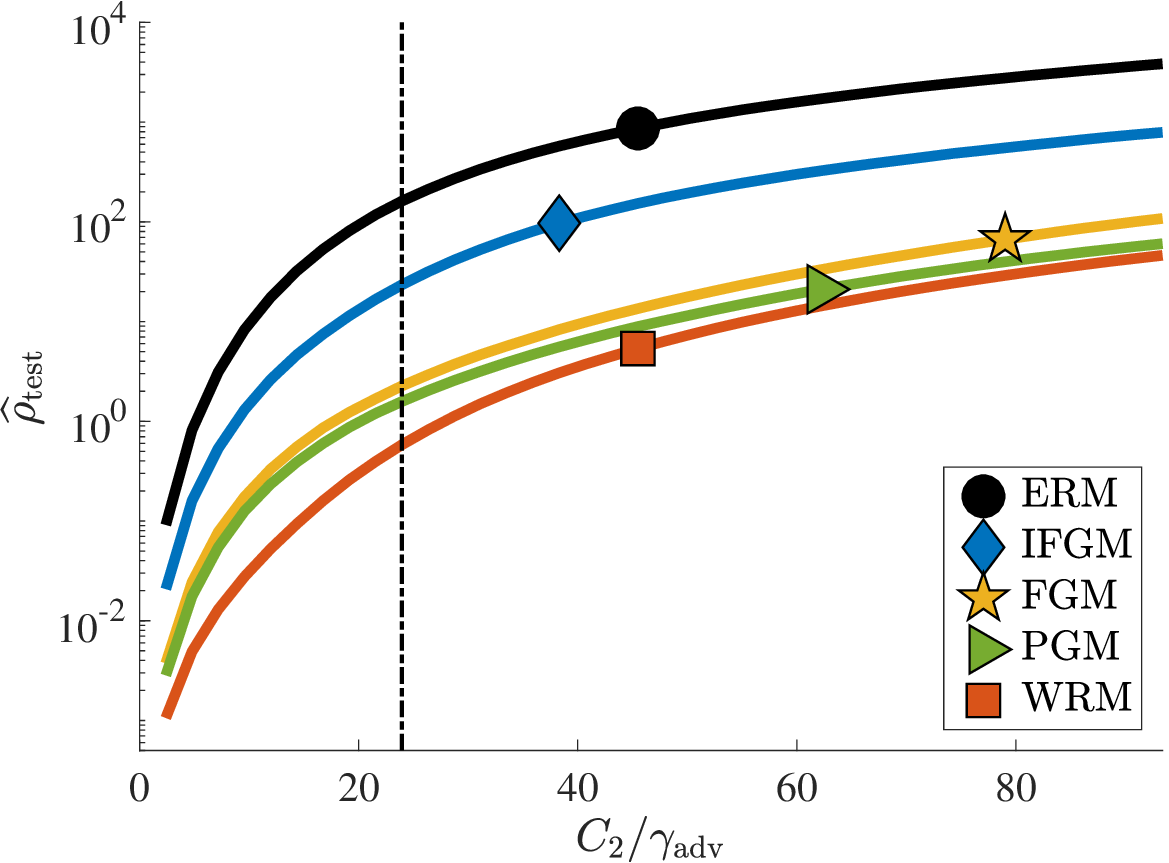}}%
\end{minipage}
\centering
\begin{minipage}{0.45\columnwidth}%
\centering
\subfigure[Perturbations on a test datapoint]{\includegraphics[width=0.85\textwidth]{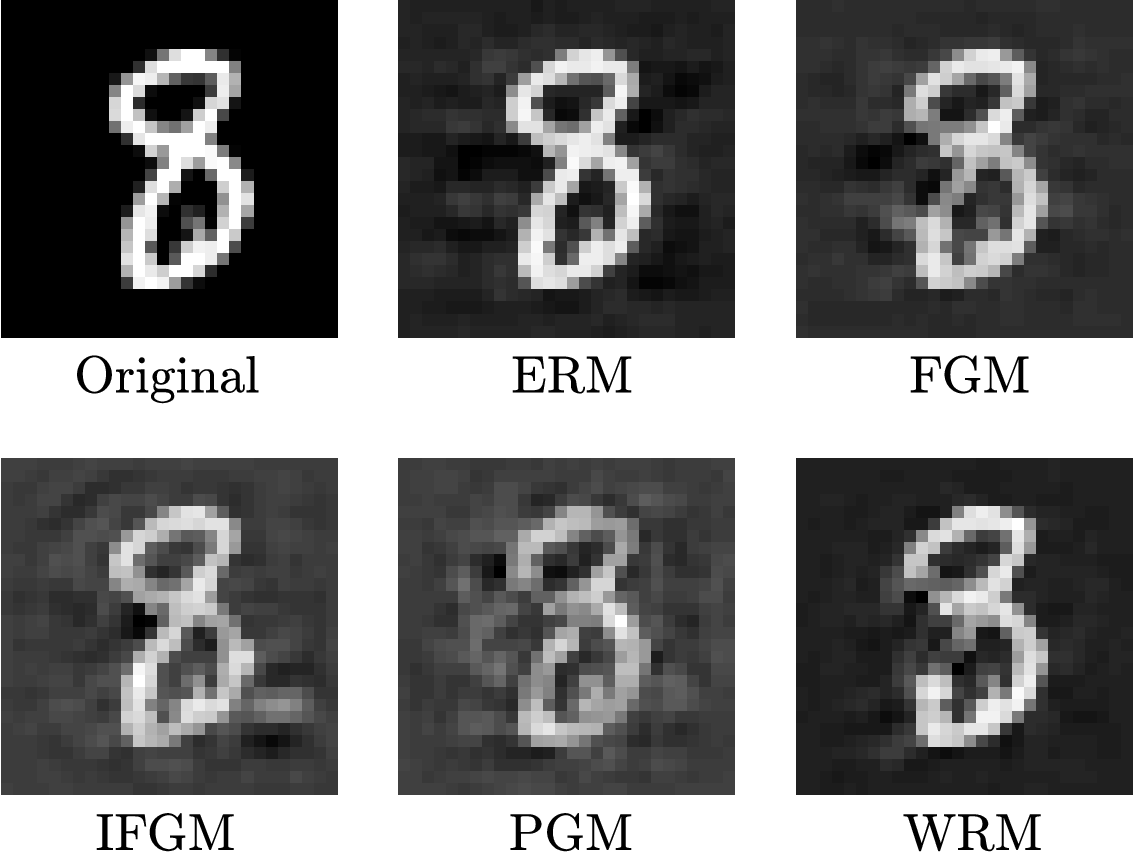}}%
\end{minipage}
\centering
\caption[]{\label{fig:mnist2}Stability of the loss surface. In (a), we show
  the average distance of the perturbed distribution $\what{\rho}_{\rm test}$
  for a given $\gamma_{\rm adv}$, an indicator of local stability to inputs
  for the decision surface. The vertical bar in (a) indicates the
  $\gamma$ we use for training WRM. In (b) we visualize the smallest WRM perturbation (largest $\gamma_{\rm adv}$) necessary to make a model misclassify a datapoint. More examples are in Appendix \ref{sec:more-stability}.}
\end{figure}
\fi

\subsubsection{The Stanford Dogs dataset}
\label{section:dogs}

We now test our approach in a more realistic classification scenario, where
our goal is to reliably classify images of dogs to one of $120$ breeds using
the Stanford Dogs dataset~\cite{khosla2011novel}. In our experiments, we use
$2000$ training images resized to $224 \times 224$ pixels, and use the
ResNet-50 network~\cite{HeZhReSu16} pre-trained on the ImageNet dataset for
initialization.

First, we begin with a principled experiment, where we choose $\gamma$ based
on concrete upper bounds derived in the previous section. Due to the large
size of the ResNet model, our concrete upper bounds on the smoothness of this
model become vacuous even when replace ReLU activations with its smooth
counterparts. Therefore, we begin this subsection by considering adversarial
perturbations in the learned semantic feature space given by the output of the
pre-trained ResNet on the large ImageNet dataset. We use the pre-trained
network as a fixed feature extractor, and only fine-tune the last few layers
following our principled adversarial training
procedure~Algorithm~\ref{alg:thealg}, with $\gamma$ chosen according to our
bounds in Section~\ref{section:examples}.

We pass the images through the pre-trained ResNet architecture and extract the
$7 \times 7 \times 2048$ semantic features (normalized to $[0, 1]$) before
they are fed into the output layer. Instead of passing the features through a
global average pooling layer followed by a softmax layer, as usually done when
fine-tuning a ResNet on a new dataset, we design a classifier with smooth
activation functions (ELUs). Specifically, we use a convolutional layer with
$1024$ convolutions of kernel size $7 \times 7$ followed by an ELU, a fully
connected layer with 120 outputs, and finally a softmax layer. For our
experiments, we set $\gamma=10^5$, observing that for this value we satisfy
have a smoothness upper bound (via the bound~\eqref{eq:lipschitzconsts}(b)) of
822.5. We set the number of iterations $T_{\rm adv}=20$ and the step size
$\eta=10.0$. Figure~\ref{fig:supremum} shows the convergence of the supremum
indicating that we have indeed solved the inner problem~\eqref{eqn:inner-sup},
which plots averages of the adversarial loss~\eqref{eqn:inner-sup}
$\phi_{\gamma}(\theta; z_i)$ during the first epoch of the training
procedure. In Figure~\ref{fig:dogs-theory}, we present classification errors
under both PGD and WRM attacks in the semantic feature space as in the last
subsection.

We now illustrate the usefulness of our approach even large-scale scenarios
where our conservative bounds on smoothness given in
Section~\ref{section:examples} become too loose. We fine-tune the whole
ResNet-50 network on the Stanford Dogs dataset with Algorithm~\ref{alg:thealg}
with perturbations on raw pixels. We use $\gamma=1.0$, and set the adversarial
budget $\epsilon$ for the other adversarial training algorithms as the level
of robustness achieved, as in the MNIST
experiment. Figure~\ref{fig:dogs-hacky} shows results over PGM
(Figure~\ref{fig:dogs-hacky}(a)) and WRM (Figure~\ref{fig:dogs-hacky}(b))
attacks. We see that, as with the MNIST dataset, WRM is competitive against
the other baselines over a full range of test perturbations.

\begin{figure}[!!t]
\centering
{\includegraphics[width=0.4\textwidth]{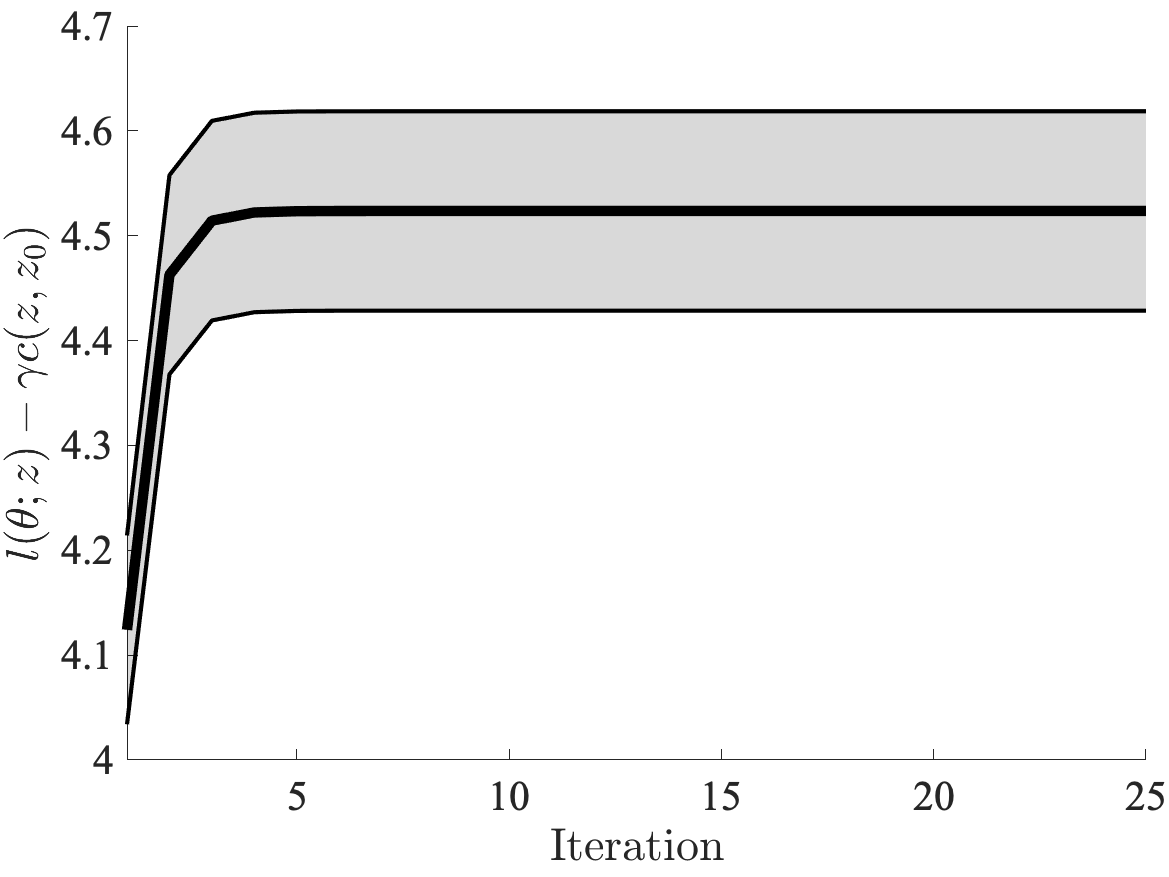}}%
\centering
\caption[]{\label{fig:supremum}Convergence of the inner problem~\eqref{eqn:inner-sup}. We show the mean and standard error across all datapoints of the loss $l(\theta; z) - \gamma c(z, z_0)$ with respect to iteration. The plot is shown for the the first epoch of training. }
\end{figure}

\begin{figure}[!!t]
\begin{minipage}{0.49\columnwidth}
\centering
\subfigure[Test error vs. $\epsilon_{\rm adv}$ for $\|\cdot\|_{\infty}$-PGM attack]
{\includegraphics[width=0.9\textwidth]{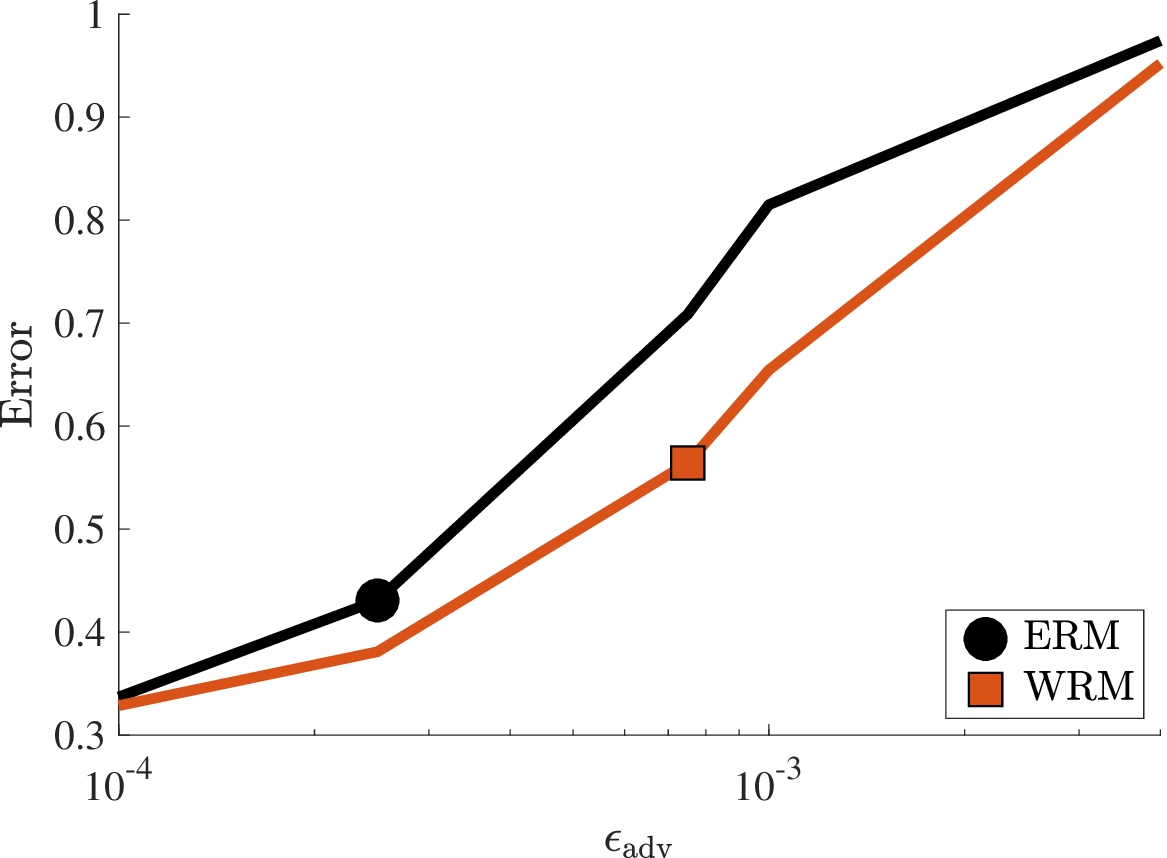}}%
\end{minipage}
\centering
\begin{minipage}{0.49\columnwidth}%
\centering
\subfigure[Test error vs. $1/\gamma_{\rm adv}$ for $\|\cdot\|_{2}$-WRM attack]{\includegraphics[width=0.9\textwidth]{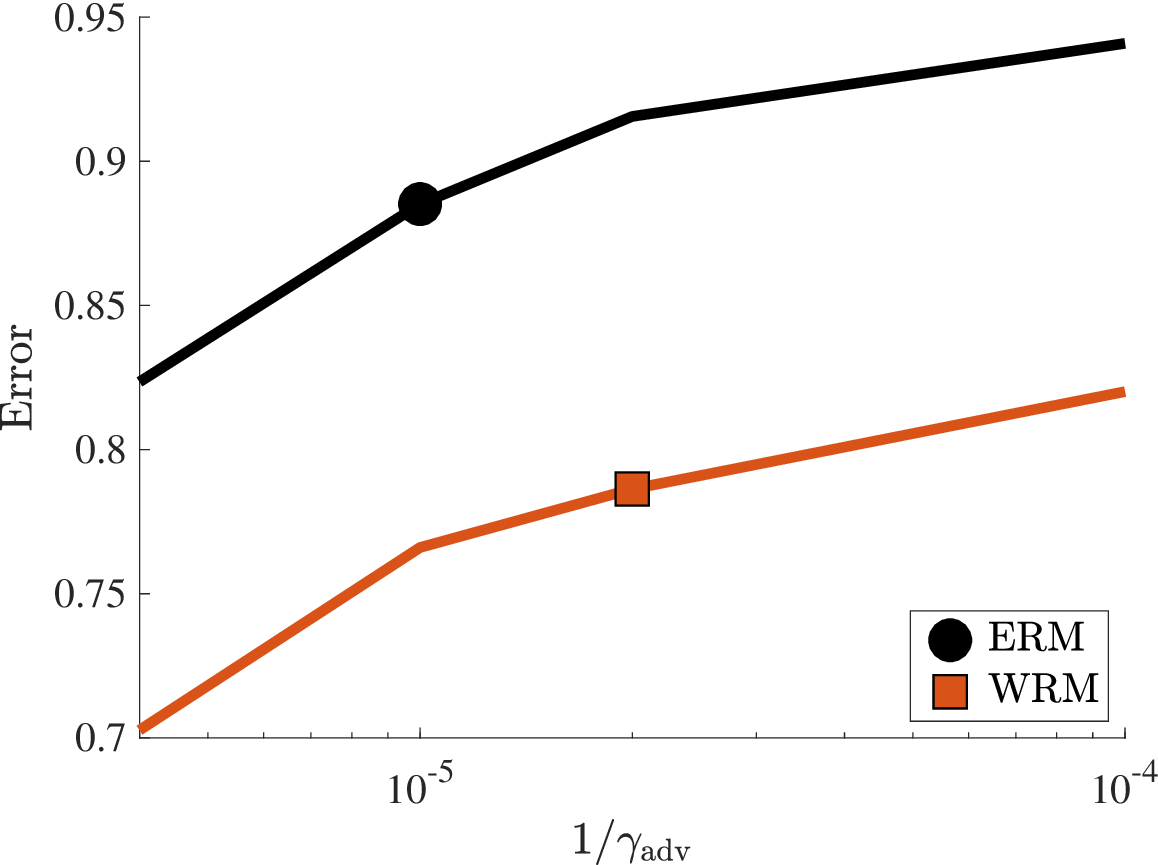}}%
\end{minipage}
\centering
\caption[]{\label{fig:dogs-theory}PGM and WRM attacks on the Stanford Dogs dataset.}
\end{figure}

\begin{figure}[!!t]
\begin{minipage}{0.49\columnwidth}
\centering
\subfigure[Test error vs. $\epsilon_{\rm adv}$ for $\|\cdot\|_{\infty}$-PGM attack]
{\includegraphics[width=0.9\textwidth]{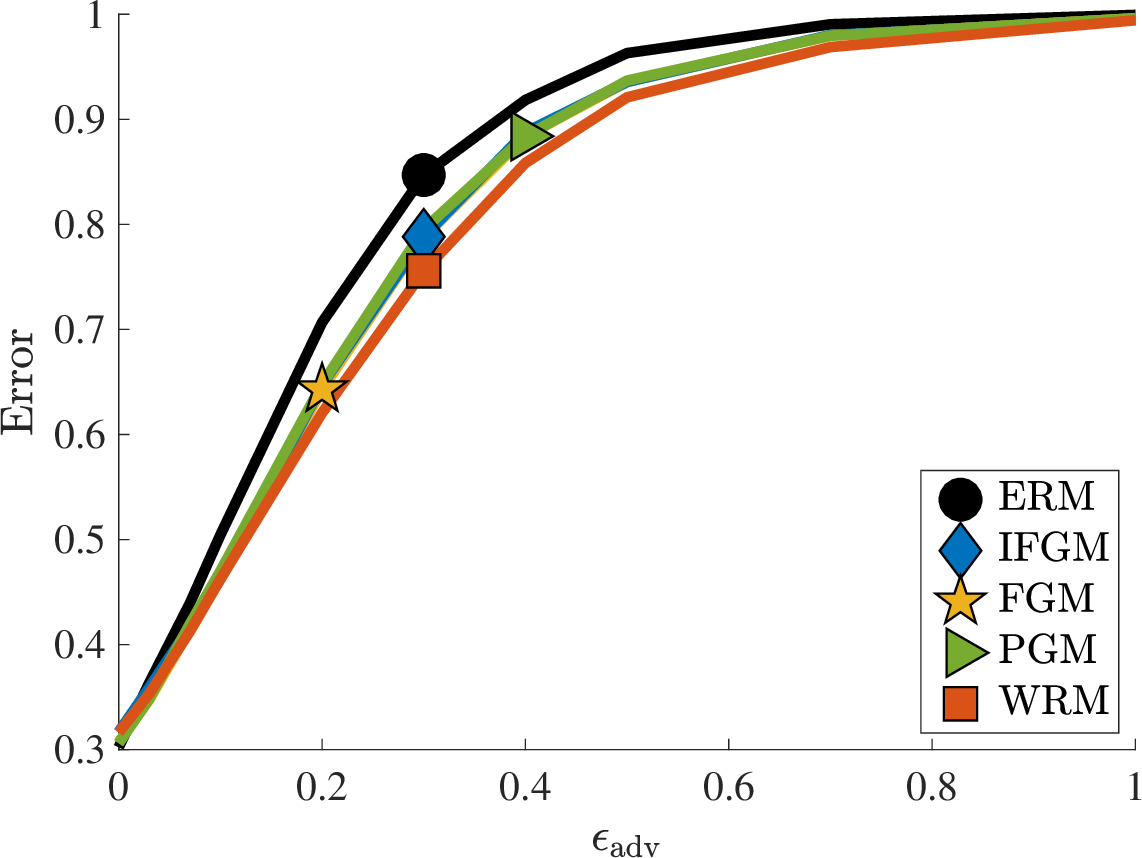}}%
\end{minipage}
\centering
\begin{minipage}{0.49\columnwidth}%
\centering
\subfigure[Test error vs. $1/\gamma_{\rm adv}$ for $\|\cdot\|_{2}$-WRM attack]{\includegraphics[width=0.9\textwidth]{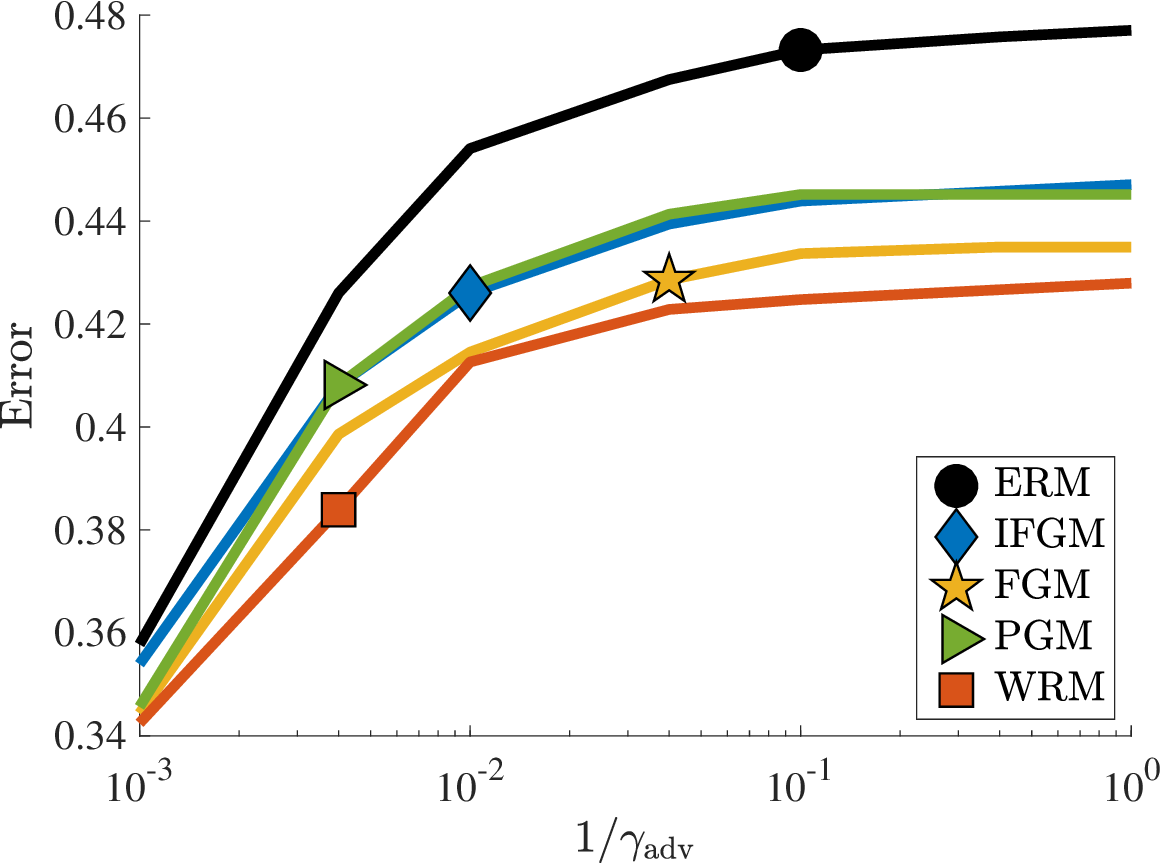}}%
\end{minipage}
\centering
\caption[]{\label{fig:dogs-hacky}PGM and WRM attacks on the Stanford Dogs dataset.}
\end{figure}

\subsection{Robust Markov decision processes}\label{sec:cartpole}

For our final experiments, we consider distributional robustness in the
context of Q-learning, a model-free reinforcement learning technique. We
consider Markov decision processes (MDP's) $(\mathcal S, \mathcal A, P_{sa}, r)$ with state space $\mathcal S$, action space $\mathcal A$, state-action transition
probabilities $P_{sa}$, and rewards $r: \mathcal S \to \R$. The goal of a
reinforcement-learning agent is to maximize (discounted) cumulative rewards
$\sum_t \lambda^t \E[ r(s^t)]$ (with discount factor $\lambda$); this is analogous to minimizing $\E_P[\loss(\theta; Z)]$ in supervised learning. Robust MDP's consider an ambiguity set $\mathcal{P}_{sa}$ for state-action
transitions. The goal is maximizing the worst-case realization $\inf_{P
  \in \mathcal{P}_{sa}} \sum_t \lambda^t \E_P[ r(s^t)]$, analogous to
problem \eqref{eqn:OG-problem}.

In a standard MDP, Q-learning learns a quality function
$Q: \mathcal{S} \times \mathcal{A} \to \R$ via the iterations
\begin{equation}\label{eqn:qlearning}
Q(s^t, a^t) \gets Q(s^t, a^t) + \alpha_t \left ( r(s^{t}) + \lambda \max_a Q(s^{t+1}, a) - Q(s^t,a^t) \right )
\end{equation}
such that $\argmax_a Q(s, a)$ is (eventually) the optimal action in
state $s$ to maximize cumulative reward. When the underlying
environment has a continuous state-space and we represent $Q$ with a
differentiable function (e.g. \cite{ASmnih2015human}), we can modify
the update~\eqref{eqn:qlearning} with an adversarial state perturbation to incorporate distributional robustness. Namely, we draw the nominal state-transition
update $\what{s}^{t+1} \sim p_{sa}(s^t, a^t)$, and proceed with the update
\eqref{eqn:qlearning} using the perturbation
\begin{equation}\label{eqn:qlearningthing}
  s^{t+1} \gets \argmin_{s} \left \{ r(s)  + \lambda \max_a Q(s, a)  + \gamma c(s, \hat s^{t+1}) \right \}.
\end{equation}
For large $\gamma$, we can again solve
problem~\eqref{eqn:qlearningthing} efficiently using gradient descent. This
 provides robustness to uncertainties in state-action
transitions. For tabular Q-learning, where we represent $Q$ only over
a discretized covering of the underlying state-space, we can
either neglect the second term in the update \eqref{eqn:qlearningthing} and,
after performing the update, round $s^{t+1}$ as usual, or we can perform
minimization directly over the discretized covering. In the former case,
since the update~\eqref{eqn:qlearningthing} simply modifies the state-action
transitions (independent of $Q$), standard results on convergence for
tabular Q-learning (e.g. \citet{ASszepesvari1999unified}) apply under these
adversarial dynamics.

We test our adversarial training procedure in the cart-pole
environment, where the goal is to balance a pole on a cart by moving the cart
left or right. The environment caps episode lengths to 400 steps and ends the
episode prematurely if the pole falls too far from the vertical or the cart
translates too far from its origin. We use reward
$r(\beta):=e^{-|\beta|}$ for the angle $\beta$ of the pole from the
vertical. We use a tabular representation for $Q$ with 30
discretized states for $\beta$ and 15 for its time-derivative $\dot{\beta}$
(we perform the update \eqref{eqn:qlearningthing} without the $Q$-dependent term). The action space is binary: push the cart left or right
with a fixed force.  Due to the nonstationary, policy-dependent radius for the Wasserstein ball, an analogous $\epsilon$ for the fast-gradient
method (or other variants) is not well-defined. Thus, we only compare with an
agent trained on the nominal MDP. We test both models with perturbations to the physical parameters: we shrink/magnify the pole's mass by 2, the pole's
length by 2, and the strength of gravity $g$ by 5. The
system's dynamics are such that the heavy, short, and strong-gravity
cases are more unstable than the original environment, whereas
their counterparts are less unstable.

Table 1 shows performance of trained models over the original and all
perturbed MDPs. Both models perform similarly over easier environments, but
the robust model greatly outperforms in harder environments. Interestingly, as
shown in Figure 5, the robust model also learns more efficiently than the
nominal model in the original MDP. We hypothesize that a potential side-effect
of robustness is that adversarial perturbations encourage better exploration
of the environment.

\begin{figure}[!!t]
\begin{minipage}{\textwidth}
\begin{minipage}{0.49\textwidth}
\centering
\ifdefined\useorstyle
{\scriptsize
\begin{tabular}{c|cc}
  Environment & Regular & Robust \\ \hline \hline
  Original & 399.7 $\pm$ 0.1 & 400.0 $\pm$ 0.0 \\
  \hline 
  \multicolumn{3}{c} {}  \\
     \multicolumn{3}{c} {Easier environments} \\
  Light & 400.0 $\pm$ 0.0 & 400.0 $\pm$ 0.0 \\
   Long & 400.0 $\pm$ 0.0 & 400.0 $\pm$ 0.0 \\
   Soft $g$ & 400.0 $\pm$ 0.0 & 400.0 $\pm$ 0.0 \\
   \hline 
  \multicolumn{3}{c} {}     \\
      \multicolumn{3}{c} {Harder environments} \\
  Heavy & 150.1 $\pm$ 4.7 & 334.0 $\pm$ 3.7 \\
  Short & 245.2 $\pm$ 4.8 & 400.0 $\pm$ 0.0 \\
  Strong $g$ & 189.8 $\pm$ 2.3 & 398.5 $\pm$ 0.3 \\
  \hline
  \hline
  \end{tabular}
  \captionof{table}{Episode length over 1000 trials\\(mean $\pm$ standard error)}
 }
\else
{\small
\begin{tabular}{c|cc}
  Environment & Regular & Robust \\ \hline \hline
  Original & 399.7 $\pm$ 0.1 & 400.0 $\pm$ 0.0 \\
  \hline 
  \multicolumn{3}{c} {}   \\ %
     \multicolumn{3}{c} {Easier environments} \\
  Light & 400.0 $\pm$ 0.0 & 400.0 $\pm$ 0.0 \\
   Long & 400.0 $\pm$ 0.0 & 400.0 $\pm$ 0.0 \\
   Soft $g$ & 400.0 $\pm$ 0.0 & 400.0 $\pm$ 0.0 \\
   \hline 
  \multicolumn{3}{c} {}   \\ %
      \multicolumn{3}{c} {Harder environments} \\
  Heavy & 150.1 $\pm$ 4.7 & 334.0 $\pm$ 3.7 \\
  Short & 245.2 $\pm$ 4.8 & 400.0 $\pm$ 0.0 \\
  Strong $g$ & 189.8 $\pm$ 2.3 & 398.5 $\pm$ 0.3 \\
  \hline
  \hline
  \end{tabular}
  \captionof{table}{Episode length over 1000 trials\\(mean $\pm$ standard error)}
 }
 \fi
\end{minipage}
\begin{minipage}{0.49\columnwidth}%
\centering
\includegraphics[width=0.87\textwidth]{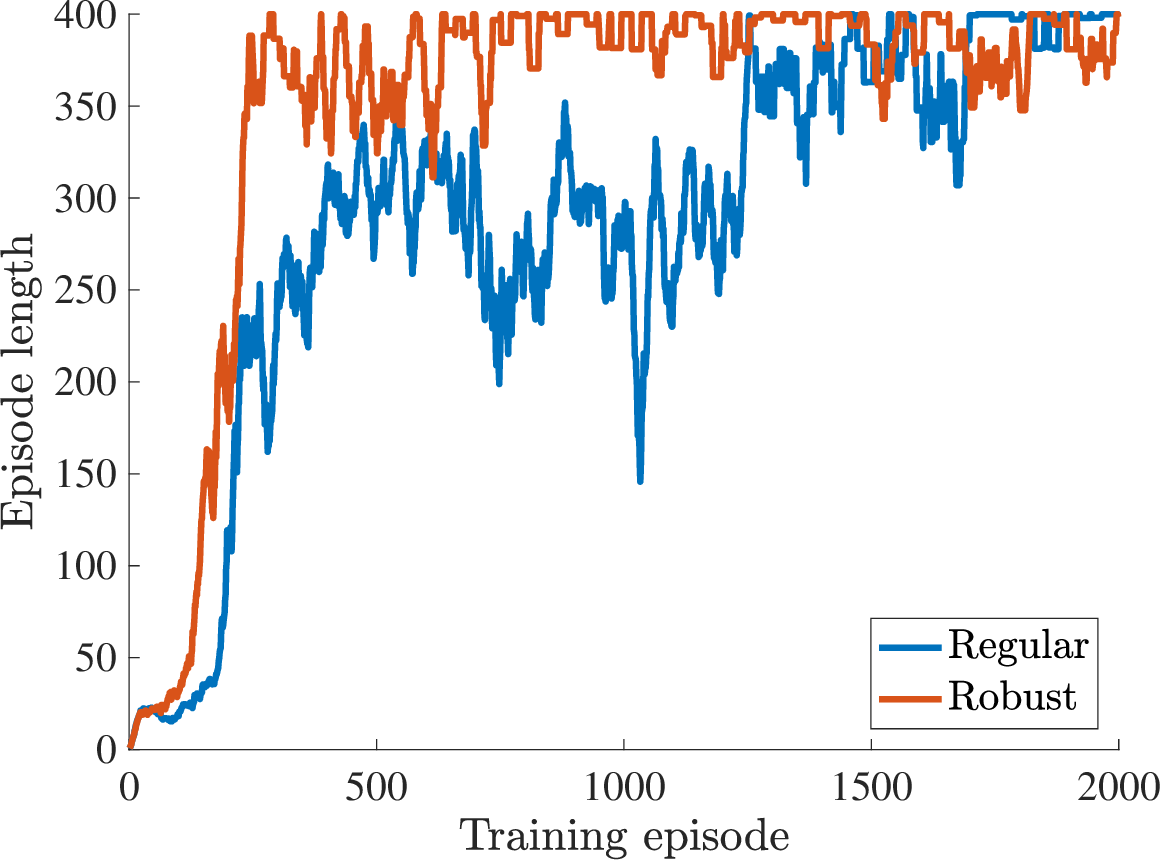}
\captionof{figure}{Episode lengths during training. The environment caps episodes to $400$ steps.}
\end{minipage}
\end{minipage}
\end{figure}
\section{Discussion and future work}

Explicit distributional robustness of the form~\eqref{eqn:constrained} is
intractable except in limited cases. We provide a principled method for
efficiently guaranteeing distributional robustness with a simple form of
adversarial data perturbation. Using only assumptions about the smoothness of
the loss function $\loss$, we prove that our method enjoys strong statistical
guarantees and fast optimization rates for a large class of problems. The
NP-hardness of certifying robustness for ReLU networks, coupled with our
empirical success and theoretical certificates for smooth networks in deep
learning, suggest that using smooth networks may be preferable if we wish to
guarantee robustness.  Empirical evaluations indicate that our methods are in
fact robust to perturbations in the data, and they match or outperform less-principled
adversarial training techniques. The major benefit of our approach is its
simplicity and wide applicability across many models and machine-learning
scenarios.

There remain many avenues for future investigation. Our optimization result
(Theorem~\ref{thm:convergence}) applies only for small values of robustness
$\rho$ and to a limited class of Wasserstein costs. Our
statistical guarantees (Theorems~\ref{theorem:robustness}
and~\ref{theorem:dist-concentration}) use $\linf{\cdot}$-covering numbers as
a measure of model complexity, which can become prohibitively large for deep
networks. In a learning-theoretic context, where the goal is to provide insight into
convergence behavior as well as comfort that a procedure will ``work''
given enough data, such guarantees are satisfactory, but this may not be
enough in security-essential contexts. This problem currently persists for most learning-theoretic
guarantees in deep learning, and the recent works of~\citet{BartlettFoTe17},
  \citet{DziugaiteRo17}, and  \citet{NeyshaburBhMcSr17} attempt to mitigate this shortcoming. Replacing our  covering number
arguments with more intricate notions such as margin-based
bounds~\citep{BartlettFoTe17} would extend the scope and usefulness of our
theoretical guarantees. Of course, the certificate~\eqref{eqn:certify-tests}
still holds regardless.

More broadly, this work focuses on small-perturbation attacks, and our
theoretical guarantees show that it is possible to efficiently build models
that provably guard against such attacks. Our method becomes another
heuristic for protection against attacks with large adversarial
budgets. Indeed, in the large-perturbation regime, efficiently training
certifiably secure systems remains an important open question. We believe
that conventional $\linf{\cdot}$-defense heuristics developed for image classification
do not offer much comfort in the large-perturbation/perceptible-attack
setting: $\linf{\cdot}$-attacks with a large budget can render images
indiscernible to human eyes, while, for example, $\lone{\cdot}$-attacks
allow a concerted perturbation to critical regions of the image. Certainly
$\linf{\cdot}$-attack and defense models have been fruitful in building a foundation
for security research in deep learning, but moving beyond them may be
necessary for more advances in the large-perturbation regime.
 \clearpage 
\subsubsection*{Acknowledgments}

We thank Jacob Steinhardt for valuable feedback. AS, HN, and JD were
partially supported by the SAIL-Toyota Center for AI Research. AS was also
partially supported by a Stanford Graduate Fellowship and a Fannie \& John
Hertz Foundation Fellowship. HN was partially supported by a Samsung
Fellowship. JD was partially supported by the National Science
Foundation award NSF-CAREER-1553086.

\setlength{\bibsep}{3pt}

\bibliographystyle{abbrvnat}

\ifdefined\useorstyle
\else

\newpage
\appendix
\section{Additional Experiments}
\label{sec:more-experiments}
\subsection{MNIST attacks}
\label{sec:more-attacks}
We repeat Figure \ref{fig:mnist} using FGM (tow row of Figure \ref{fig:mnist-fgm}) and IFGM (bottom row of Figure \ref{fig:mnist-fgm}) attacks. The same trends are evident as in Figure \ref{fig:mnist}.

\begin{figure}[!!h]
\begin{minipage}{0.49\columnwidth}
\centering
\subfigure[Test error vs. $\epsilon_{\rm adv}$ for $\|\cdot\|_2$-FGM attack]{\includegraphics[width=0.85\textwidth]{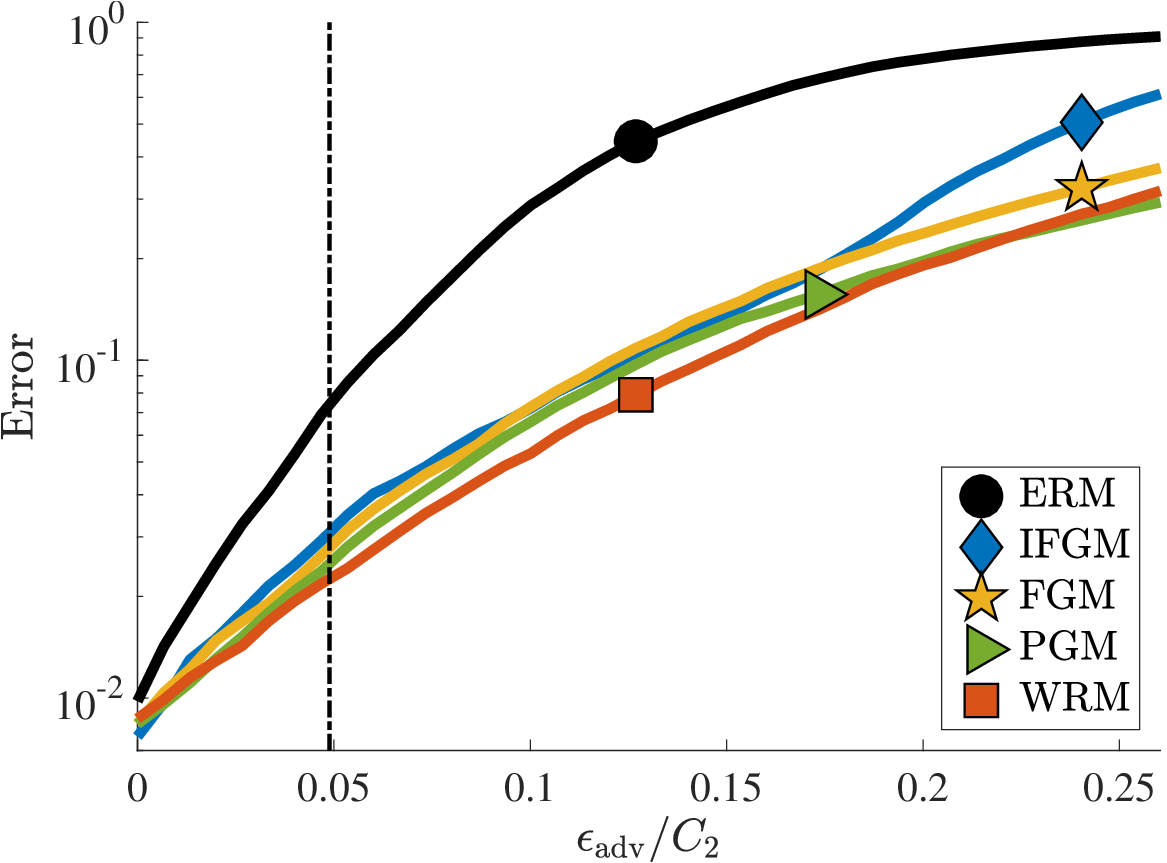}}%
\end{minipage}
\centering
\begin{minipage}{0.49\columnwidth}%
\centering
\subfigure[Test error vs. $\epsilon_{\rm adv}$ for $\|\cdot\|_{\infty}$-FGM attack]{\includegraphics[width=0.85\textwidth]{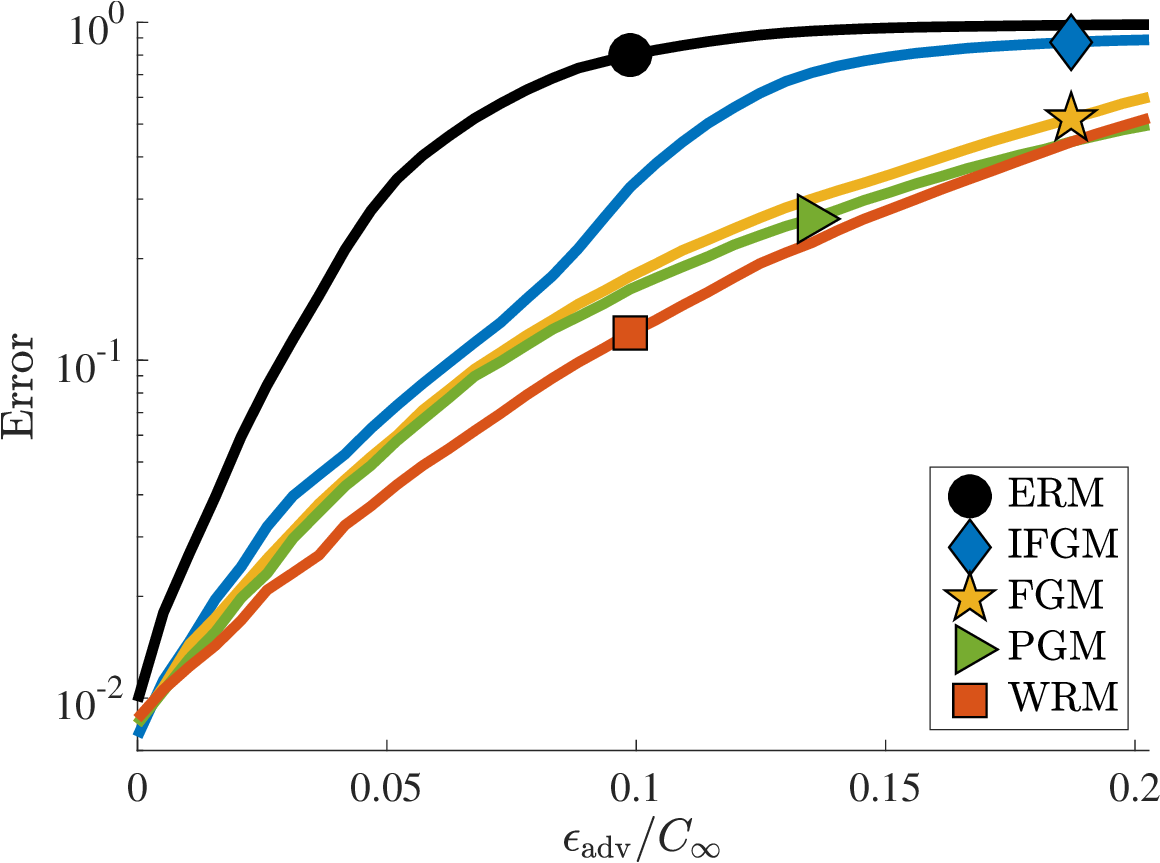}}%
\end{minipage}
\begin{minipage}{0.49\columnwidth}
\centering
\subfigure[Test error vs. $\epsilon_{\rm adv}$ for $\|\cdot\|_2$-IFGM attack]{\includegraphics[width=0.85\textwidth]{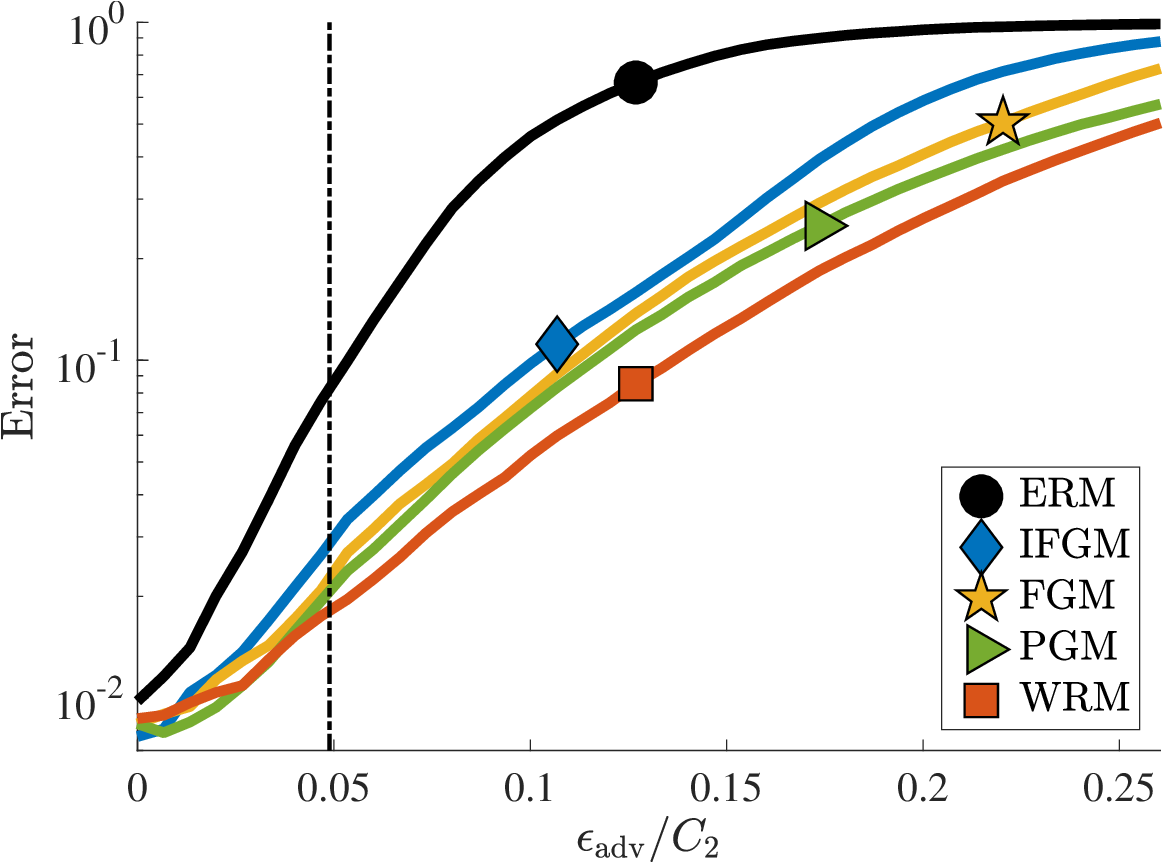}}%
\end{minipage}
\centering
\begin{minipage}{0.49\columnwidth}%
\centering
\subfigure[Test error vs. $\epsilon_{\rm adv}$ for $\|\cdot\|_{\infty}$-IFGM attack]{\includegraphics[width=0.85\textwidth]{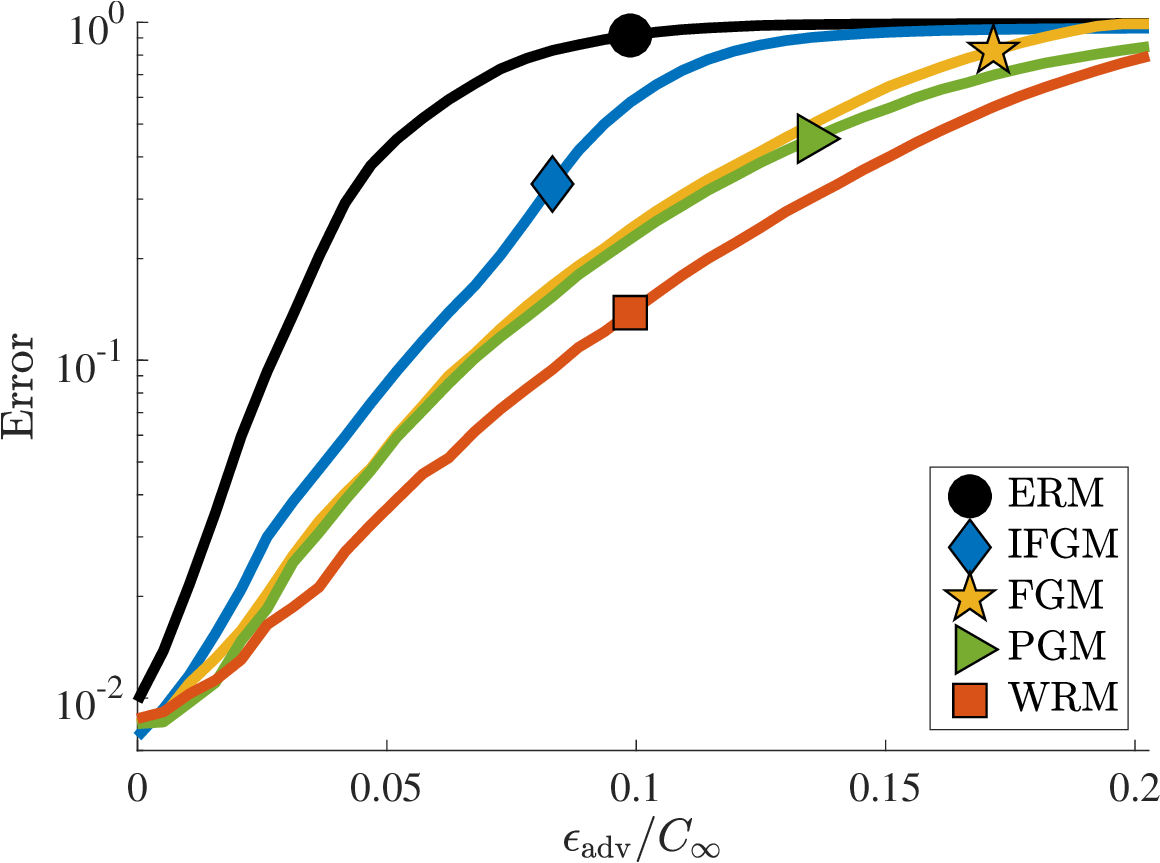}}%
\end{minipage}

\centering

\caption[]{\label{fig:mnist-fgm}Further attacks on the MNIST
  dataset. We illustrate test misclassification error vs. the adversarial
  perturbation level $\epsilon_{\rm adv}$. Top row: FGM attacks, bottom row: IFGM attacks. Left column: Euclidean-norm attacks, right column: $\infty$-norm attacks. The vertical bar
  in (a) and (c) indicates the perturbation level that was used for training the PGM, FGM, and IFGM models and the estimated radius
  $\sqrt{\what{\rho}_n(\theta_{\rm WRM})}$.}
\end{figure}

\subsection{MNIST stability of loss surface}
\label{sec:more-stability}
In Figure \ref{fig:mnist3}, we repeat the illustration in Figure \ref{fig:mnist2}(b) for more digits. WRM's ``misclassifications'' are consistently reasonable to the human eye, as gradient-based perturbations actually transform the original image to other labels. Other models do not exhibit this behavior with the same consistency (if at all). Reasonable misclassifications correspond to having learned a data representation that makes gradients interpretable.

\begin{figure}[!!t]
\begin{minipage}{0.31\columnwidth}%
\centering
\subfigure[]{\includegraphics[width=0.8\textwidth]{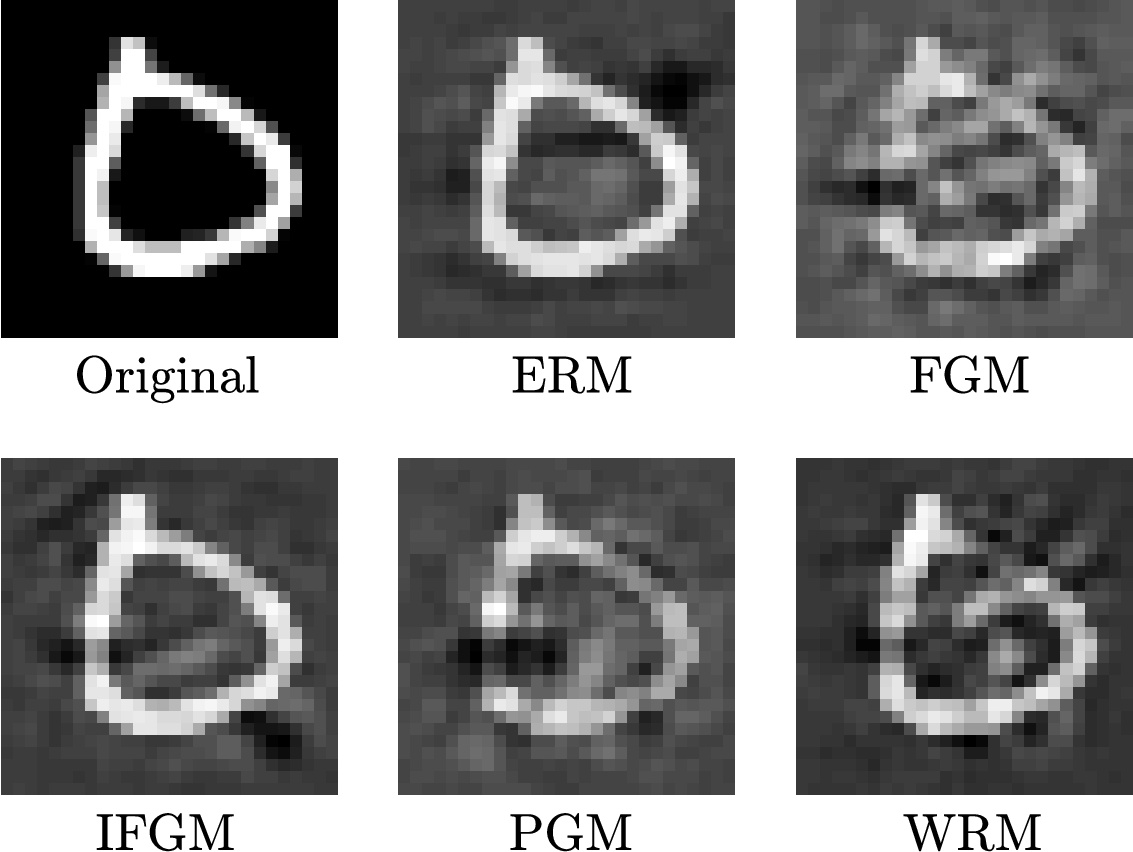}}%
\end{minipage}\centering\begin{minipage}{0.31\columnwidth}%
\centering
\subfigure[]{\includegraphics[width=0.8\textwidth]{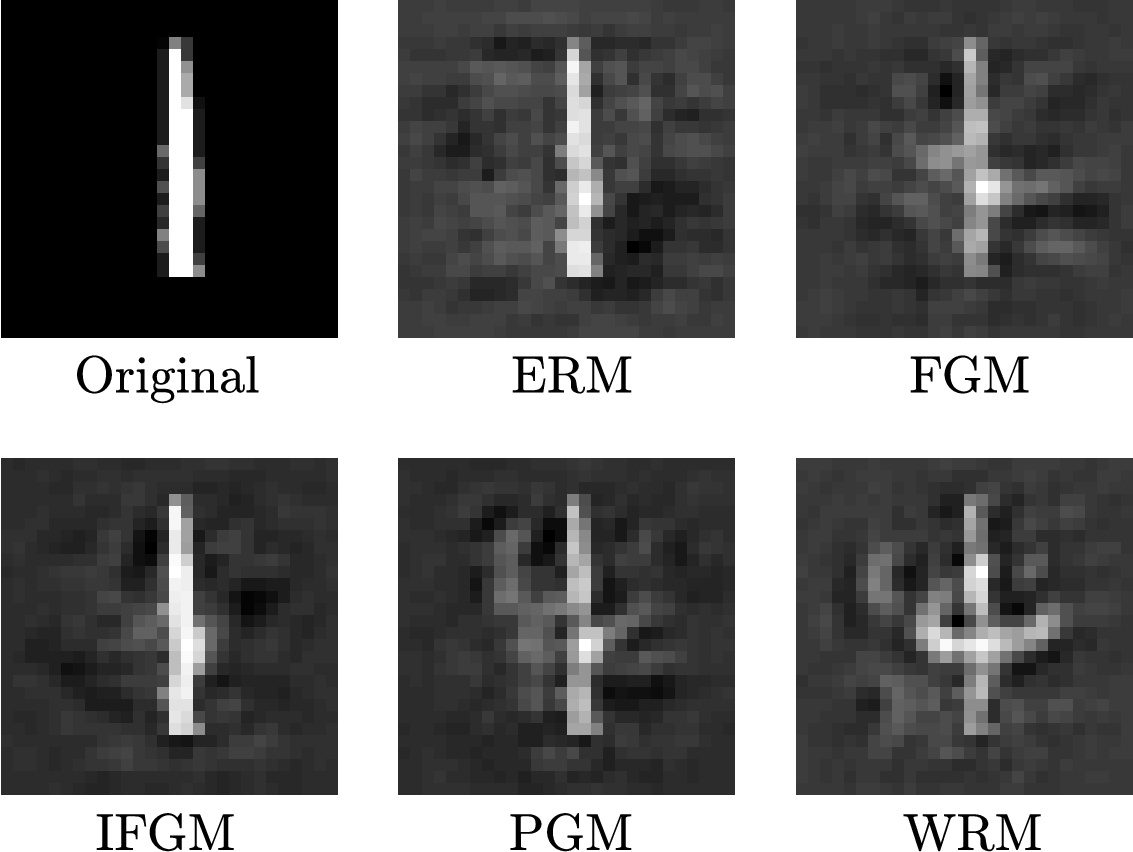}}%
\end{minipage}\begin{minipage}{0.31\columnwidth}%
\centering
\subfigure[]{\includegraphics[width=0.8\textwidth]{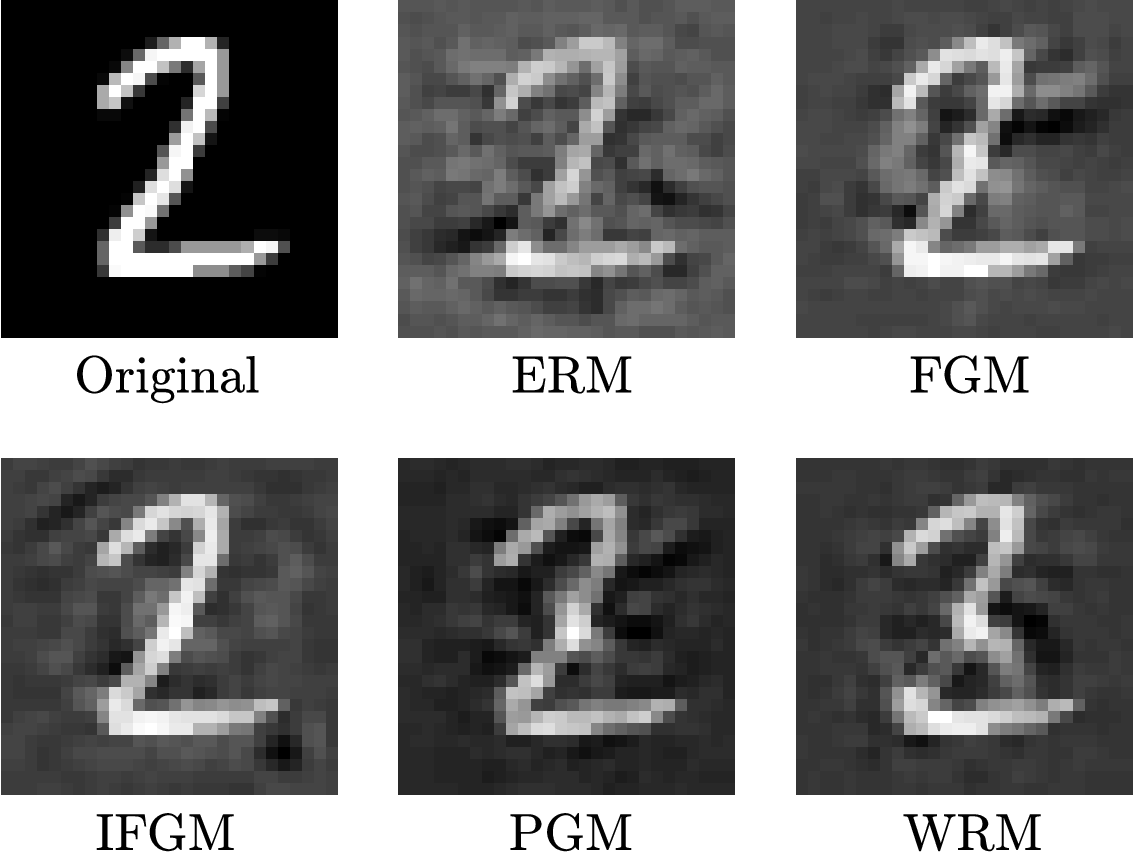}}%
\end{minipage}
\begin{minipage}{0.31\columnwidth}%
\centering
\subfigure[]{\includegraphics[width=0.8\textwidth]{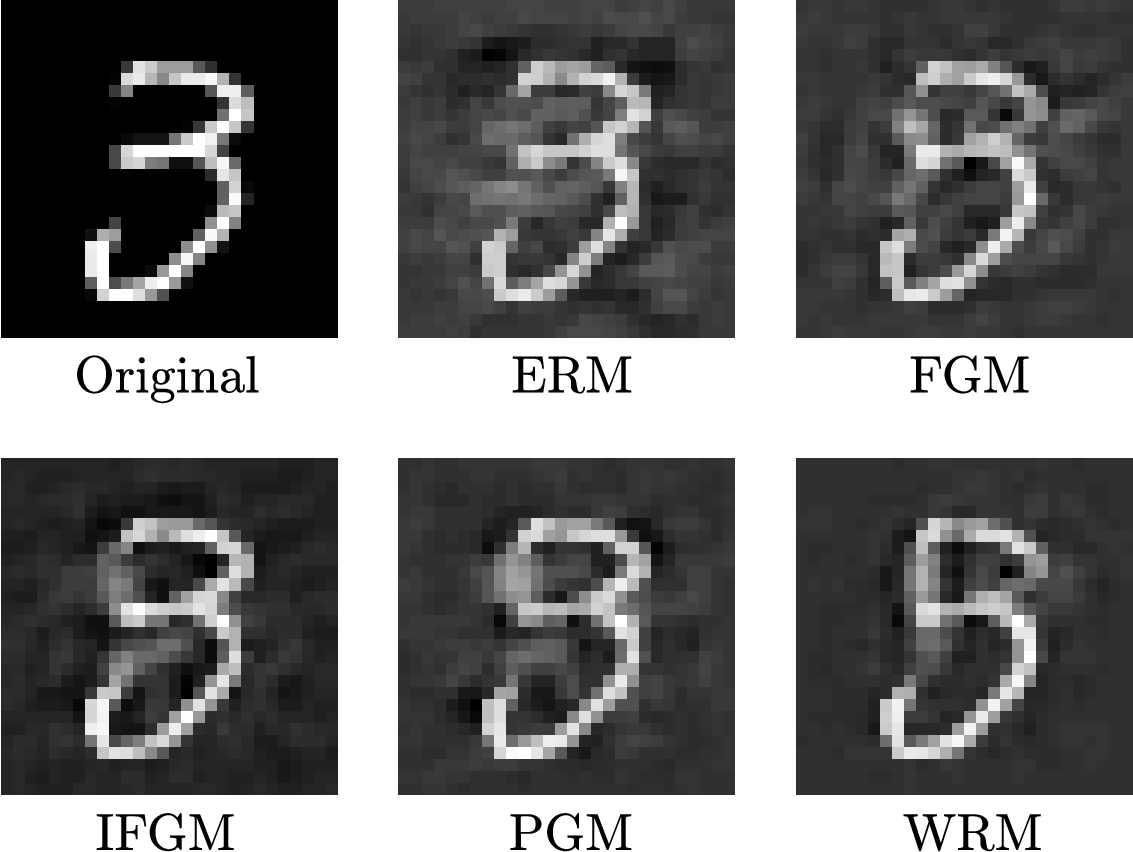}}%
\end{minipage}\centering\begin{minipage}{0.31\columnwidth}%
\centering
\subfigure[]{\includegraphics[width=0.8\textwidth]{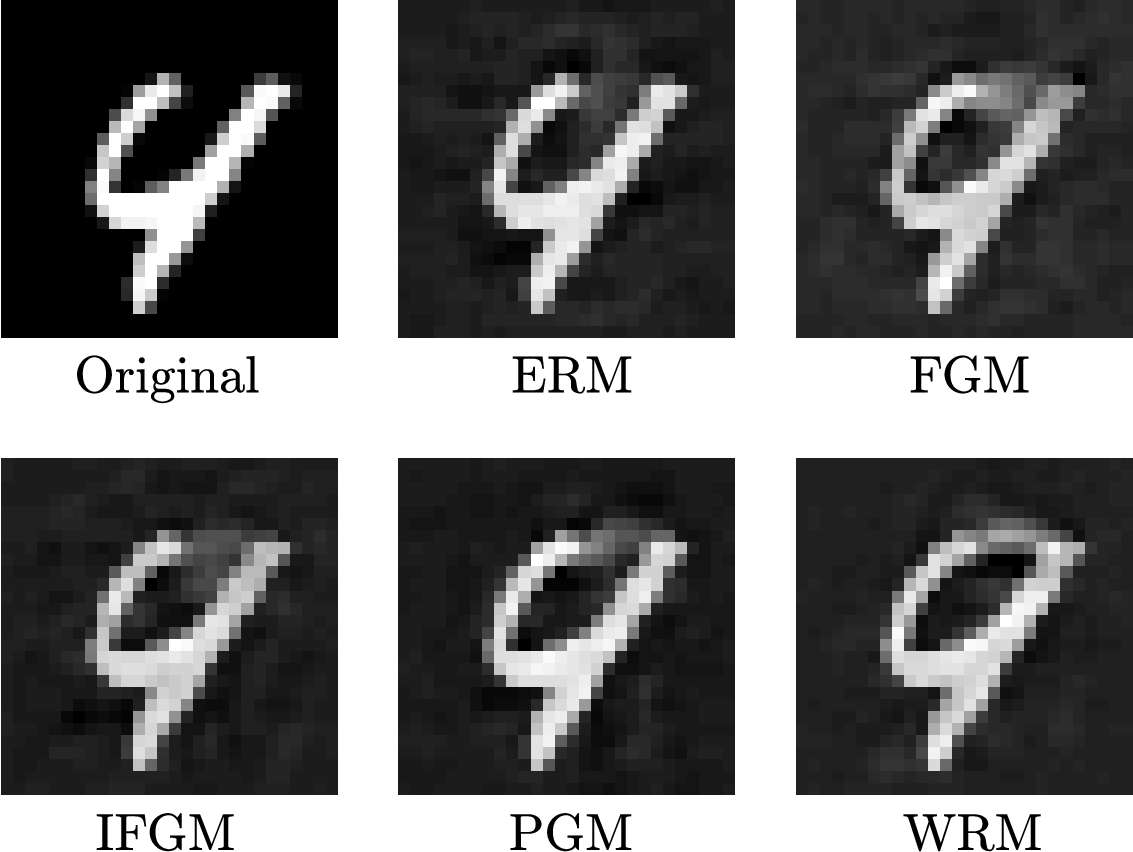}}%
\end{minipage}\begin{minipage}{0.31\columnwidth}%
\centering
\subfigure[]{\includegraphics[width=0.8\textwidth]{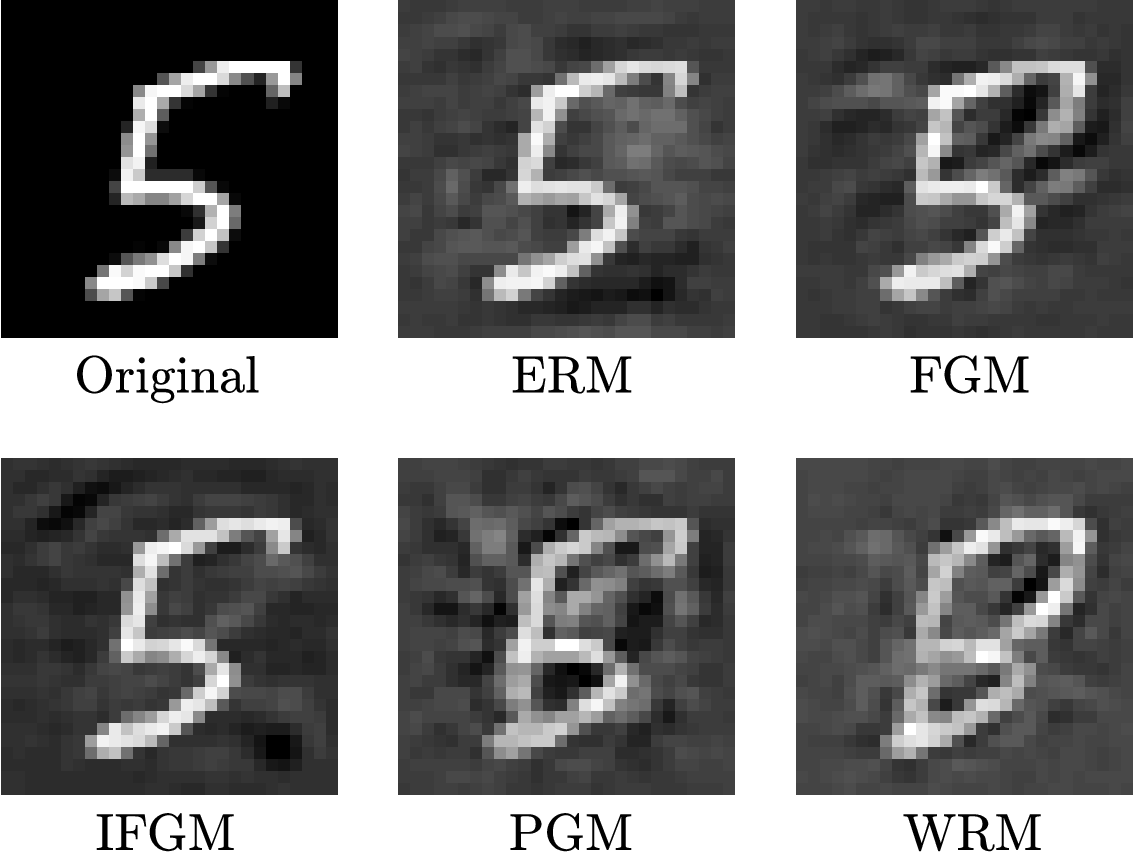}}%
\end{minipage}
\begin{minipage}{0.31\columnwidth}%
\centering
\subfigure[]{\includegraphics[width=0.8\textwidth]{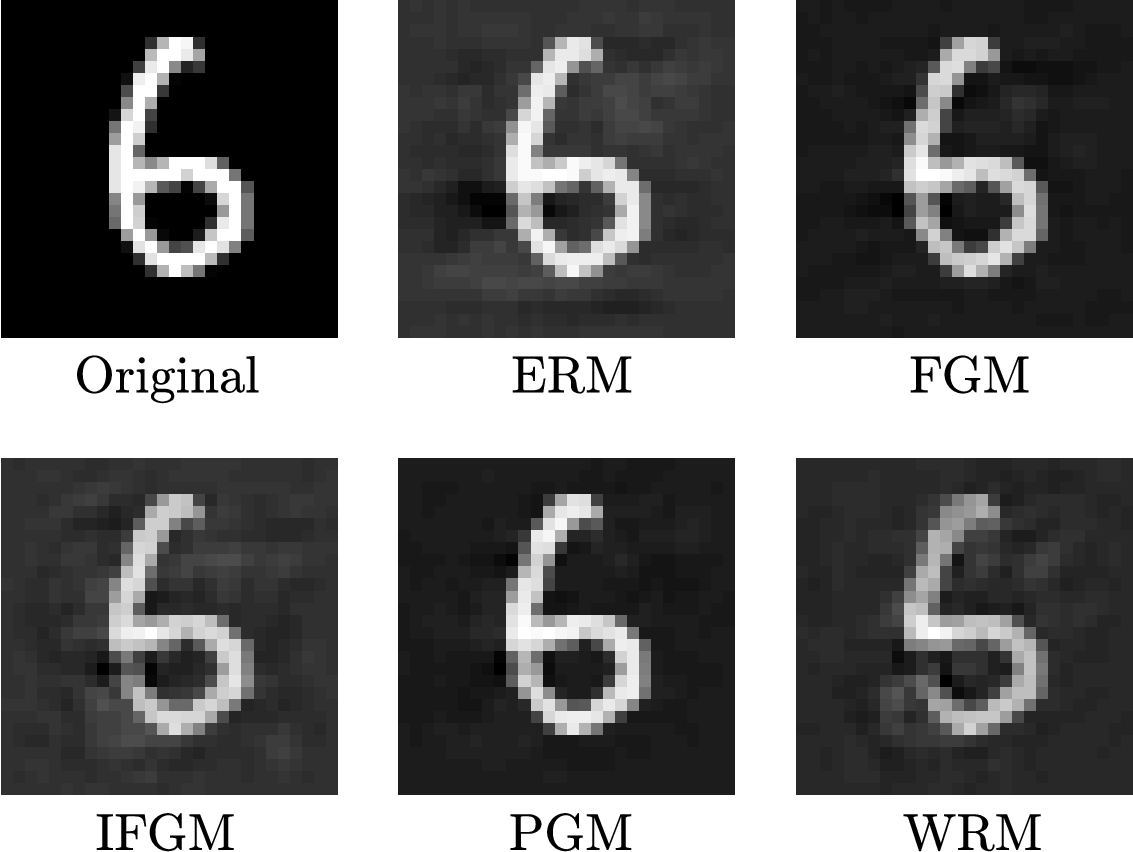}}%
\end{minipage}\centering\begin{minipage}{0.31\columnwidth}%
\centering
\subfigure[]{\includegraphics[width=0.8\textwidth]{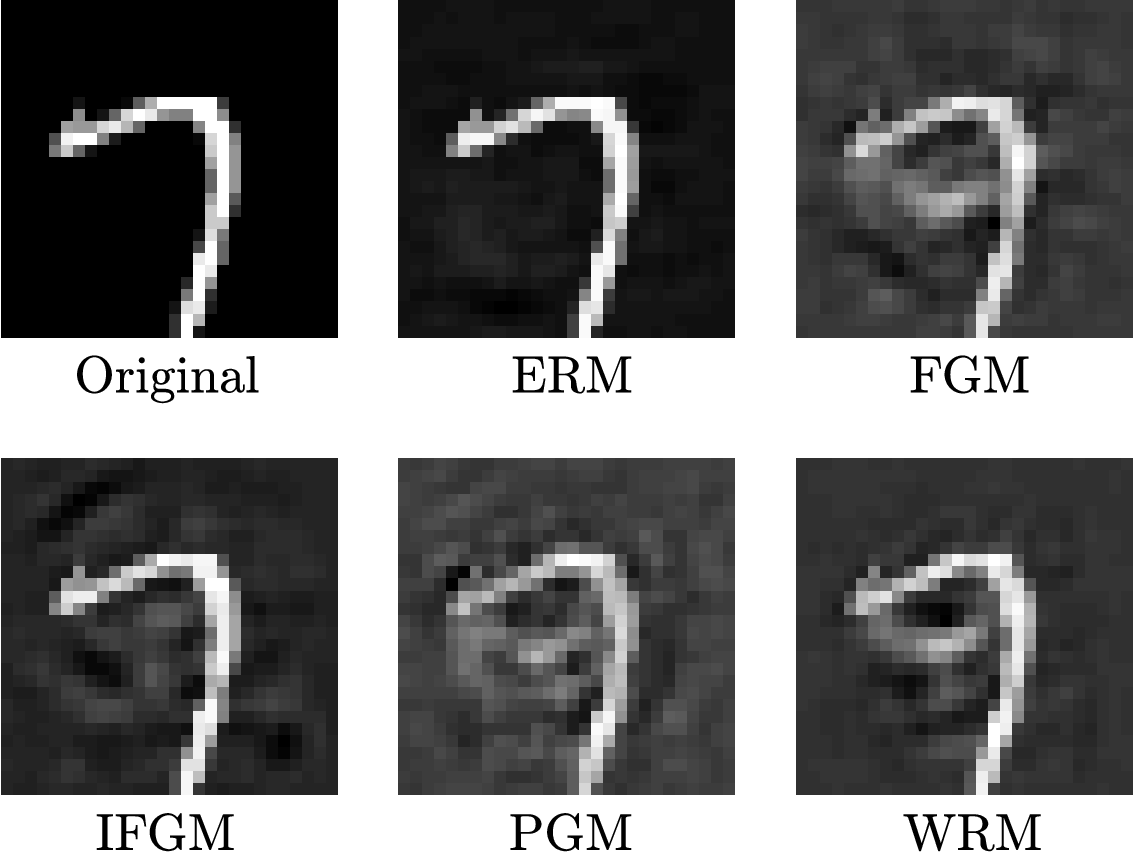}}%
\end{minipage}\begin{minipage}{0.31\columnwidth}%
\centering
\subfigure[]{\includegraphics[width=0.8\textwidth]{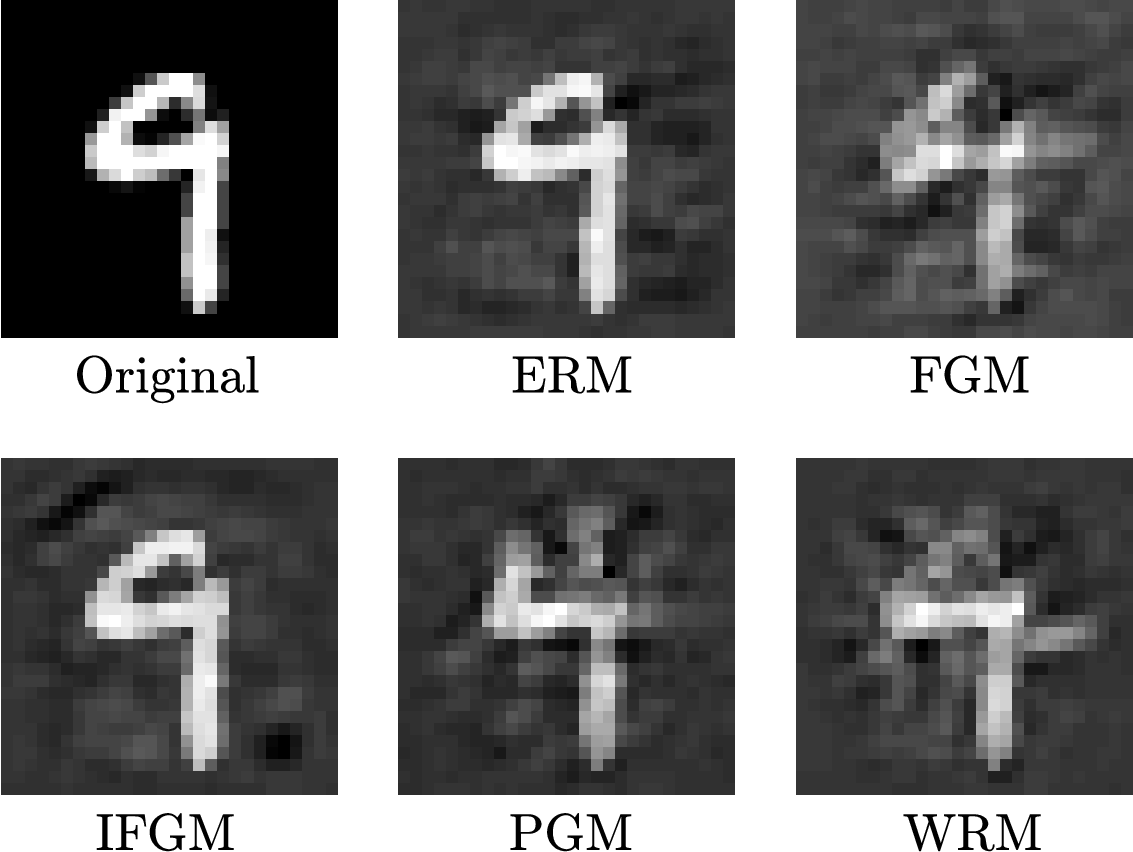}}%
\end{minipage}
\centering
\caption[]{\label{fig:mnist3}Visualizing stability
  over inputs. We illustrate the smallest WRM perturbation (largest $\gamma_{\rm adv}$) necessary to make a model misclassify a datapoint.}
\end{figure}

\subsection{MNIST Experiments with varied $\gamma$}
In Figure \ref{fig:mnist4}, we choose a fixed WRM adversary (fixed
$\gamma_{\rm adv}$) and perturb WRM models trained with various penalty
parameters $\gamma$. As the bound~\eqref{eqn:robustness-any-rho}
with $\eta = \gamma$ suggests, even when the adversary has more budget than
that used for training ($1/\gamma < 1/\gamma_{\rm adv}$), degradation in
performance is still \emph{smooth}. Further, as we decrease the penalty
$\gamma$, the amount of achieved robustness---measured here by
test error on adversarial perturbations with $\gamma_{\rm adv}$---has
diminishing gains; this is again consistent with our theory which says that the
inner problem~\eqref{eqn:inner-sup} is not efficiently computable for small $\gamma$.

\begin{figure}[!!t]
\begin{minipage}{0.49\columnwidth}
\centering
\subfigure[$\what{\rho}_n$ vs. $1/\gamma$]{\includegraphics[width=0.95\textwidth]{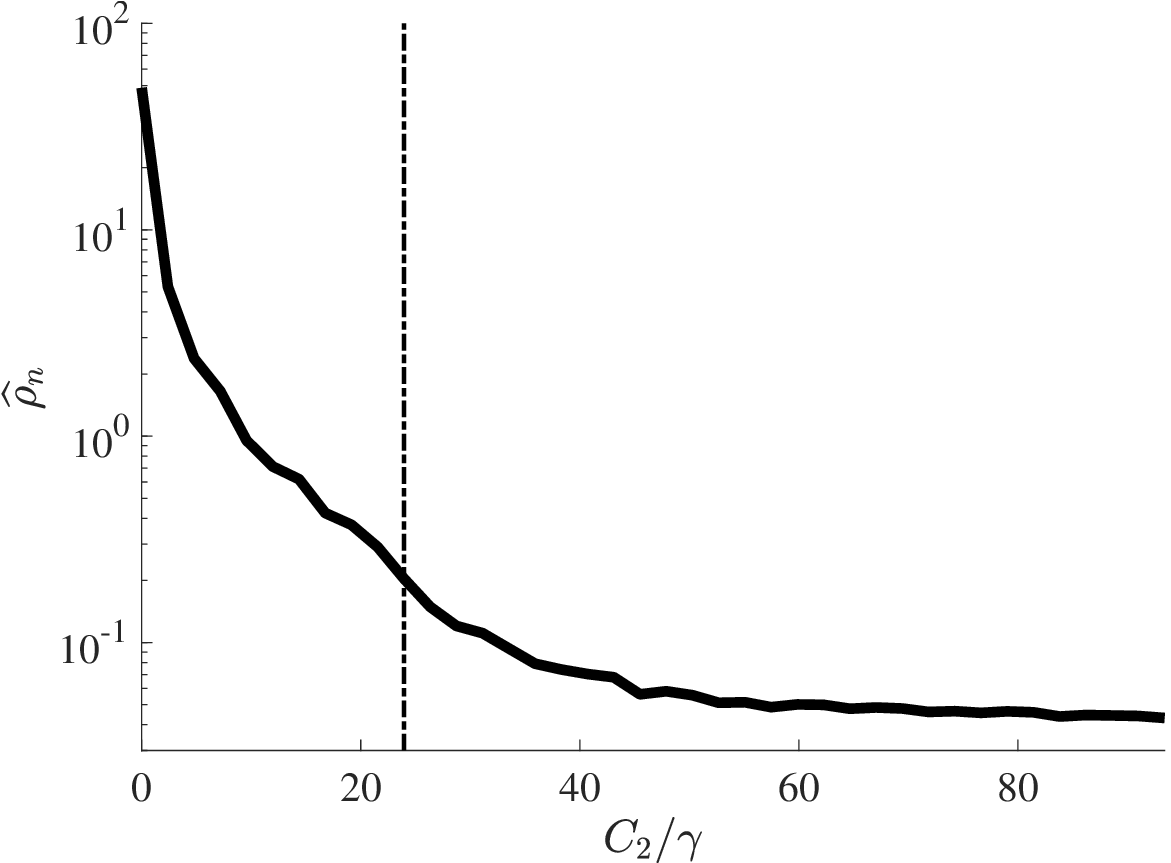}}%
\end{minipage}
\centering
\begin{minipage}{0.49\columnwidth}%
\centering
\subfigure[Test error vs. $1/\gamma$]{\includegraphics[width=0.95\textwidth]{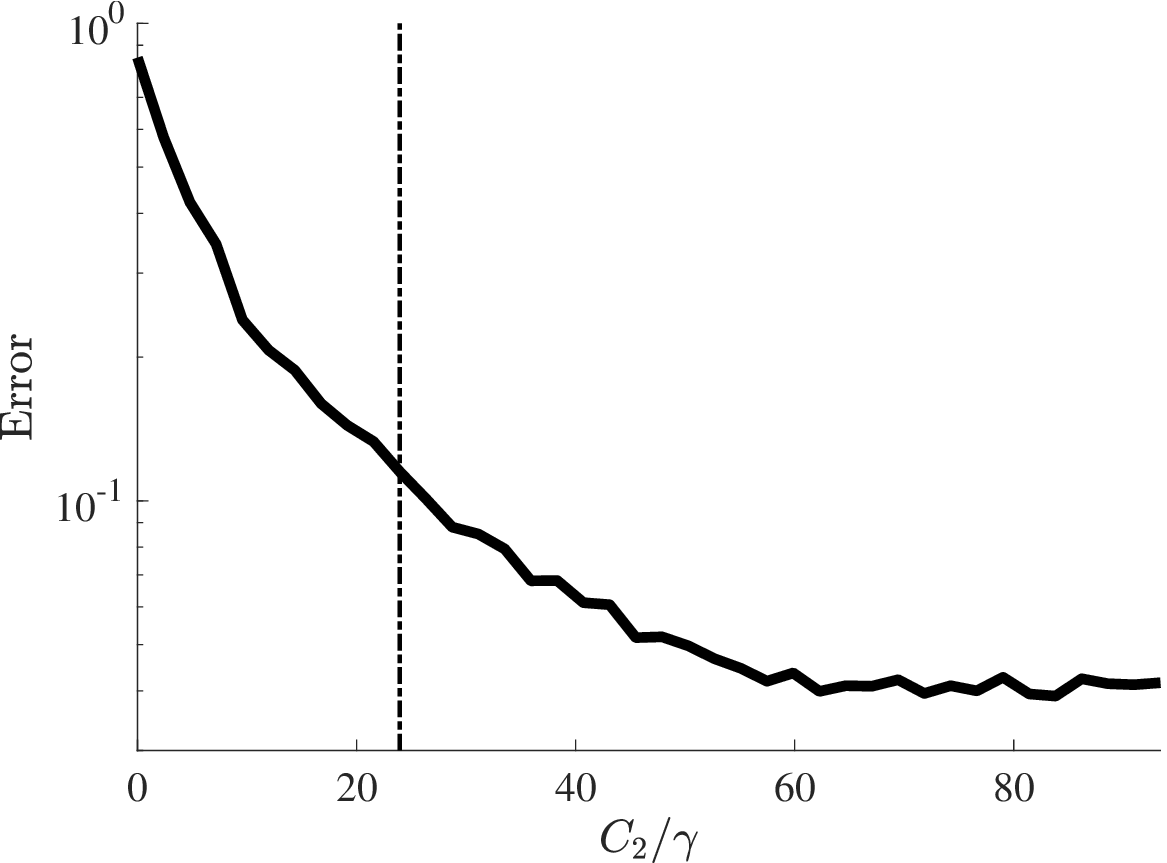}}%
\end{minipage}
\centering

\caption[]{\label{fig:mnist4}(a) Stability and (b) test error for a fixed adversary. We train WRM models with various levels of $\gamma$ and perturb them with a fixed WRM adversary ($\gamma_{\rm adv}$ indicated by the vertical bar). }
\end{figure}

\clearpage 
\subsection{MNIST experiments with a larger adversarial budget}
\label{sec:large-adversary}

Figures \ref{fig:mnist_bound_big} and \ref{fig:mnist-big} repeat Figures
\ref{fig:loss-bound}(b), \ref{fig:mnist}, and \ref{fig:mnist-fgm} for a larger training adversarial budget
($\gamma = 0.02C_2$) as well as larger test adversarial budgets. The
distinctions in performance between various methods are less apparent now. For
our method, the inner supremum
is no longer strongly concave for over 10\% of the data,
indicating that we no longer have guarantees of performance. For large
adversaries (i.e. large desired robustness values) our approach becomes a
heuristic just like the other approaches.

\begin{figure}[!!h]
\centering
\includegraphics[width=0.4\textwidth]{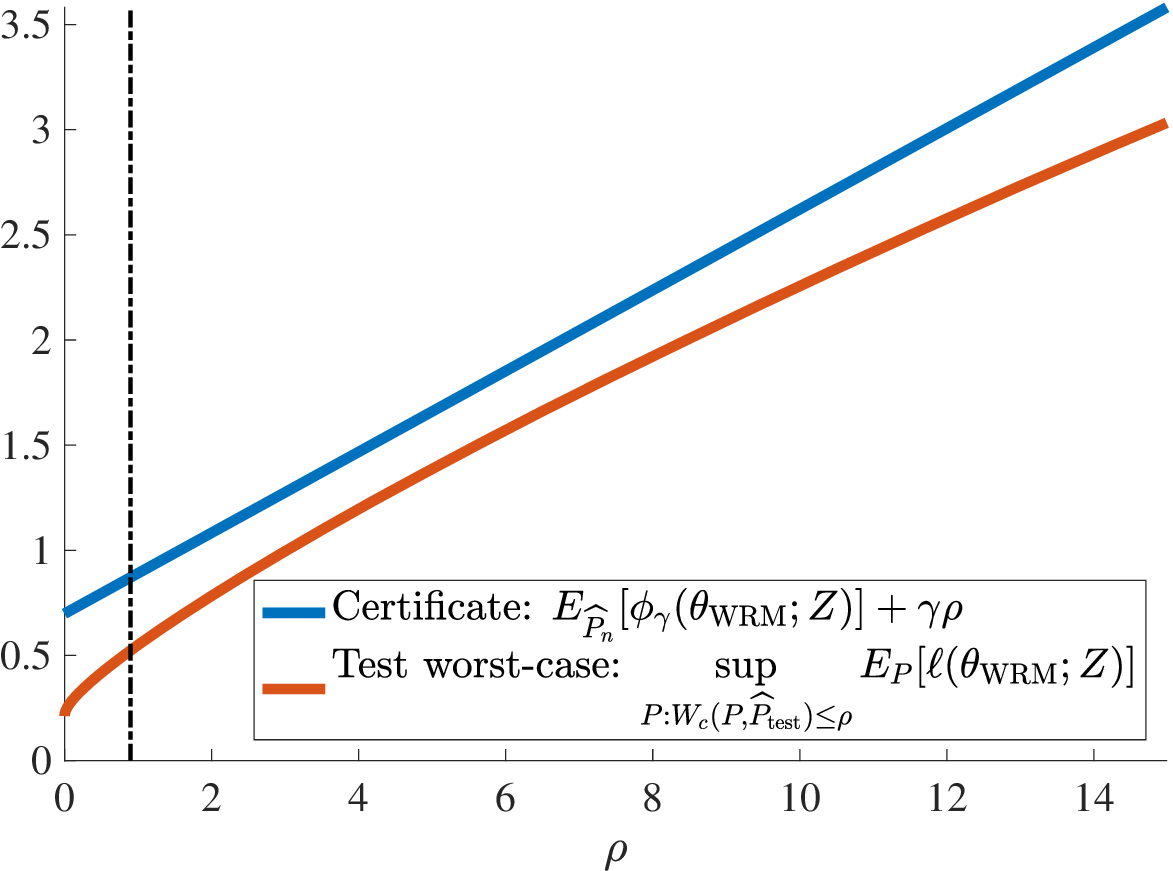}%
\centering
\caption[]{\label{fig:mnist_bound_big}Empirical comparison between certificate of robustness~\eqref{eqn:robustness-any-rho} ({\color{matlab_blue} blue}) and out-of-sample (test)
  worst-case performance ({\color{matlab_red} red}) for the experiments on MNIST with a larger training adversary. The statistical error term $\epsilon_n(t)$ is
  omitted from the certificate. The vertical bar indicates the achieved level of robustness on the training set $\what{\tol}_n$($\theta_{\rm WRM}$). }
\end{figure}

\begin{figure}[!!h]
\begin{minipage}{0.49\columnwidth}
\centering
\subfigure[Test error vs. $\epsilon_{\rm adv}$ for $\|\cdot\|_2$-PGM attack]{\includegraphics[width=0.85\textwidth]{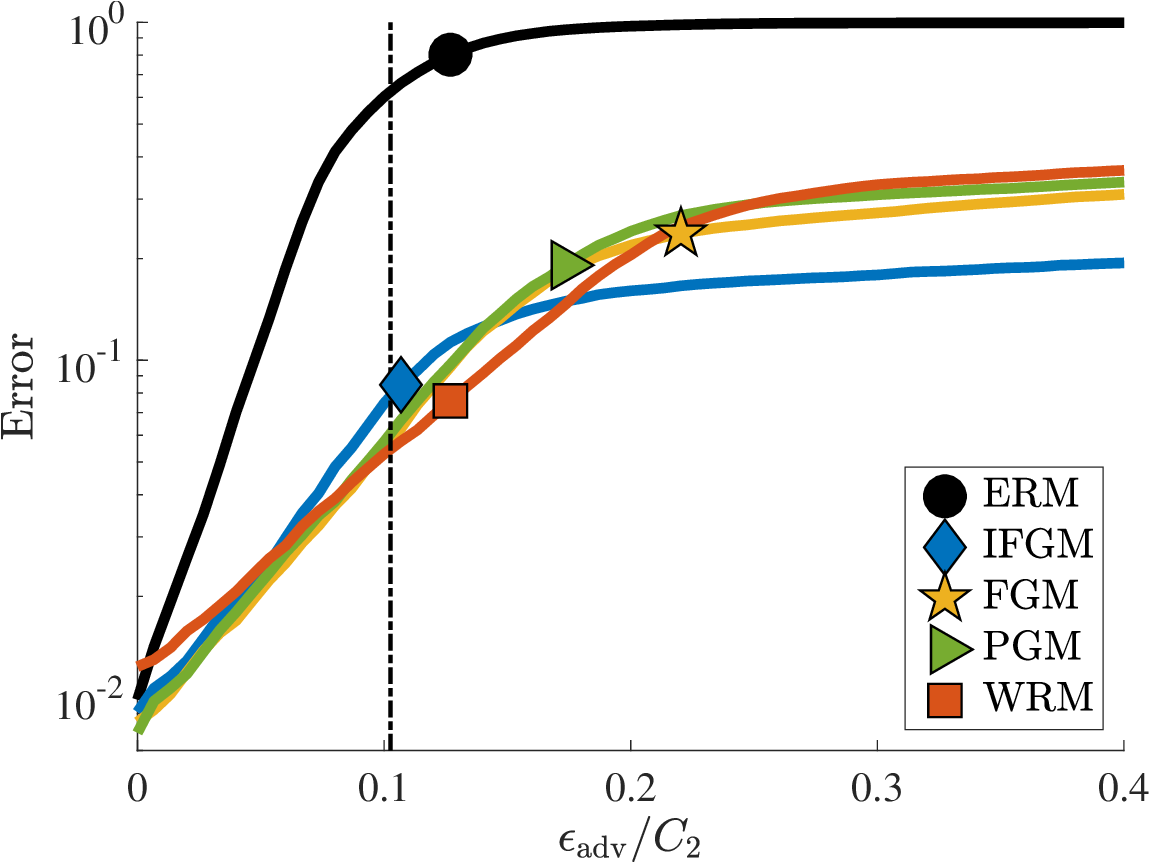}}%
\end{minipage}
\centering
\begin{minipage}{0.49\columnwidth}%
\centering
\subfigure[Test error vs. $\epsilon_{\rm adv}$ for $\|\cdot\|_{\infty}$-PGM attack]{\includegraphics[width=0.85\textwidth]{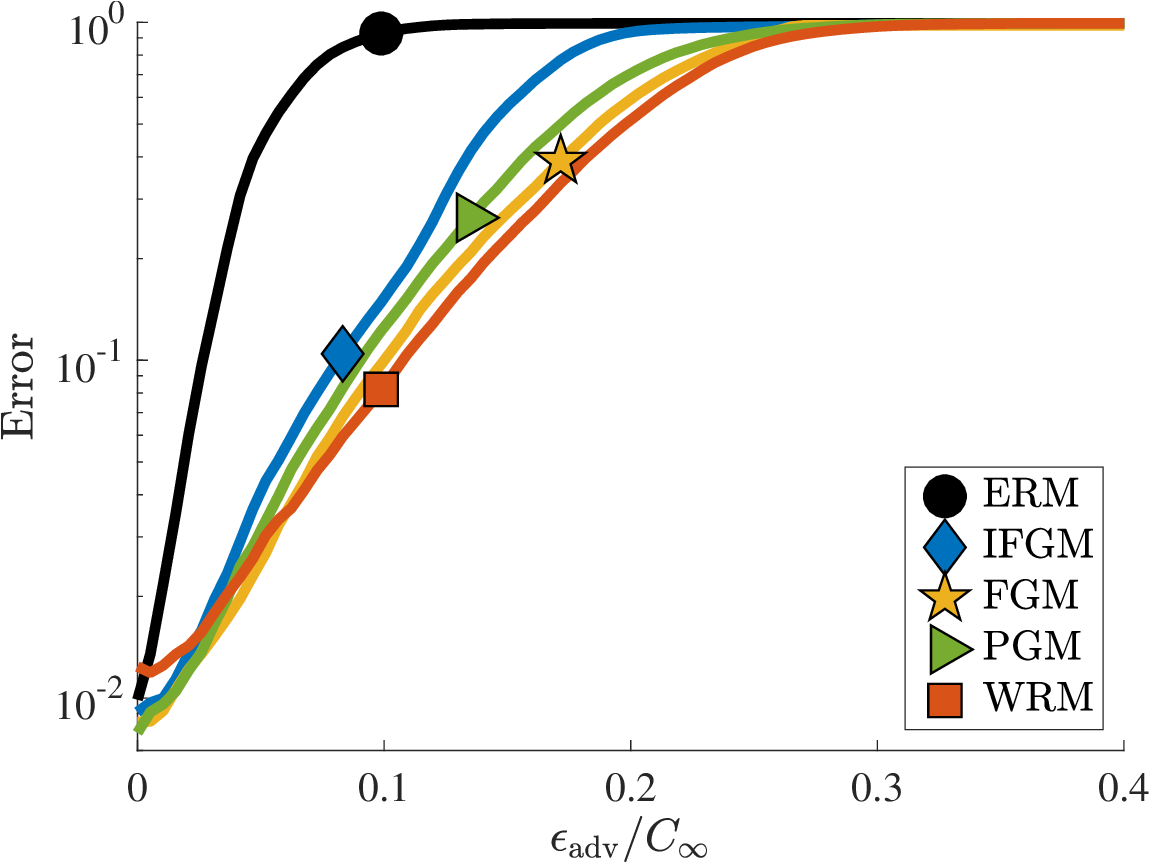}}%
\end{minipage}
\begin{minipage}{0.49\columnwidth}
\centering
\subfigure[Test error vs. $\epsilon_{\rm adv}$ for $\|\cdot\|_2$-FGM attack]{\includegraphics[width=0.85\textwidth]{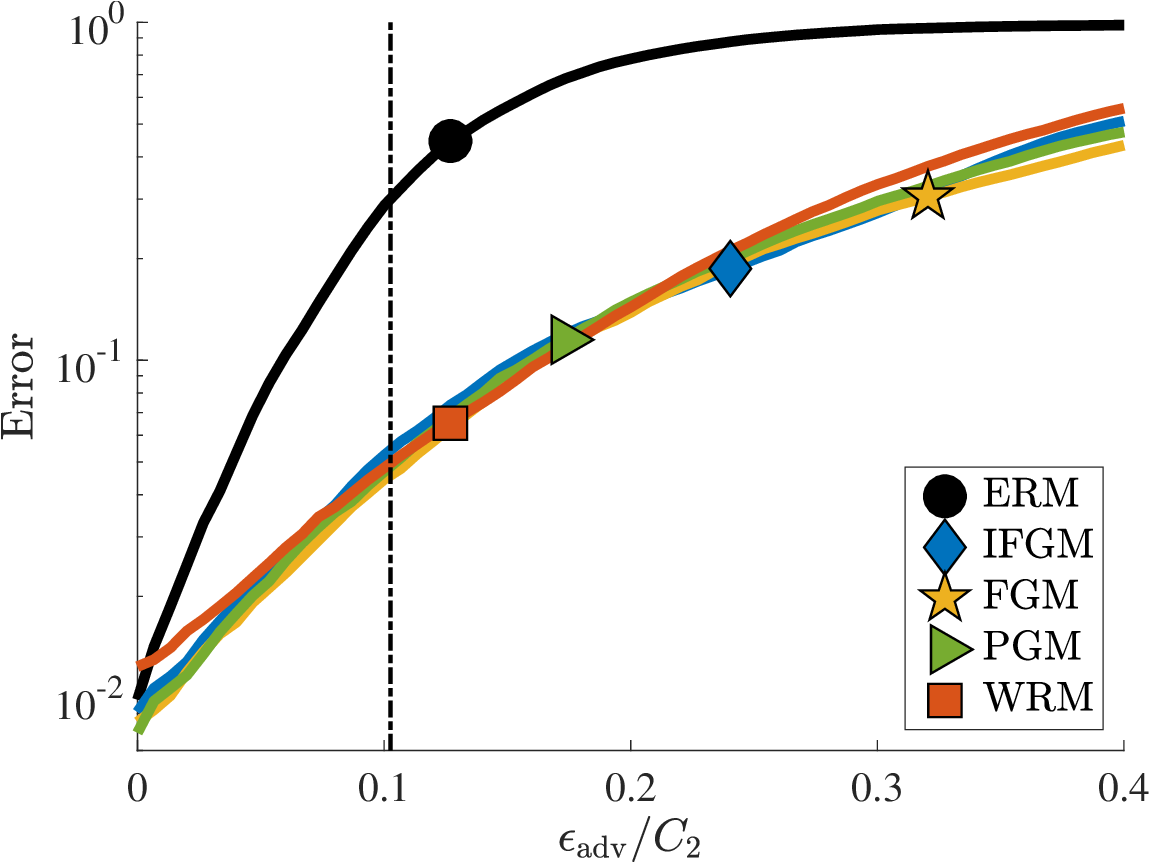}}%
\end{minipage}
\centering
\begin{minipage}{0.49\columnwidth}%
\centering
\subfigure[Test error vs. $\epsilon_{\rm adv}$ for $\|\cdot\|_{\infty}$-FGM attack]{\includegraphics[width=0.85\textwidth]{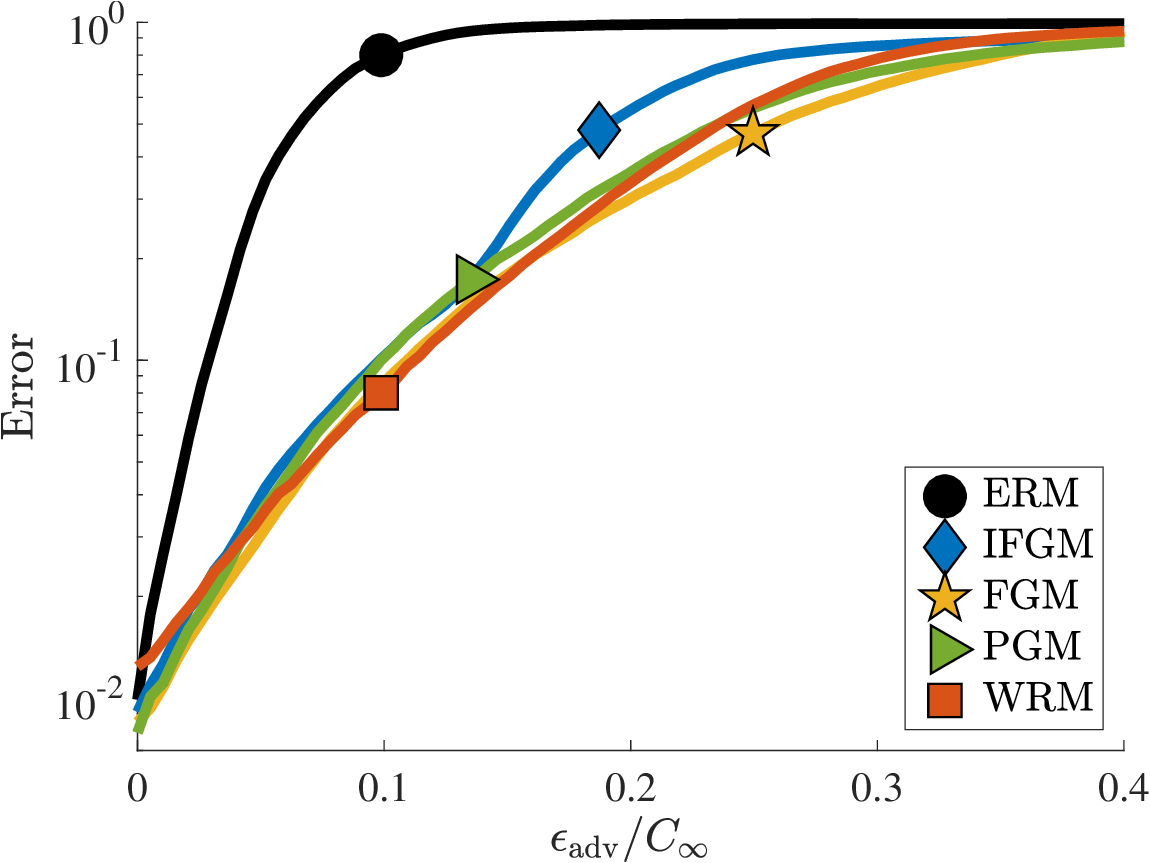}}%
\end{minipage}
\begin{minipage}{0.49\columnwidth}
\centering
\subfigure[Test error vs. $\epsilon_{\rm adv}$ for $\|\cdot\|_2$-IFGM attack]{\includegraphics[width=0.85\textwidth]{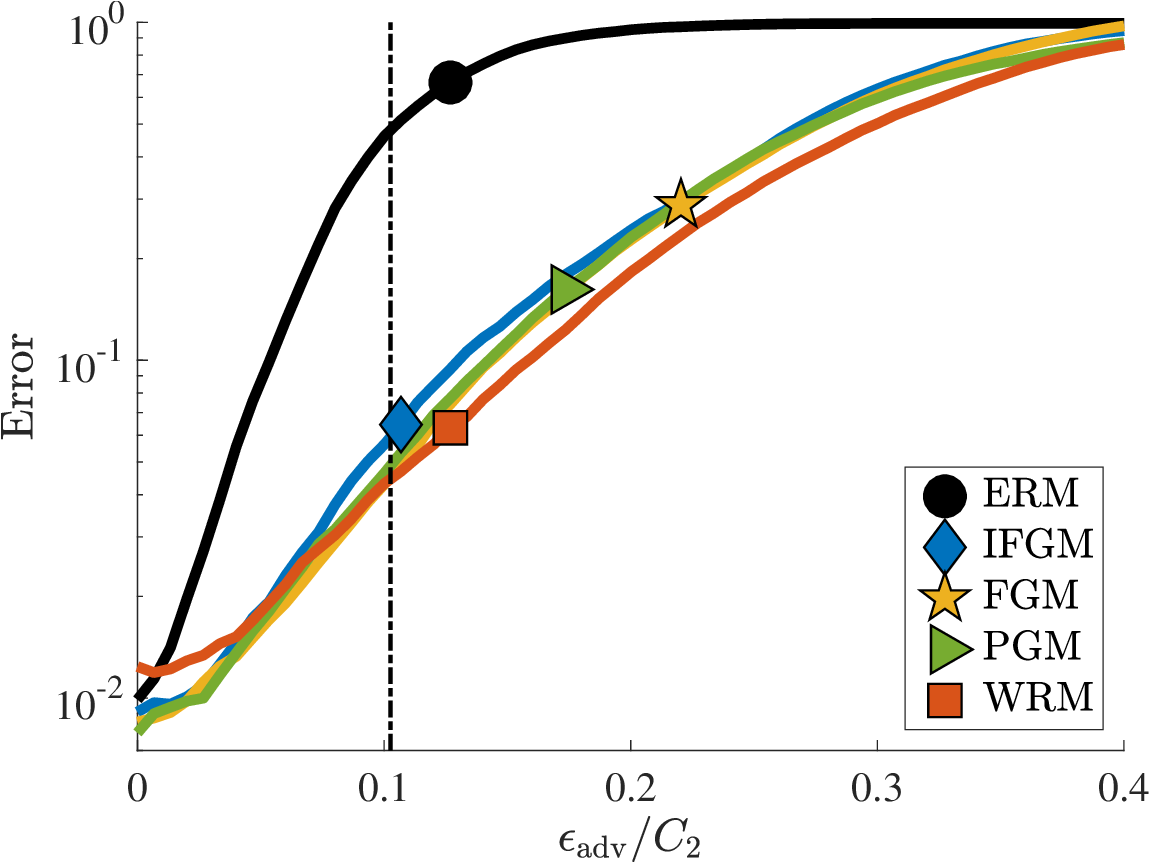}}%
\end{minipage}
\centering
\begin{minipage}{0.49\columnwidth}%
\centering
\subfigure[Test error vs. $\epsilon_{\rm adv}$ for $\|\cdot\|_{\infty}$-IFGM attack]{\includegraphics[width=0.85\textwidth]{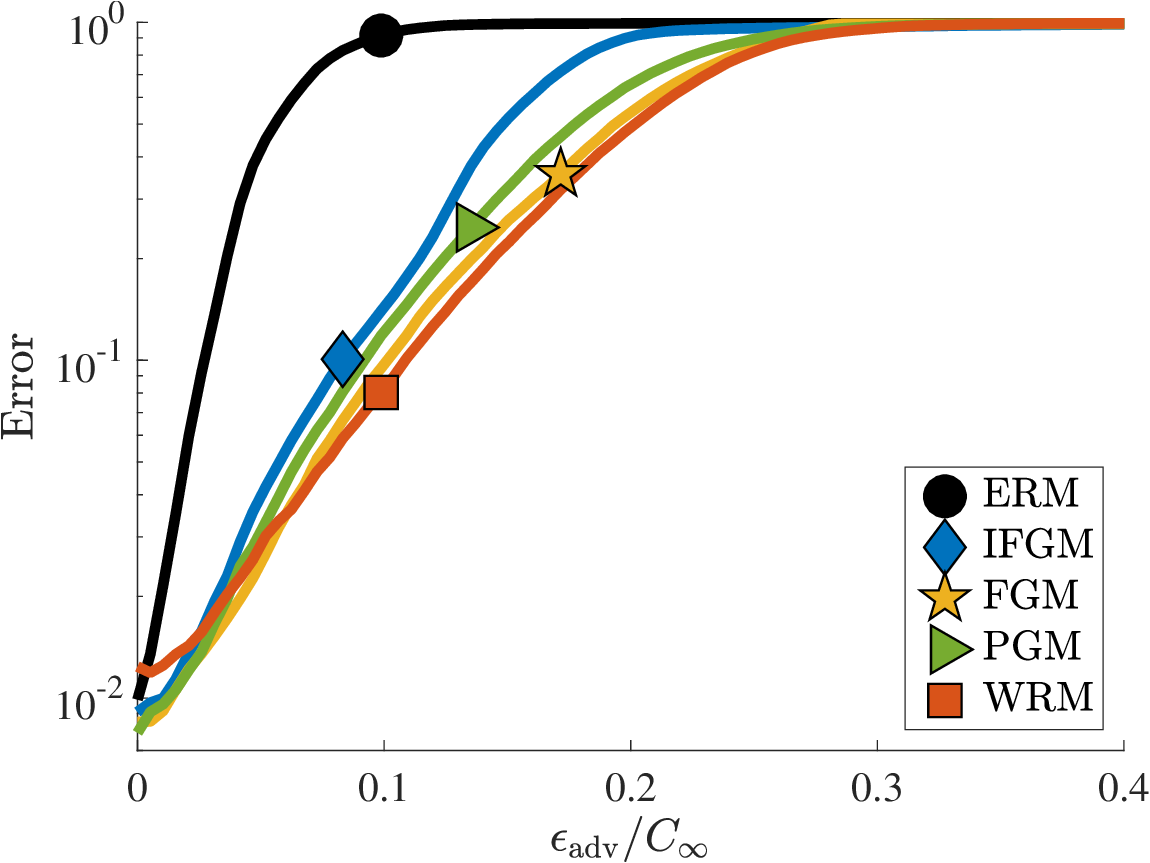}}%
\end{minipage}
\centering

\caption[]{\label{fig:mnist-big}Attacks on the MNIST
  dataset with larger (training and test) adversarial budgets.  We illustrate test misclassification error vs. the adversarial
  perturbation level $\epsilon_{\rm adv}$. Top row: PGM attacks, middle row: FGM attacks, bottom row: IFGM attacks. Left column: Euclidean-norm attacks, right column: $\infty$-norm attacks. The vertical bar in (a), (c), and (e)
  indicates the perturbation level that was used for training the PGM, FGM, and IFGM models and the estimated radius
  $\sqrt{\what{\rho}_n(\theta_{\rm WRM})}$.}
\end{figure}

\newpage 
\subsection{MNIST $\infty$-norm experiments}
\label{section:sup-norm}
We consider training FGM, IFGM, and PGM with $p=\infty$. We first compare with WRM trained in the same manner as before---with the squared Euclidean cost. Then, we consider a heuristic Lagrangian approach for training WRM with the squared $\infty$-norm cost. 

\subsubsection{Comparison with standard WRM}\label{sec:infElse-2WRM}
Our method (WRM) is trained to defend against $\ltwo{\cdot}$-norm attacks by
using the cost function
\begin{equation*}
  c((x, y), (x_0, y_0)) = \ltwo{x - x_0}^2 + \infty \cdot \indic{y \neq y_0}
\end{equation*}
with the convention that $0 \cdot \infty = 0$. Standard adversarial training
methods often train to defend against $\linf{\cdot}$-norm attacks, which we
compare our method against in this subsection. Direct comparison between these
approaches is not immediate, as we need to determine a suitable $\epsilon$ to
train FGM, IFGM, and PGM in the $\infty$-norm that corresponds to the penalty
parameter $\gamma$ for the $\ltwo{\cdot}$-norm that we use. Similar to the expression~\eqref{eqn:fgm-fair}, we use
\begin{equation}\label{eq:inf-from-2}
\epsilon := \E_{\emp}[\linf{T(\theta_{ \rm WRM}, Z)-Z}]
\end{equation}
as the adversarial training budget for FGM, IFGM and PGM with
$\linf{\cdot}$-norms. Because 2-norm adversaries tend to focus budgets on a
subset of features, the resulting $\infty$-norm perturbations are relatively
large. In Figure~\ref{fig:mnist-infwrm2} we show the results trained with a small
training adversarial budget.  In this regime, (large $\gamma$, small $\epsilon$), WRM matches the performance of other techniques. 

In
Figure~\ref{fig:mnist-infwrm2big} we show the results trained with a large
training adversarial budget.
In this regime (small $\gamma$, large
$\epsilon$), performance between WRM and other methods diverge. WRM, which provably defends against small perturbations, outperforms other
heuristics against imperceptible attacks for both Euclidean and $\infty$ norms. Further, it outperforms other heuristics on natural images, showing
that it consistently achieves a smaller price of robustness. On attacks with
large adversarial budgets (large $\epsilon_{\rm adv}$), however, the performance of WRM is worse than that
of the other methods (especially in the case of $\infty$-norm attacks). These findings verify
that WRM is a practical alternative over existing heuristics for the moderate levels of robustness where our guarantees hold.

\subsubsection{Comparison with $\linf{\cdot}$-WRM}

Our computational guarantees given in Theorem~\ref{thm:convergence} does not
hold anymore when we consider $\infty$-norm adversaries:
\begin{equation}
  \label{eqn:sup-cost}
  c((x, y), (x_0, y_0)) = \linf{x - x_0}^2 + \infty \cdot \indic{y \neq y_0}.
\end{equation}
Optimizing the Lagrangian formulation~\eqref{eqn:inner-sup} with the
$\infty$-norm is difficult since subtracting a multiple of the
$\infty$-norm does not add (negative) curvature in all directions. In
Appendix~\ref{section:prox}, we propose a heuristic algorithm for solving the
inner supremum problem~\eqref{eqn:inner-sup} with the above cost
function~\eqref{eqn:sup-cost}. Our approach is based on a variant of proximal
algorithms.

We compare our proximal heuristic introduced in Appendix~\ref{section:prox}
with other adversarial training procedures that were trained against
$\infty$-norm adversaries.  Results are shown in Figure
\ref{fig:mnist-proxsmall} for a small training adversary and Figure
\ref{fig:mnist-proxbig} for a large training adversary. We observe that
similar trends as in Section~\ref{sec:infElse-2WRM} hold again.
 
\clearpage

\begin{figure}[!!t]
\begin{minipage}{0.49\columnwidth}
\centering
\subfigure[Test error vs. $\epsilon_{\rm adv}$ for $\|\cdot\|_2$-PGM attack]{\includegraphics[width=0.85\textwidth]{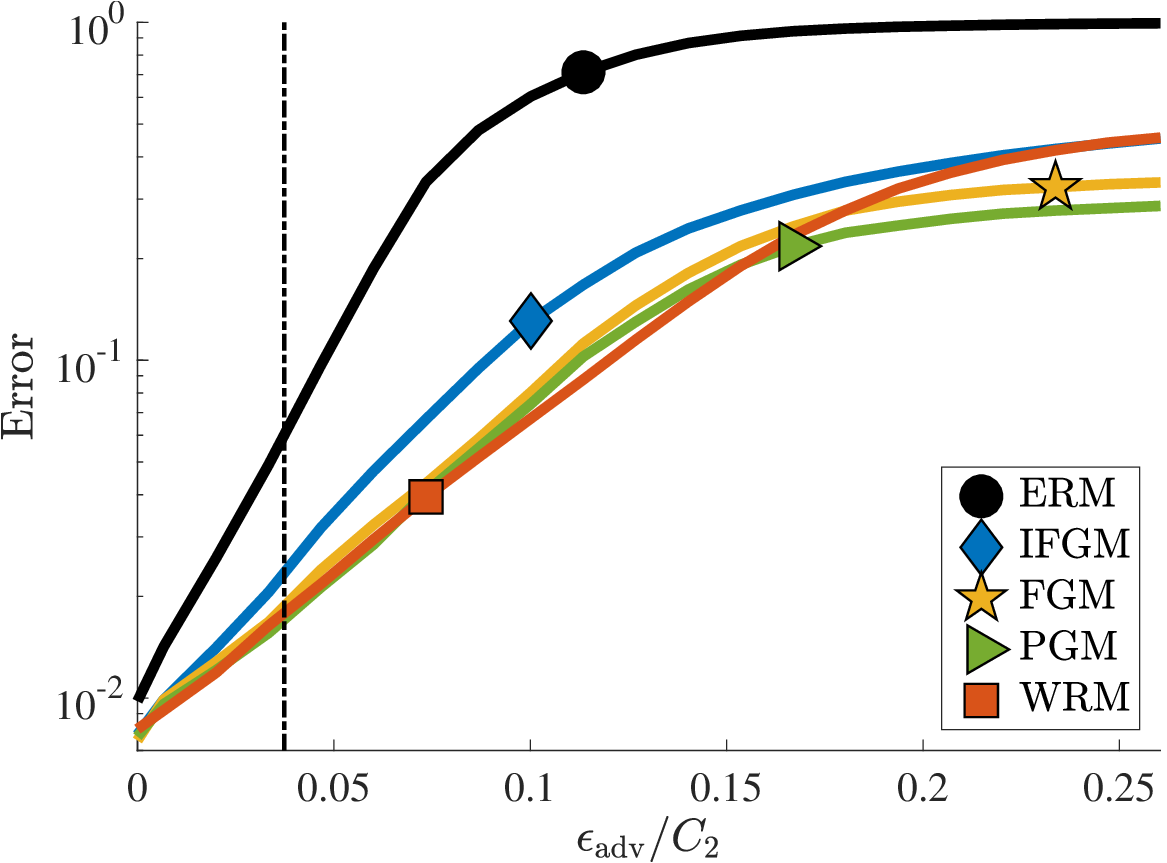}}%
\end{minipage}
\centering
\begin{minipage}{0.49\columnwidth}%
\centering
\subfigure[Test error vs. $\epsilon_{\rm adv}$ for $\|\cdot\|_{\infty}$-PGM attack]{\includegraphics[width=0.85\textwidth]{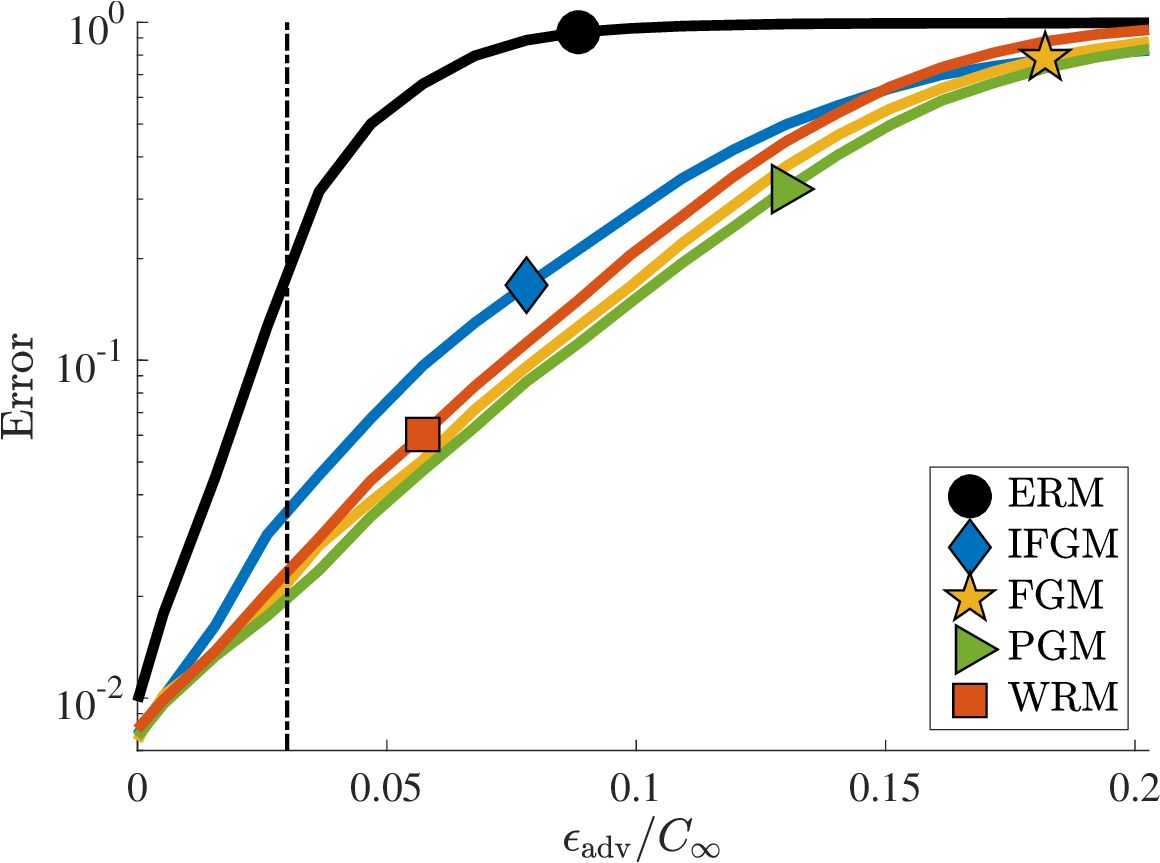}}%
\end{minipage}
\begin{minipage}{0.49\columnwidth}
\centering
\subfigure[Test error vs. $\epsilon_{\rm adv}$ for $\|\cdot\|_2$-FGM attack]{\includegraphics[width=0.85\textwidth]{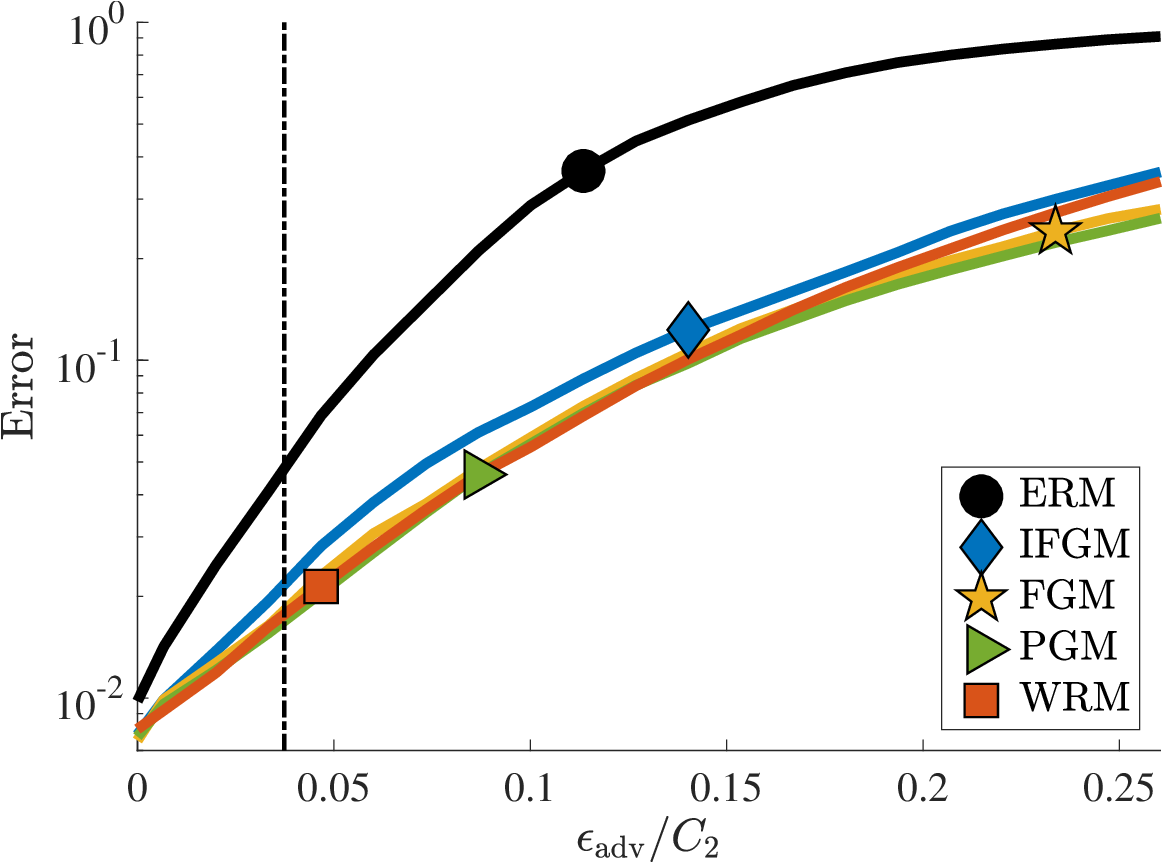}}%
\end{minipage}
\centering
\begin{minipage}{0.49\columnwidth}%
\centering
\subfigure[Test error vs. $\epsilon_{\rm adv}$ for $\|\cdot\|_{\infty}$-FGM attack]{\includegraphics[width=0.85\textwidth]{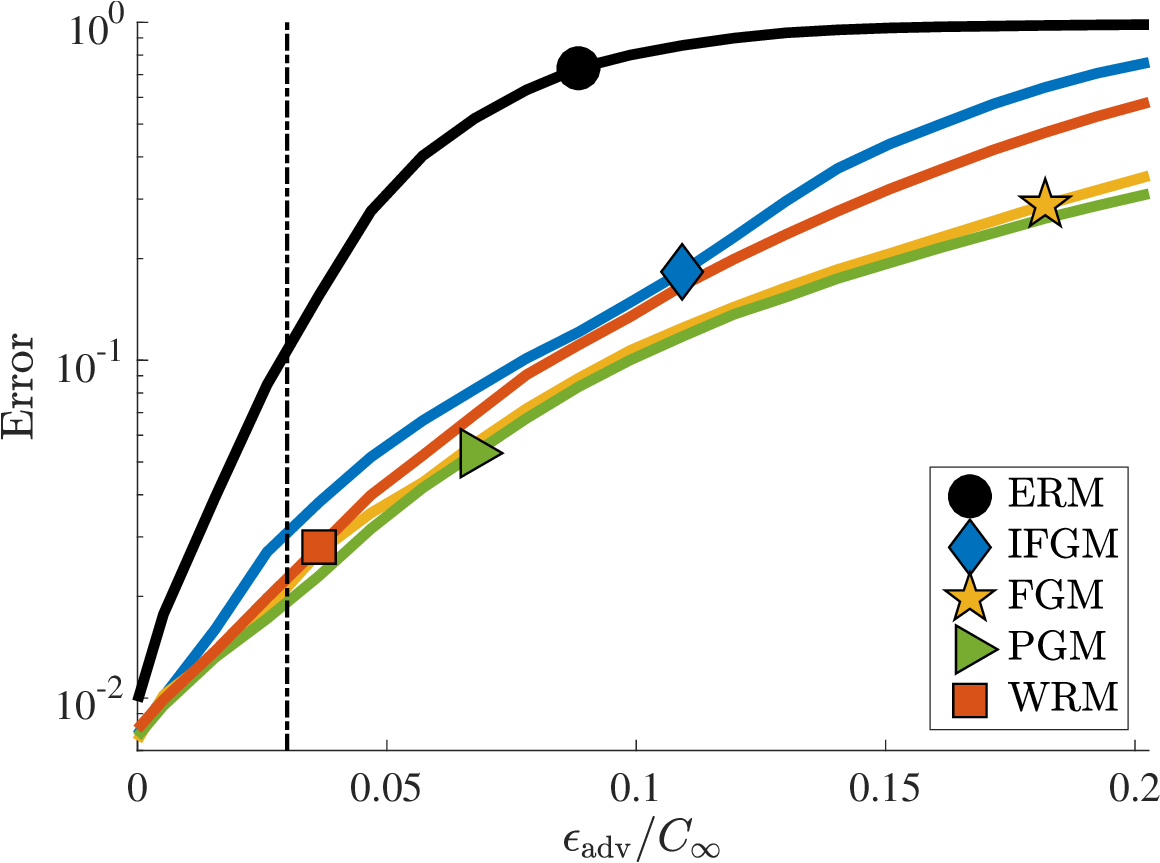}}%
\end{minipage}
\begin{minipage}{0.49\columnwidth}
\centering
\subfigure[Test error vs. $\epsilon_{\rm adv}$ for $\|\cdot\|_2$-IFGM attack]{\includegraphics[width=0.85\textwidth]{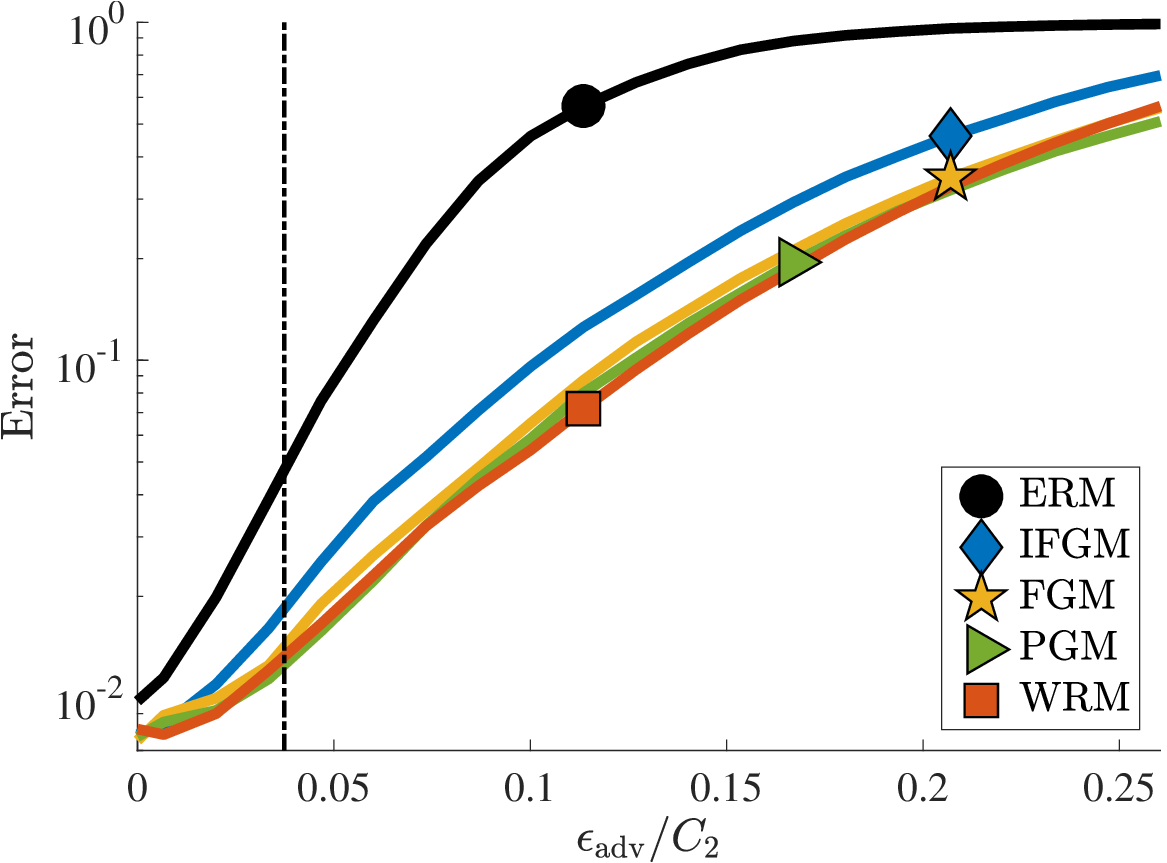}}%
\end{minipage}
\centering
\begin{minipage}{0.49\columnwidth}%
\centering
\subfigure[Test error vs. $\epsilon_{\rm adv}$ for $\|\cdot\|_{\infty}$-IFGM attack]{\includegraphics[width=0.85\textwidth]{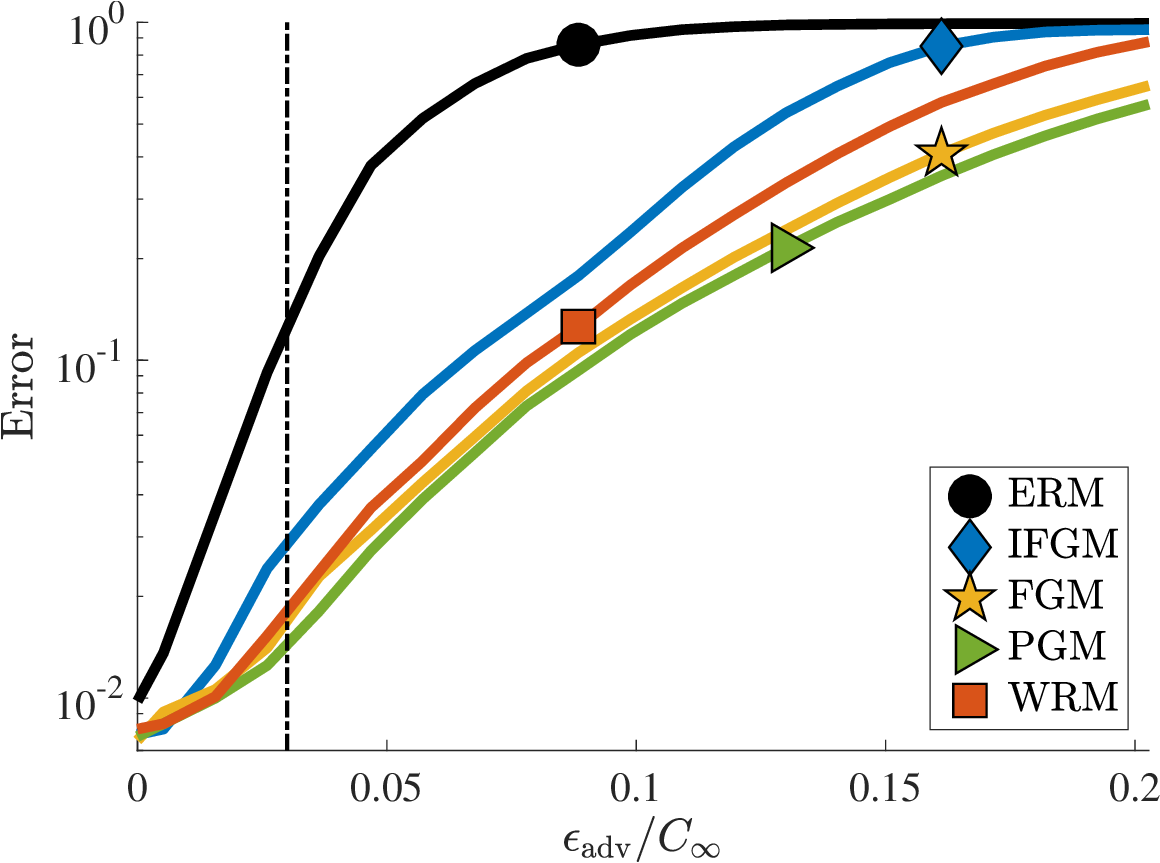}}%
\end{minipage}
\centering
\caption[]{\label{fig:mnist-infwrm2}Attacks on the MNIST
  dataset. We compare standard WRM with $\infty$-norm PGM, FGM, IFGM. We illustrate test misclassification error vs. the adversarial
  perturbation level $\epsilon_{\rm adv}$. Top row: PGM attacks, middle row: FGM attacks, bottom row: IFGM attacks. Left column: Euclidean-norm attacks, right column: $\infty$-norm attacks. The vertical bar
  in (a), (c), and (e) indicates the estimated radius
  $\sqrt{\what{\rho}_n(\theta_{\rm WRM})}$. The vertical bar in (b), (d), and (f) indicates the perturbation level that was used for training the PGM, FGM, and IFGM models via \eqref{eq:inf-from-2}.}
\end{figure}

\clearpage

\begin{figure}[!!t]
\begin{minipage}{0.49\columnwidth}
\centering
\subfigure[Test error vs. $\epsilon_{\rm adv}$ for $\|\cdot\|_2$-PGM attack]{\includegraphics[width=0.85\textwidth]{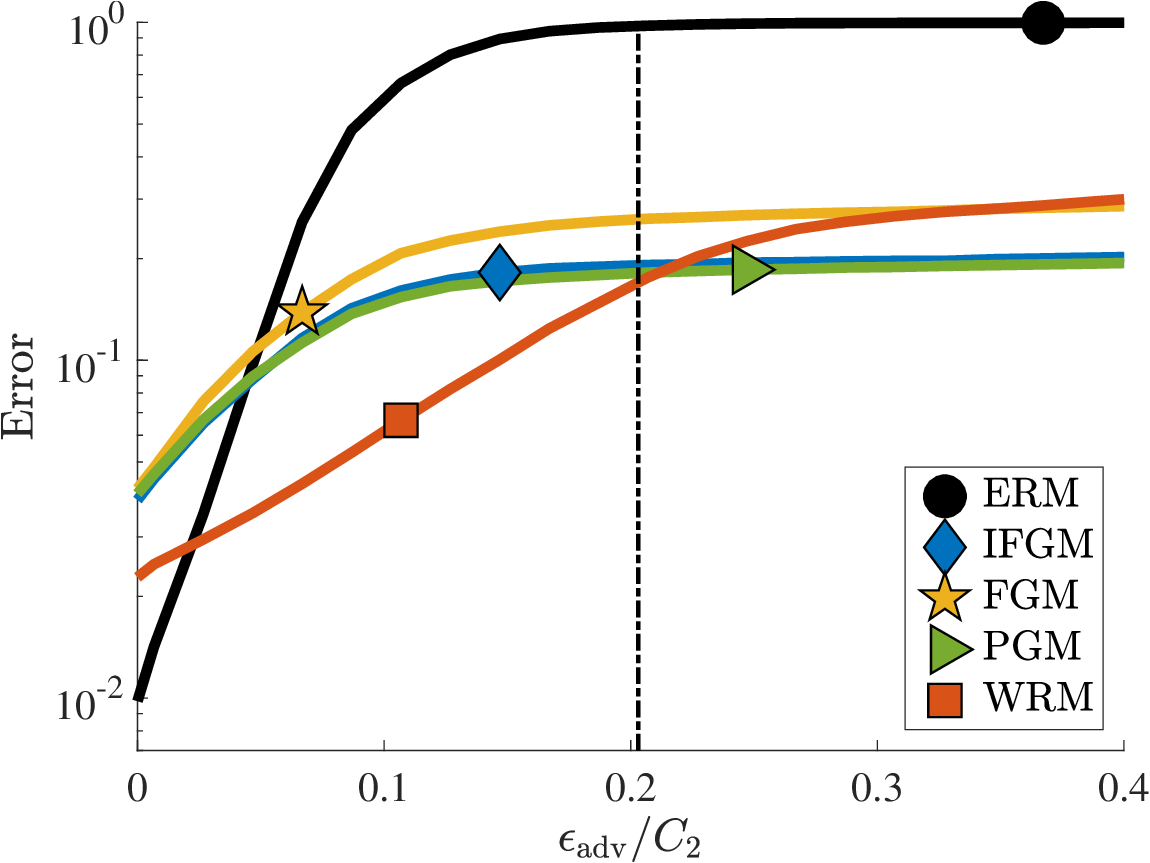}}%
\end{minipage}
\centering
\begin{minipage}{0.49\columnwidth}%
\centering
\subfigure[Test error vs. $\epsilon_{\rm adv}$ for $\|\cdot\|_{\infty}$-PGM attack]{\includegraphics[width=0.85\textwidth]{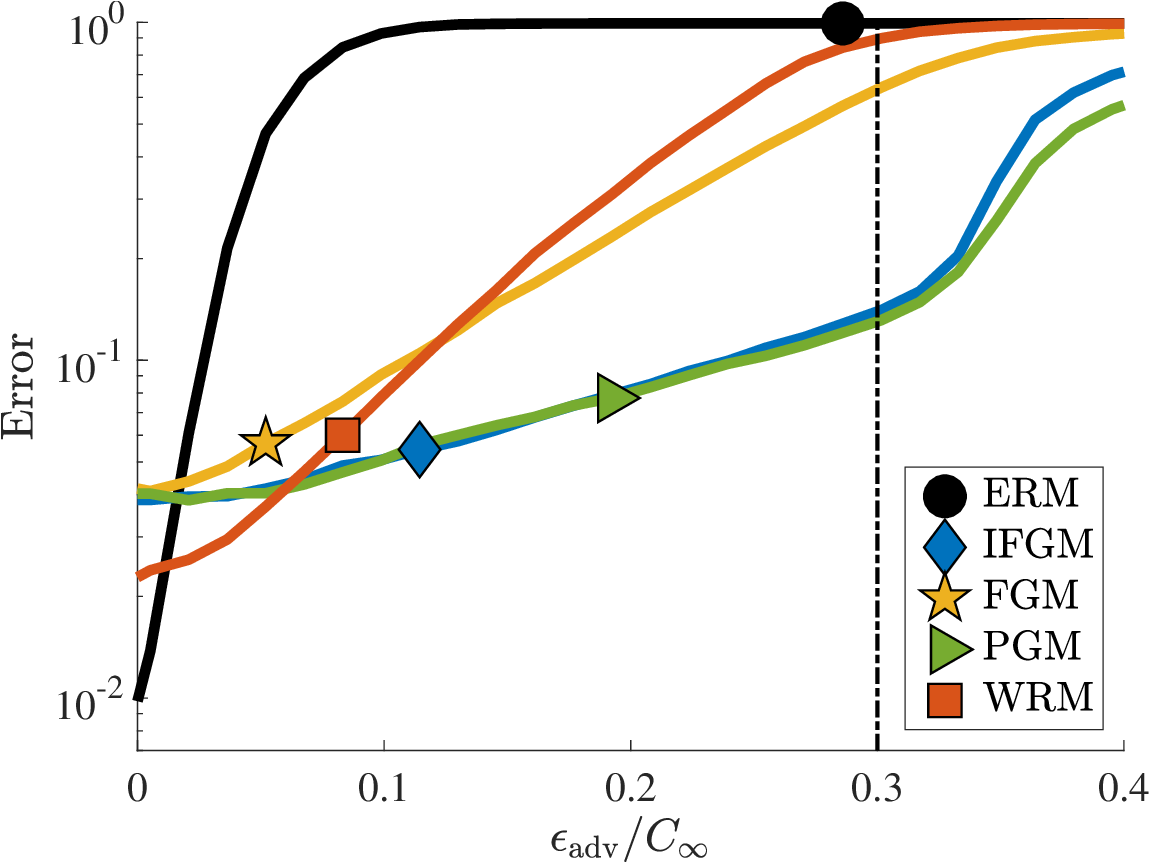}}%
\end{minipage}
\begin{minipage}{0.49\columnwidth}
\centering
\subfigure[Test error vs. $\epsilon_{\rm adv}$ for $\|\cdot\|_2$-FGM attack]{\includegraphics[width=0.85\textwidth]{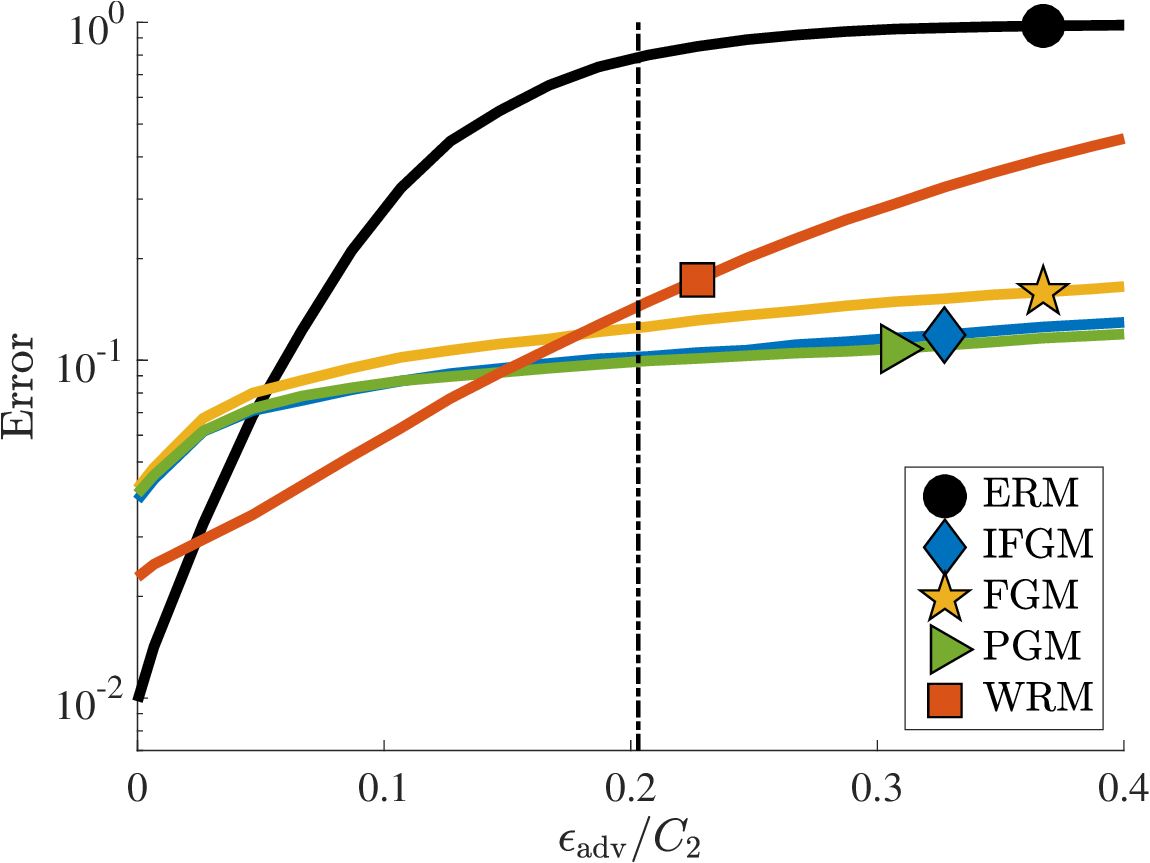}}%
\end{minipage}
\centering
\begin{minipage}{0.49\columnwidth}%
\centering
\subfigure[Test error vs. $\epsilon_{\rm adv}$ for $\|\cdot\|_{\infty}$-FGM attack]{\includegraphics[width=0.85\textwidth]{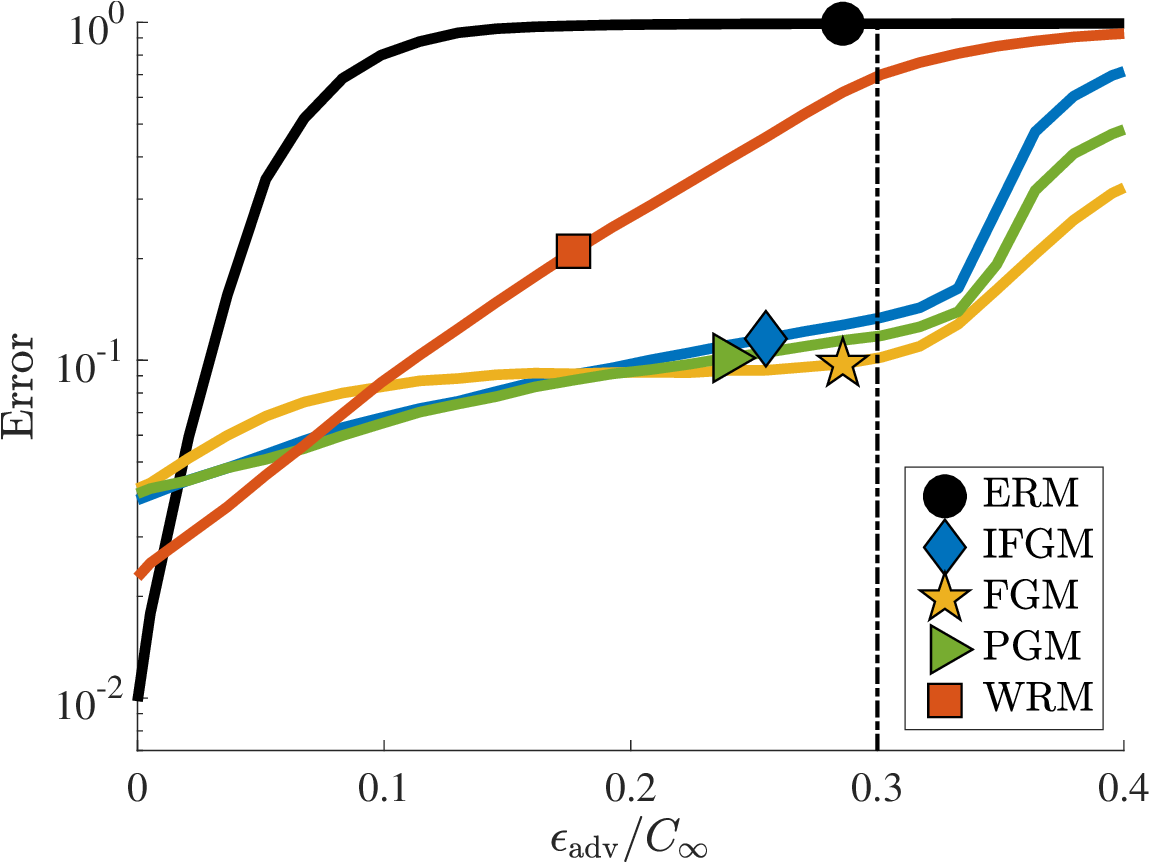}}%
\end{minipage}
\begin{minipage}{0.49\columnwidth}
\centering
\subfigure[Test error vs. $\epsilon_{\rm adv}$ for $\|\cdot\|_2$-IFGM attack]{\includegraphics[width=0.85\textwidth]{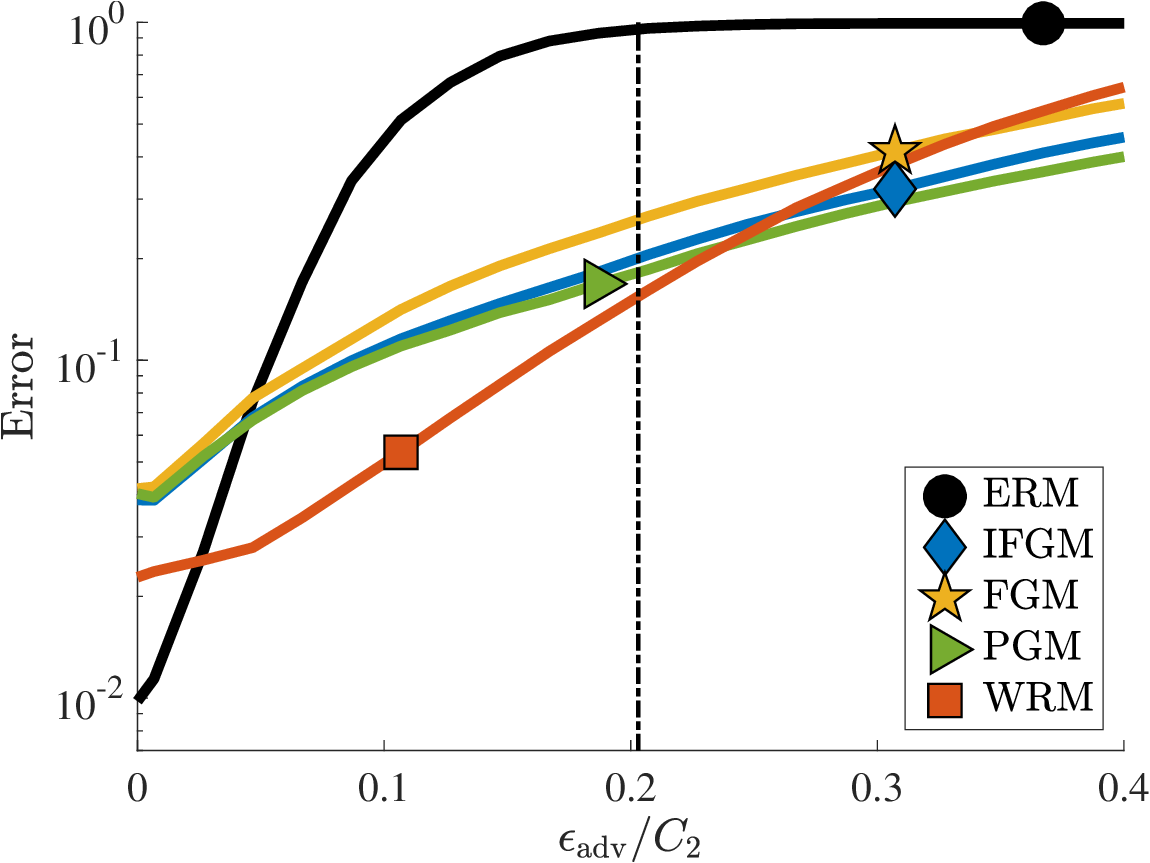}}%
\end{minipage}
\centering
\begin{minipage}{0.49\columnwidth}%
\centering
\subfigure[Test error vs. $\epsilon_{\rm adv}$ for $\|\cdot\|_{\infty}$-IFGM attack]{\includegraphics[width=0.85\textwidth]{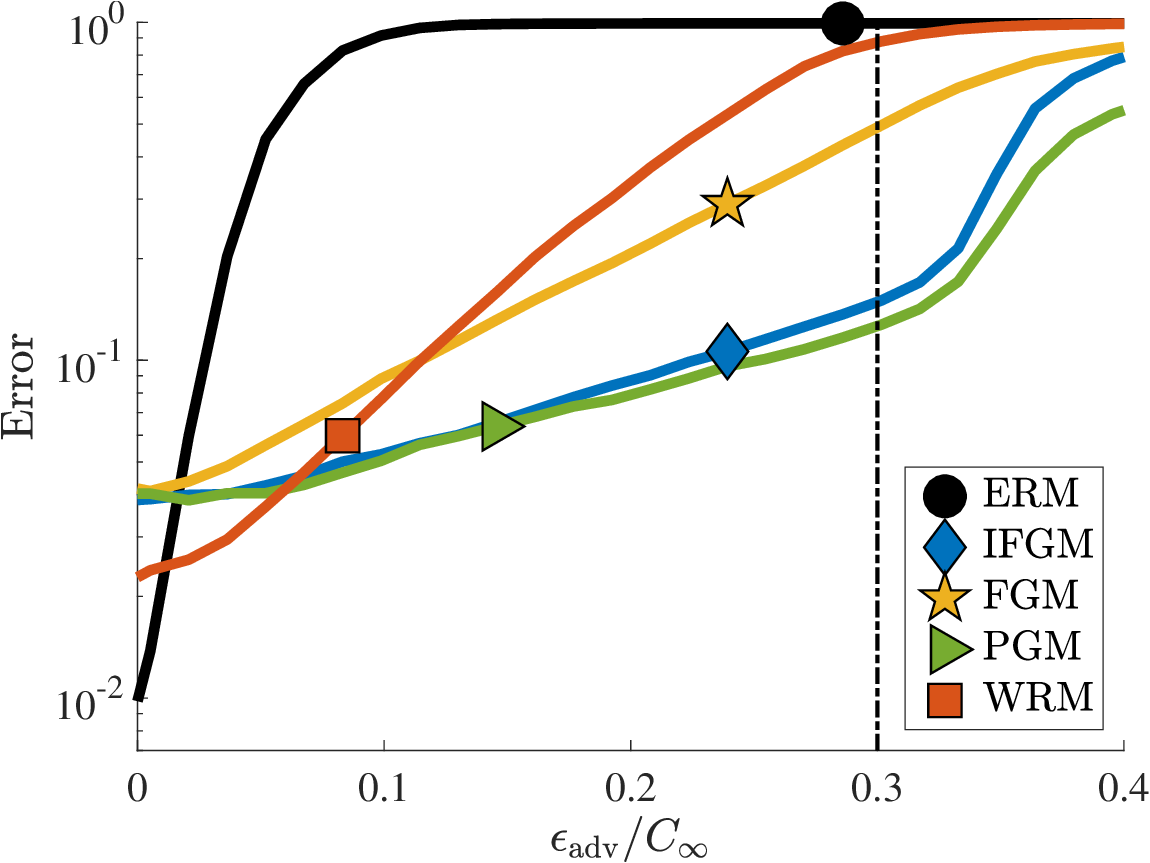}}%
\end{minipage}
\centering
\caption[]{\label{fig:mnist-infwrm2big}Attacks on the MNIST
  dataset with larger (training and test) adversarial budgets. We compare standard WRM with $\infty$-norm PGM, FGM, IFGM models. We illustrate test misclassification error vs. the adversarial
  perturbation level $\epsilon_{\rm adv}$. Top row: PGM attacks, middle row: FGM attacks, bottom row: IFGM attacks. Left column: Euclidean-norm attacks, right column: $\infty$-norm attacks. The vertical bar
  in (a), (c), and (e) indicates the estimated radius
  $\sqrt{\what{\rho}_n(\theta_{\rm WRM})}$. The vertical bar in (b), (d), and (f) indicates the perturbation level that was used for training the PGM, FGM, and IFGM models via \eqref{eq:inf-from-2}.}
\end{figure}

\clearpage

\begin{figure}[!!t]
\begin{minipage}{0.49\columnwidth}
\centering
\subfigure[Test error vs. $\epsilon_{\rm adv}$ for $\|\cdot\|_2$-PGM attack]{\includegraphics[width=0.85\textwidth]{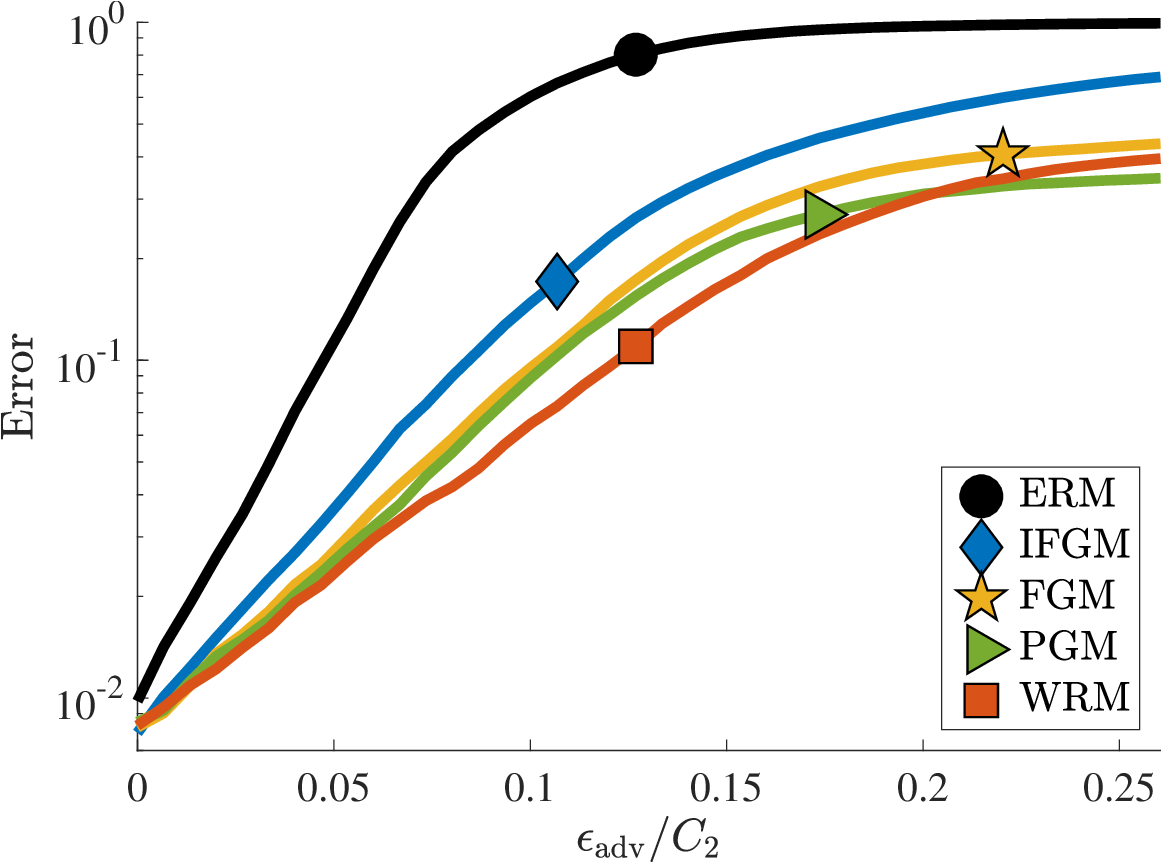}}%
\end{minipage}
\centering
\begin{minipage}{0.49\columnwidth}%
\centering
\subfigure[Test error vs. $\epsilon_{\rm adv}$ for $\|\cdot\|_{\infty}$-PGM attack]{\includegraphics[width=0.85\textwidth]{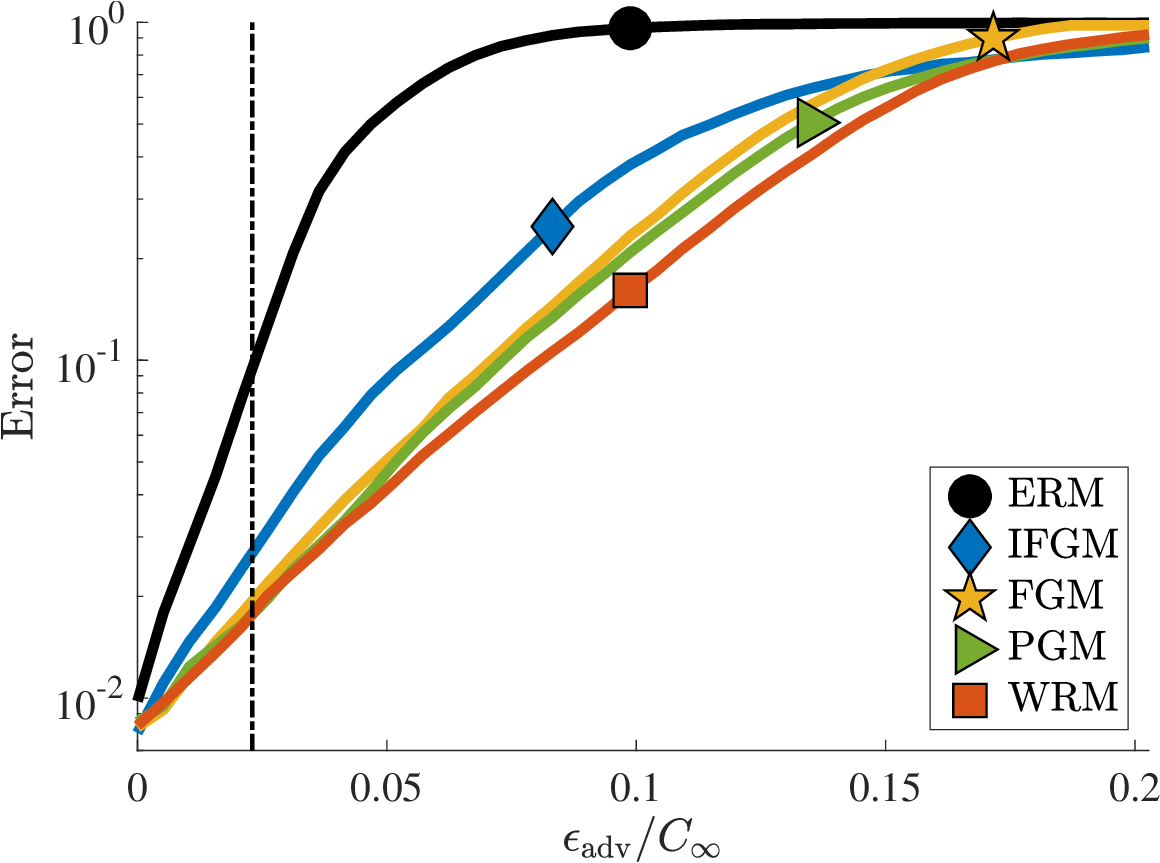}}%
\end{minipage}
\begin{minipage}{0.49\columnwidth}
\centering
\subfigure[Test error vs. $\epsilon_{\rm adv}$ for $\|\cdot\|_2$-FGM attack]{\includegraphics[width=0.85\textwidth]{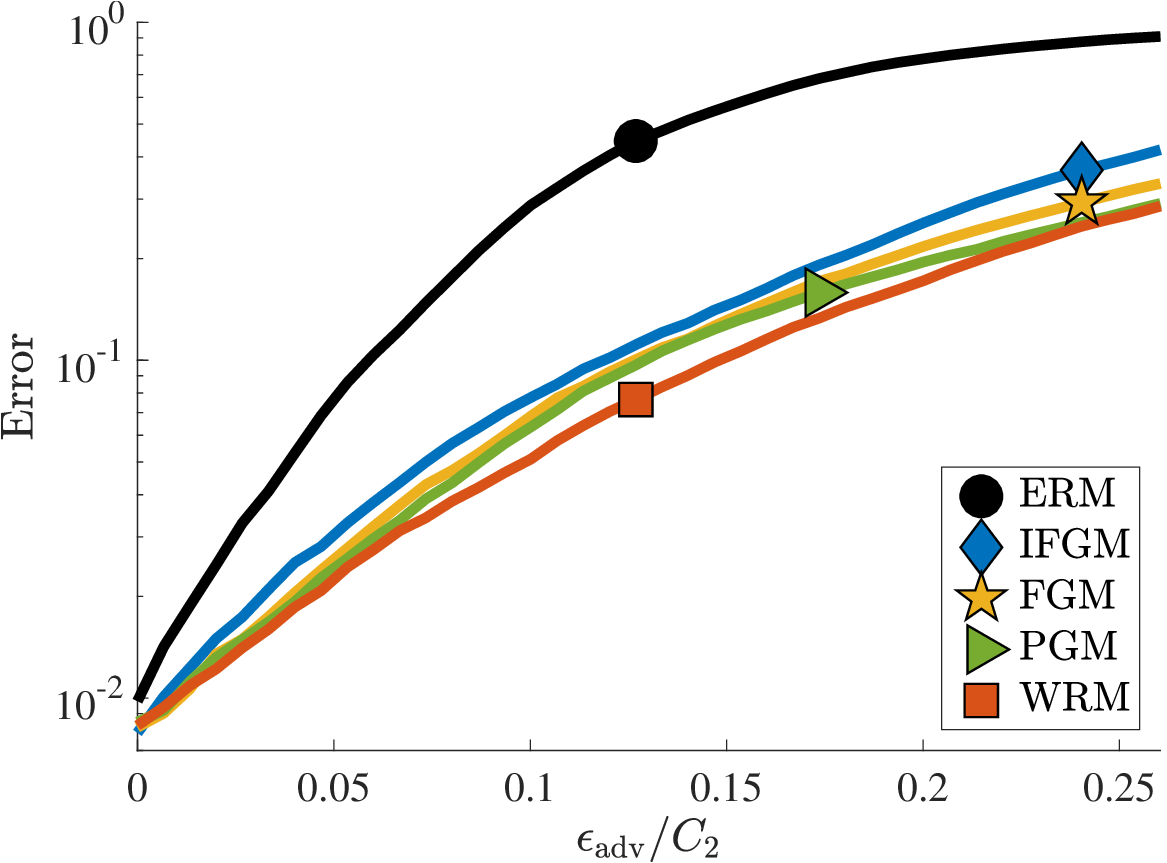}}%
\end{minipage}
\centering
\begin{minipage}{0.49\columnwidth}%
\centering
\subfigure[Test error vs. $\epsilon_{\rm adv}$ for $\|\cdot\|_{\infty}$-FGM attack]{\includegraphics[width=0.85\textwidth]{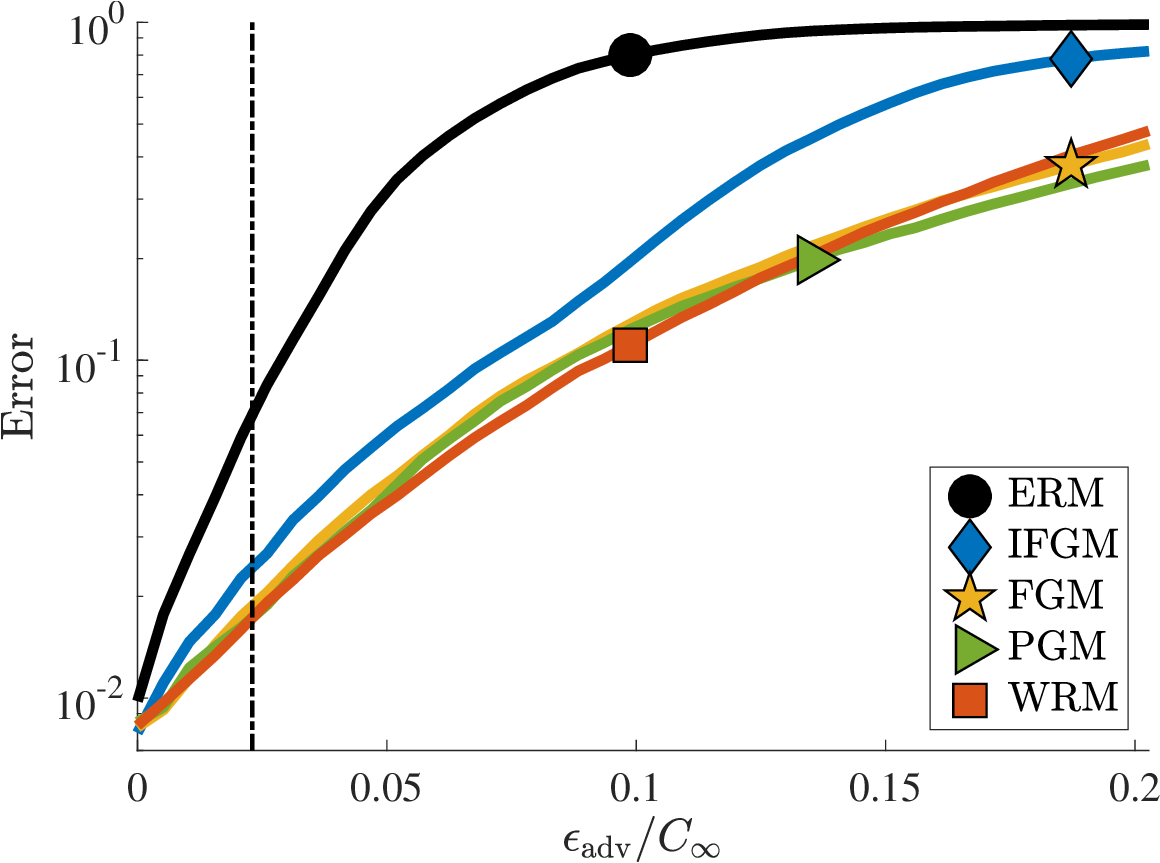}}%
\end{minipage}
\begin{minipage}{0.49\columnwidth}
\centering
\subfigure[Test error vs. $\epsilon_{\rm adv}$ for $\|\cdot\|_2$-IFGM attack]{\includegraphics[width=0.85\textwidth]{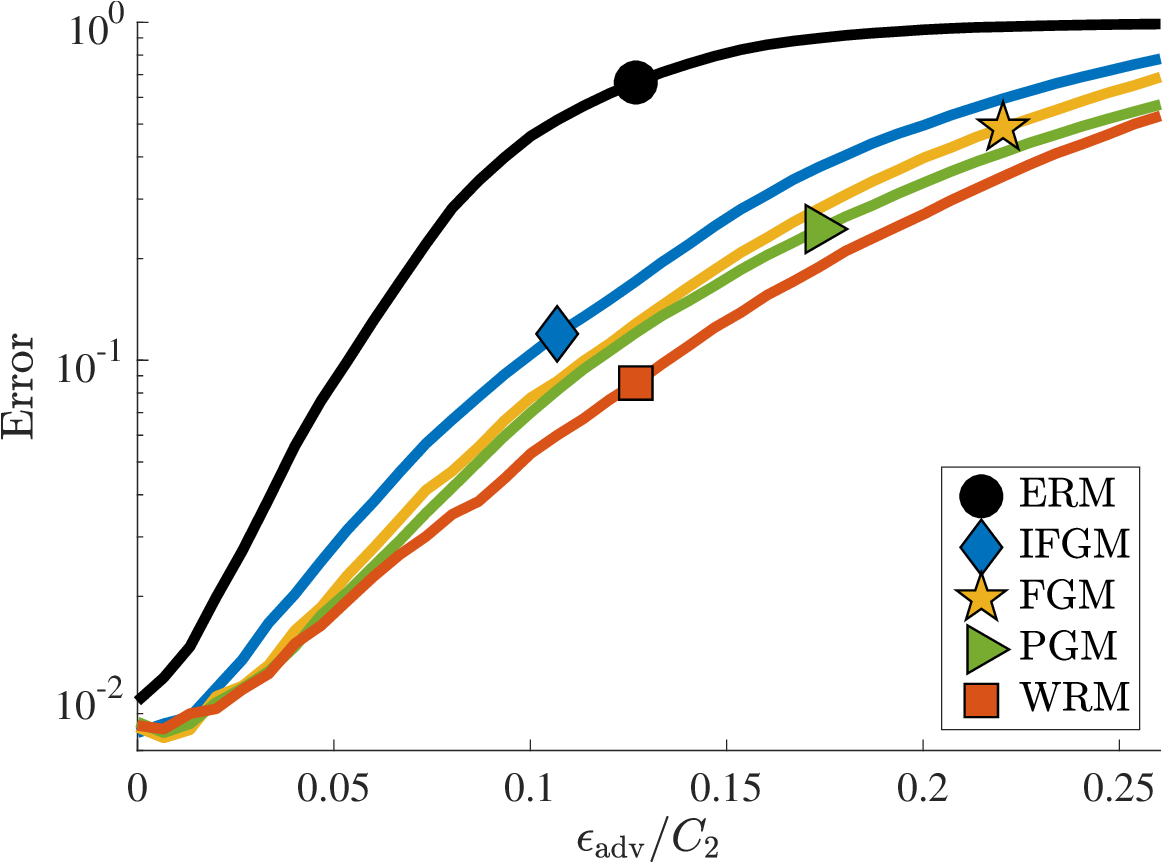}}%
\end{minipage}
\centering
\begin{minipage}{0.49\columnwidth}%
\centering
\subfigure[Test error vs. $\epsilon_{\rm adv}$ for $\|\cdot\|_{\infty}$-IFGM attack]{\includegraphics[width=0.85\textwidth]{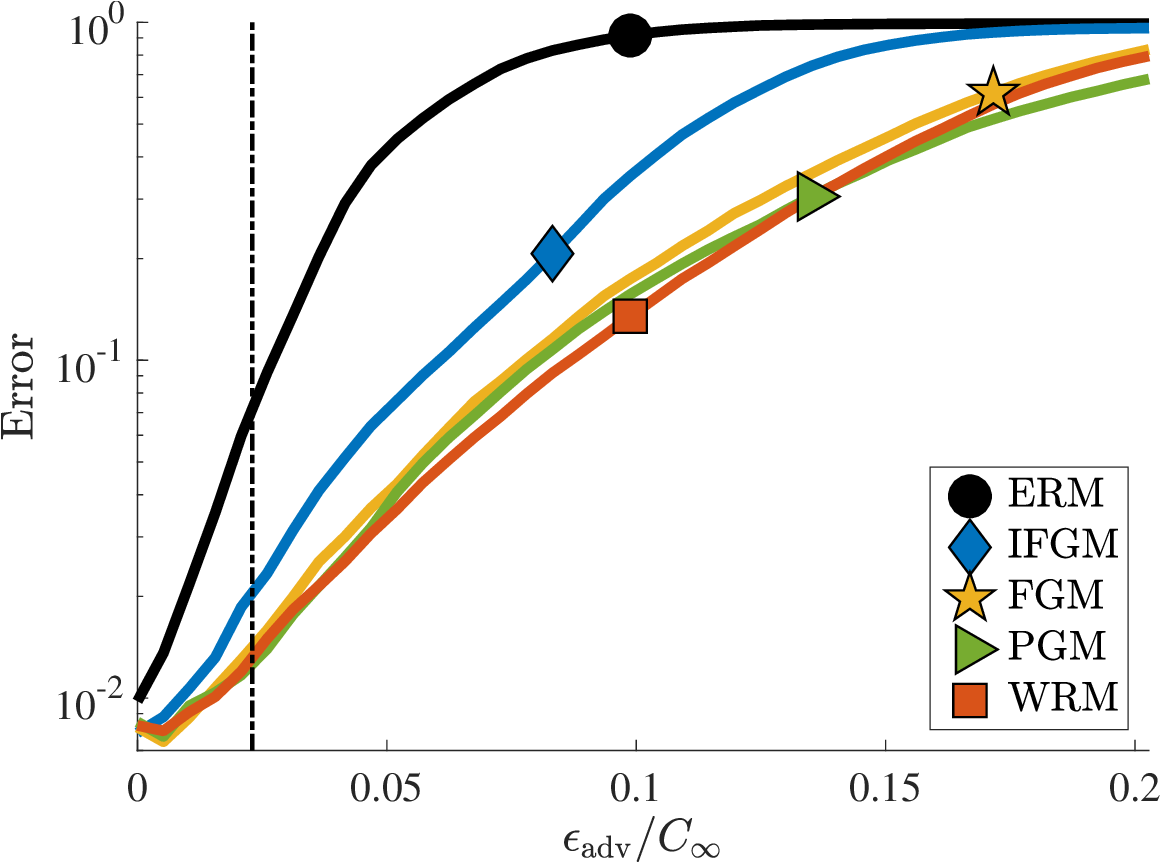}}%
\end{minipage}
\centering
\caption[]{\label{fig:mnist-proxsmall}Attacks on the MNIST
  dataset. All models are trained in the $\infty$-norm. We illustrate test misclassification error vs. the adversarial
  perturbation level $\epsilon_{\rm adv}$. Top row: PGM attacks, middle row: FGM attacks, bottom row: IFGM attacks. Left column: Euclidean-norm attacks, right column: $\infty$-norm attacks. The vertical bar in (b), (d), and (f) indicates the perturbation level that was used for training the PGM, FGM, and IFGM models and the estimated radius
  $\sqrt{\what{\rho}_n(\theta_{\rm WRM})}$.}
\end{figure}

\clearpage

\begin{figure}[!!t]
\begin{minipage}{0.49\columnwidth}
\centering
\subfigure[Test error vs. $\epsilon_{\rm adv}$ for $\|\cdot\|_2$-PGM attack]{\includegraphics[width=0.85\textwidth]{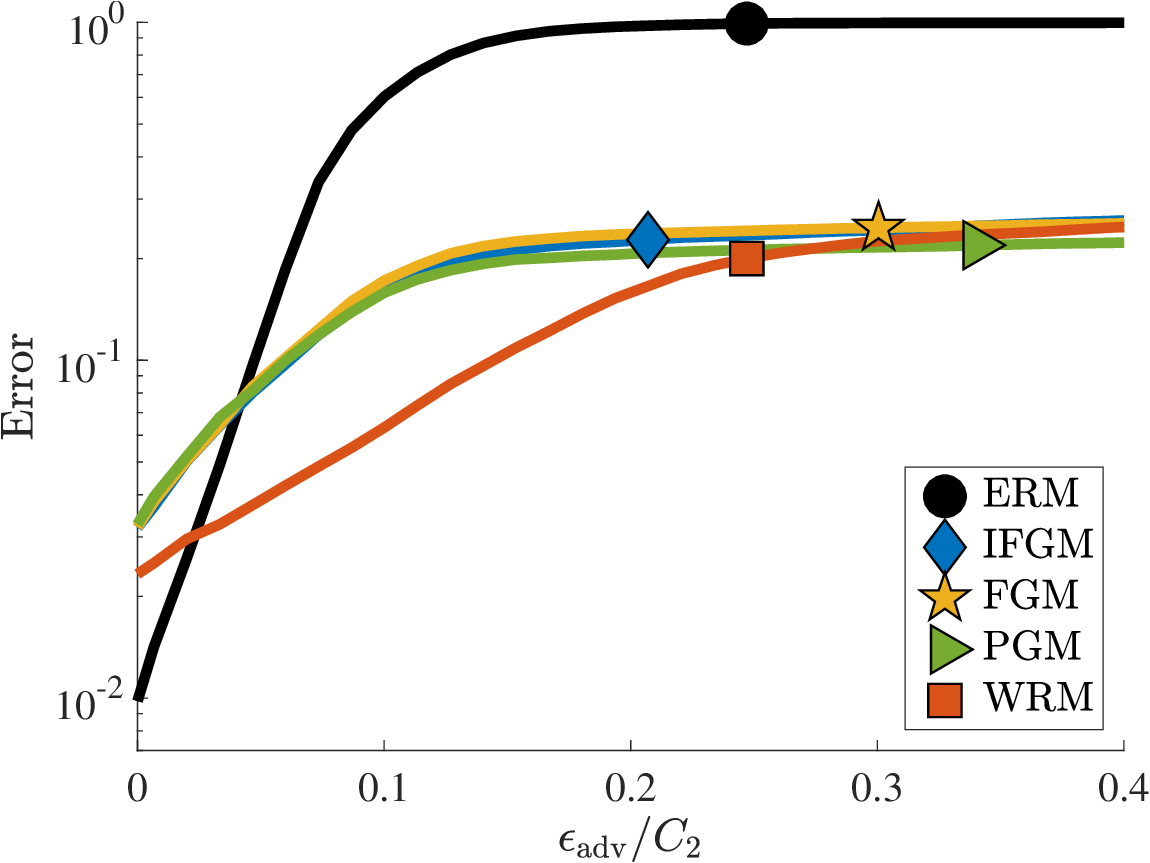}}%
\end{minipage}
\centering
\begin{minipage}{0.49\columnwidth}%
\centering
\subfigure[Test error vs. $\epsilon_{\rm adv}$ for $\|\cdot\|_{\infty}$-PGM attack]{\includegraphics[width=0.85\textwidth]{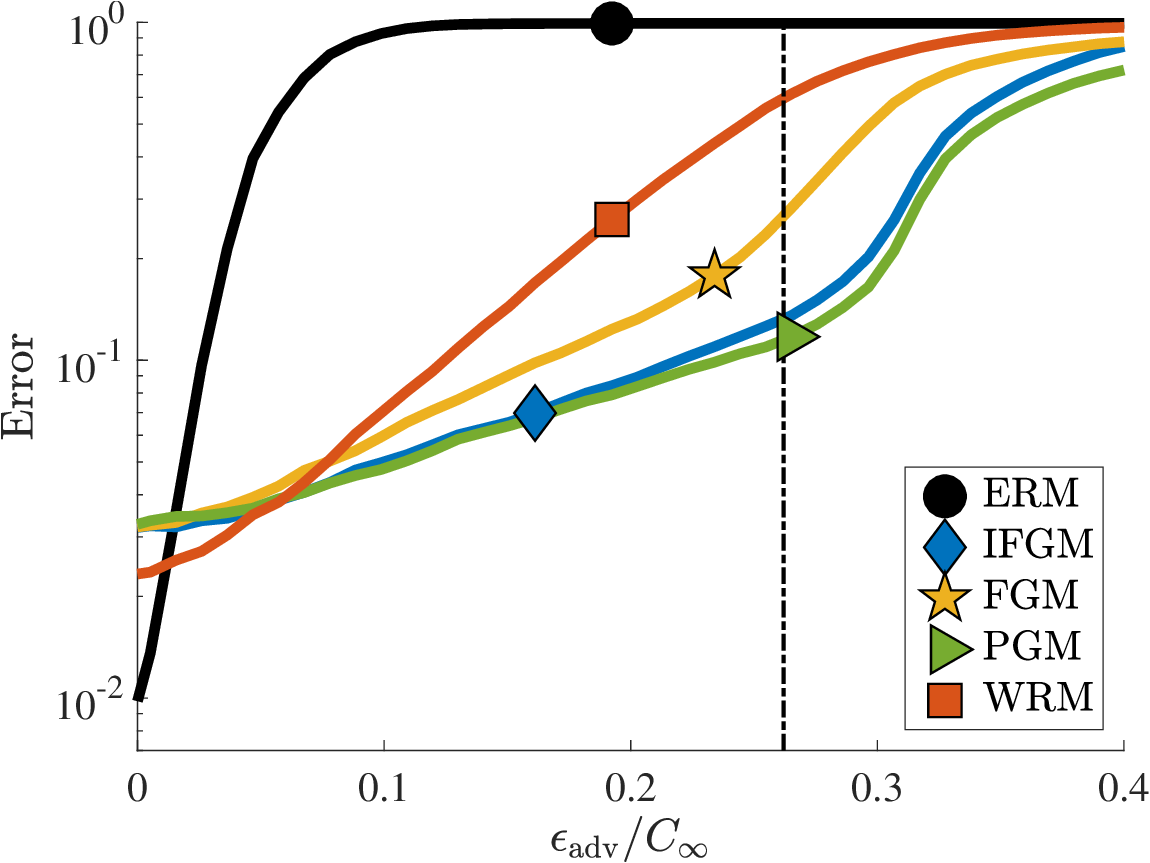}}%
\end{minipage}
\begin{minipage}{0.49\columnwidth}
\centering
\subfigure[Test error vs. $\epsilon_{\rm adv}$ for $\|\cdot\|_2$-FGM attack]{\includegraphics[width=0.85\textwidth]{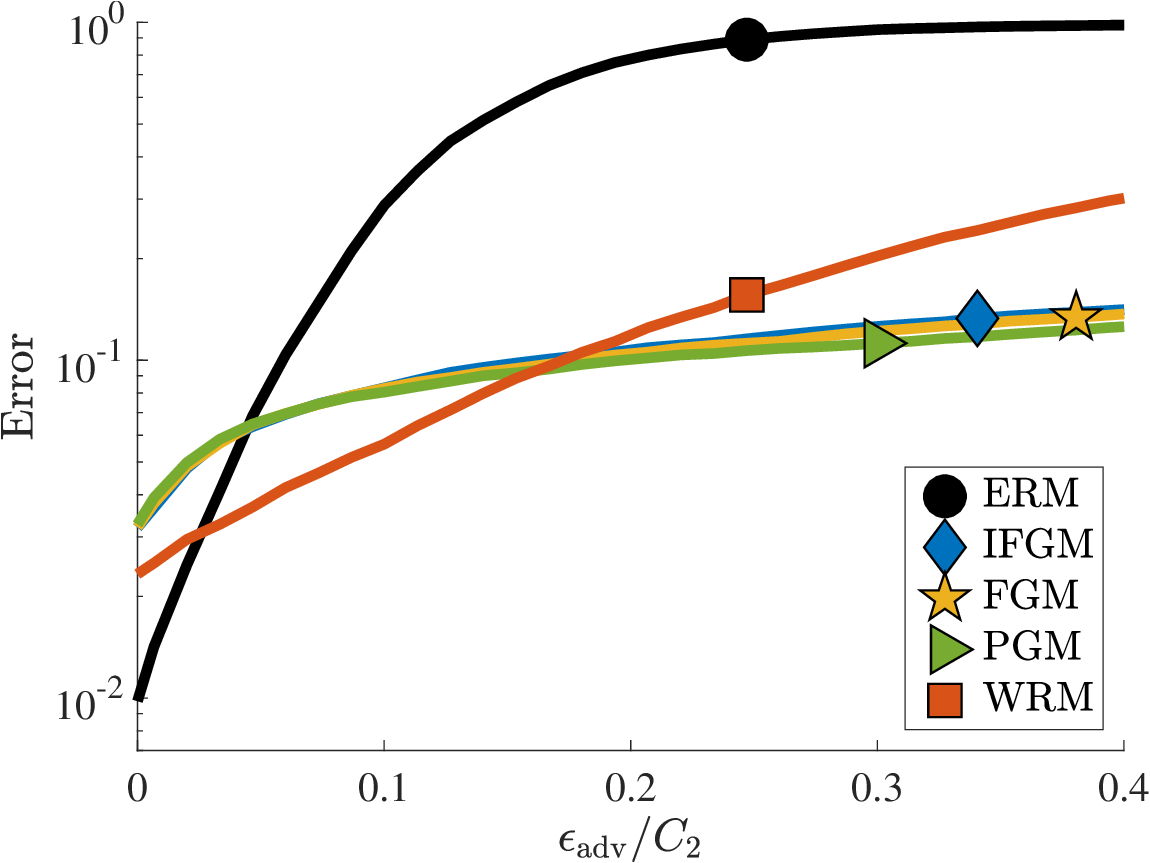}}%
\end{minipage}
\centering
\begin{minipage}{0.49\columnwidth}%
\centering
\subfigure[Test error vs. $\epsilon_{\rm adv}$ for $\|\cdot\|_{\infty}$-FGM attack]{\includegraphics[width=0.85\textwidth]{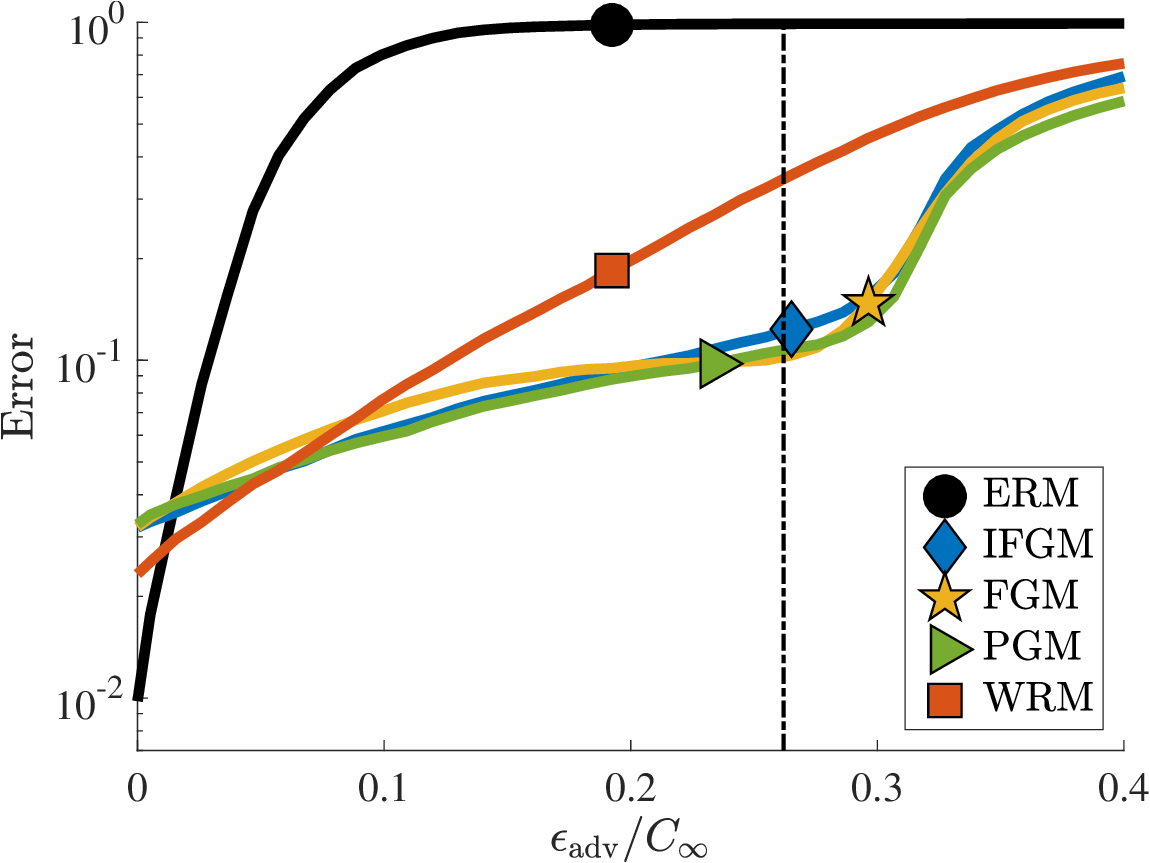}}%
\end{minipage}
\begin{minipage}{0.49\columnwidth}
\centering
\subfigure[Test error vs. $\epsilon_{\rm adv}$ for $\|\cdot\|_2$-IFGM attack]{\includegraphics[width=0.85\textwidth]{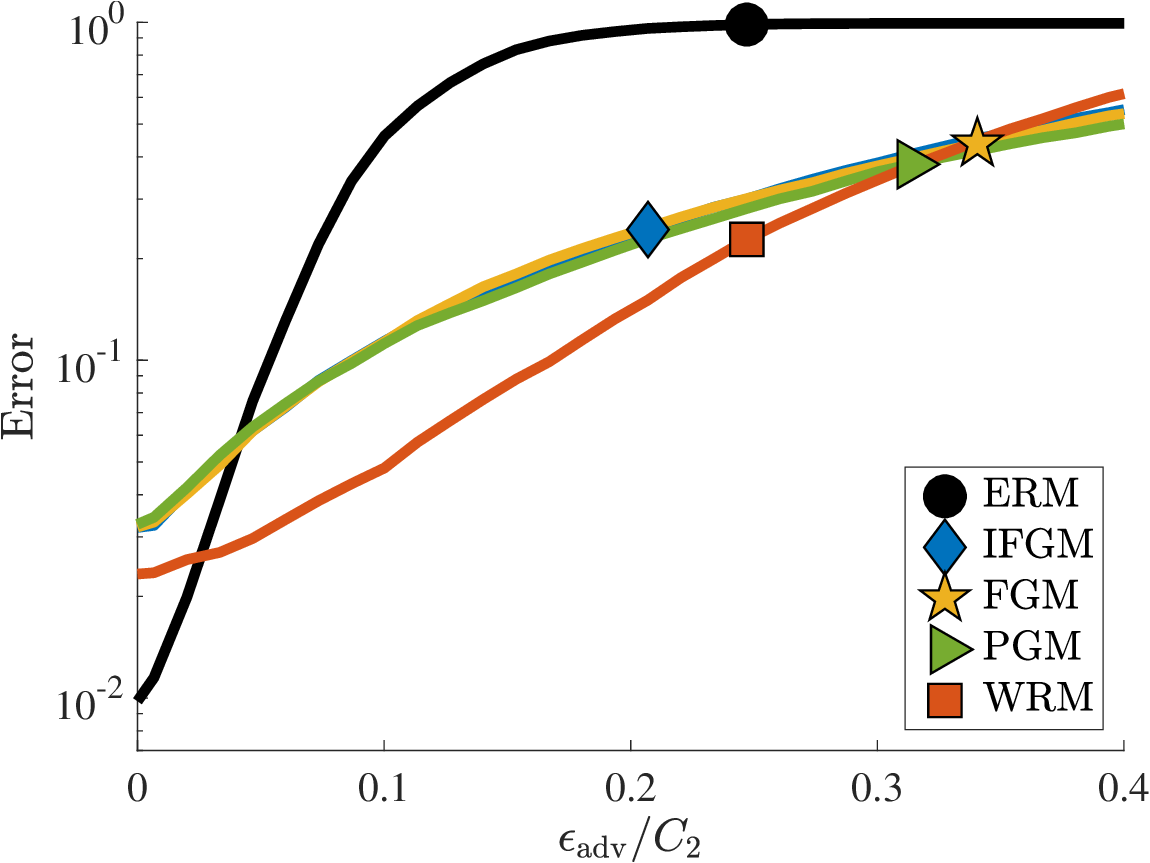}}%
\end{minipage}
\centering
\begin{minipage}{0.49\columnwidth}%
\centering
\subfigure[Test error vs. $\epsilon_{\rm adv}$ for $\|\cdot\|_{\infty}$-IFGM attack]{\includegraphics[width=0.85\textwidth]{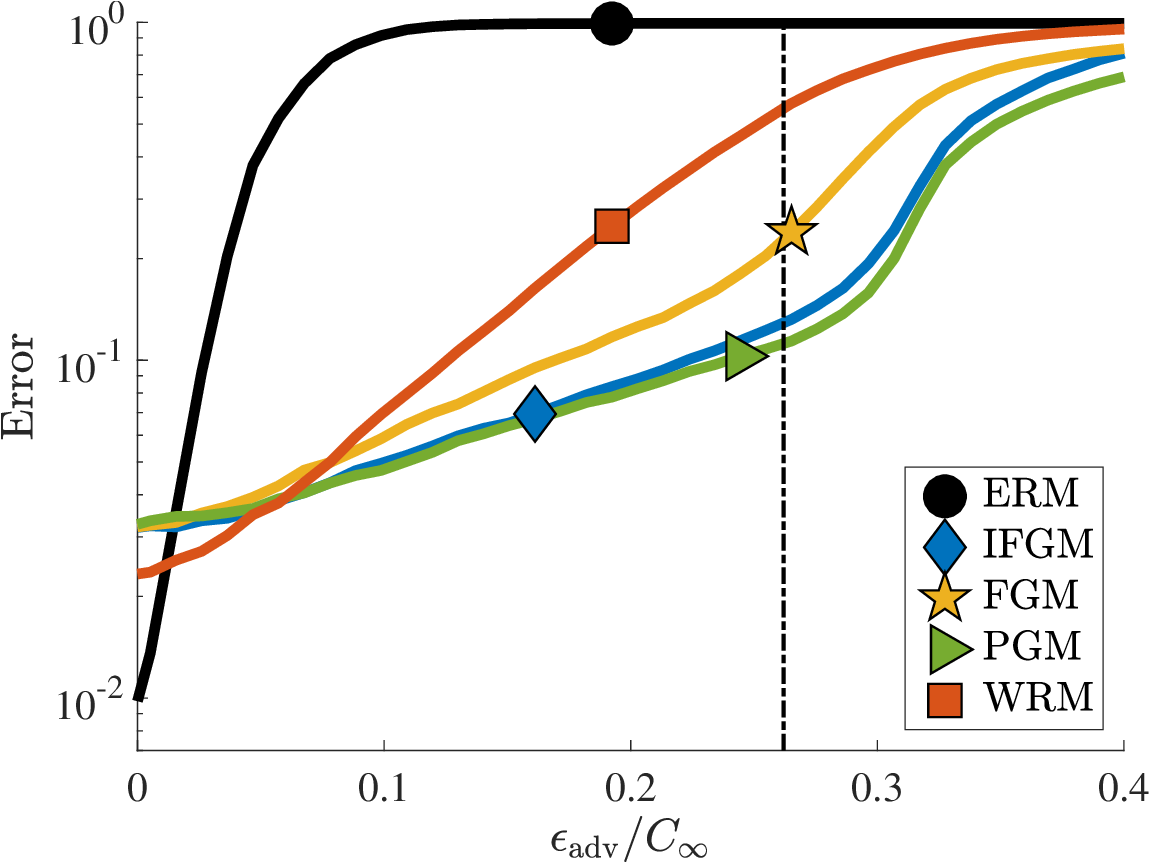}}%
\end{minipage}
\centering
\caption[]{\label{fig:mnist-proxbig}Attacks on the MNIST
  dataset with larger (training and test) adversarial budgets. All models are trained in the $\infty$-norm. We illustrate test misclassification error vs. the adversarial
  perturbation level $\epsilon_{\rm adv}$. Top row: PGM attacks, middle row: FGM attacks, bottom row: IFGM attacks. Left column: Euclidean-norm attacks, right column: $\infty$-norm attacks. The vertical bar in (b), (d), and (f) indicates the perturbation level that was used for training the PGM, FGM, and IFGM models and the estimated radius
  $\sqrt{\what{\rho}_n(\theta_{\rm WRM})}$.}
\end{figure}

\subsection{MNIST experiments when $\gamma$ is chosen according to Section~\ref{section:examples}}
\label{section:three-choose-two}

\begin{figure}[!t]
\begin{center}
{
\begin{tabular}{cc|lllll}
\multicolumn{1}{c}{$\gamma$} &  \multicolumn{1}{c|}{$\bar{\gamma}$}& \multicolumn{1}{c}{$\epsilon_{\rm adv}=0$} & \multicolumn{1}{c}{$\epsilon_{\rm adv}=0.05$} & \multicolumn{1}{c}{$\epsilon_{\rm adv}=0.10$} & \multicolumn{1}{c}{$\epsilon_{\rm adv}=0.15$} & \multicolumn{1}{c}{$\epsilon_{\rm adv}=0.20$}\\
\hline
\hline
ERM & 1.97$\times 10^6$  &   0.015 &  0.965  & 1  &  1 & 1 \\ \hline 
1 & 1.23$\times 10^6$    &   0.015 &  0.112  & 0.484  &  0.916 & 0.994 \\
1 $\times 10^2$ & 1.81$\times 10^6$     &   0.015 &  0.916  & 1  &  1 & 1 \\
1 $\times 10^4$ & 1.97$\times 10^6$   &   0.015 &  0.965  & 1  &  1 & 1 \\ 
1$\times 10^6$ & 1.97$\times 10^6$   &   0.015 &  0.966  & 1  &  1 & 1 \\ 
1 $\times 10^7$ & 1.97$\times 10^6$  &   0.015 &  0.966  & 1  &  1 & 1 \\ 
\hline
\hline
\end{tabular}
}
\end{center}

\captionof{table}{\label{tab:mnist-extra1} Performance of various WRM models. We show the smoothness upper bound ($\bar\gamma$) and test error vs. $\epsilon_{\rm adv}$ for $\|\cdot \|_{\infty}$-norm PGM attacks. }
\end{figure}

\begin{figure}[!t]
\begin{center}
{
\begin{tabular}{cc|lllll}
\multicolumn{1}{c}{$\gamma$} &  \multicolumn{1}{c|}{$\bar{\gamma}$}& \multicolumn{1}{c}{$\epsilon_{\rm adv}=0$} & \multicolumn{1}{c}{$\epsilon_{\rm adv}=0.05$} & \multicolumn{1}{c}{$\epsilon_{\rm adv}=0.10$} & \multicolumn{1}{c}{$\epsilon_{\rm adv}=0.15$} & \multicolumn{1}{c}{$\epsilon_{\rm adv}=0.20$}\\
\hline
\hline
ERM & 7.77$\times 10^3$  &   0.045 &  0.318  & 0.855  &  0.992 & 0.999 \\ \hline 
1 & 8.69$\times 10^3$    &   0.041 &  0.183  & 0.606  &  0.925 & 0.996 \\
1 $\times 10^1$ & 6.56$\times 10^3$     &   0.046 &  0.303  & 0.835  &  0.991 & 1.000 \\
1 $\times 10^2$ & 7.70$\times 10^3$   &   0.048 &  0.326  & 0.858  &  0.992 & 1.000 \\ 
1$\times 10^3$ & 7.84$\times 10^3$   &   0.046 &  0.320  & 0.857  &  0.992 & 0.999 \\ 
1 $\times 10^4$ & 7.75$\times 10^3$  &   0.048 &  0.329  & 0.860  &  0.993 & 1.000 \\
\hline
\hline
\end{tabular}
}
\end{center}

\captionof{table}{\label{tab:mnist-extra2} Performance of various $l_2$-regularized WRM models. The regularization multiplier is 0.05 for the first half of the neural net layers and 0.01 for the latter half. We show the smoothness upper bound ($\bar\gamma$) and test error vs. $\epsilon_{\rm adv}$ for $\|\cdot \|_{\infty}$-norm PGM attacks. }
\end{figure}

\begin{figure}[!t]
\begin{center}
{
\begin{tabular}{cc|lllll}
\multicolumn{1}{c}{$\gamma$} &  \multicolumn{1}{c|}{$\bar{\gamma}$}& \multicolumn{1}{c}{$\epsilon_{\rm adv}=0$} & \multicolumn{1}{c}{$\epsilon_{\rm adv}=0.05$} & \multicolumn{1}{c}{$\epsilon_{\rm adv}=0.10$} & \multicolumn{1}{c}{$\epsilon_{\rm adv}=0.15$} & \multicolumn{1}{c}{$\epsilon_{\rm adv}=0.20$}\\
\hline
\hline
ERM & 5.06$\times 10^3$  &   0.059 &  0.353  & 0.881  &  0.996 & 1.000 \\ \hline 
1 & 5.80$\times 10^3$    &   0.056 &  0.225  & 0.657  &  0.941 & 0.997 \\
1$\times 10^1$ & 4.40$\times 10^3$     &   0.057 &  0.325  & 0.853  &  0.994 & 1 \\
1$\times 10^2$ & 4.65$\times 10^3$   &   0.059 &  0.352  & 0.872  &  0.995 & 1 \\ 
1$\times 10^3$ & 4.71$\times 10^3$   &   0.060 &  0.346  & 0.876  &  0.991 & 1.000 \\ 
1 $\times 10^4$ & 4.76$\times 10^3$  &   0.060 &  0.349  & 0.851  &  0.993 & 1.000 \\
\hline
\hline
\end{tabular}
}
\end{center}

\captionof{table}{\label{tab:mnist-extra3} Performance of various $l_2$-regularized WRM models. The regularization multiplier is 0.1 for the first half of the neural net layers and 0.01 for the latter half. We show the smoothness upper bound ($\bar\gamma$) and test error vs. $\epsilon_{\rm adv}$ for $\|\cdot \|_{\infty}$-norm PGM attacks.}
\end{figure}

\begin{figure}[!t]
\begin{center}
{
\begin{tabular}{cc|lllll}
\multicolumn{1}{c}{$\gamma$} &  \multicolumn{1}{c|}{$\bar{\gamma}$}& \multicolumn{1}{c}{$\epsilon_{\rm adv}=0$} & \multicolumn{1}{c}{$\epsilon_{\rm adv}=0.05$} & \multicolumn{1}{c}{$\epsilon_{\rm adv}=0.10$} & \multicolumn{1}{c}{$\epsilon_{\rm adv}=0.15$} & \multicolumn{1}{c}{$\epsilon_{\rm adv}=0.20$}\\
\hline
\hline
ERM & 5.76$\times 10^2$  &   0.117 &  0.433  & 0.852  &  0.988 & 1.000 \\ \hline 
1 & 5.87$\times 10^2$    &   0.117 &  0.364  & 0.717  &  0.952 & 0.996 \\
1$\times 10^1$ & 7.61$\times 10^2$     &   0.116 &  0.429  & 0.832  &  0.986 & 1.000 \\
1$\times 10^2$ & 8.03$\times 10^2$   &   0.121 &  0.442  & 0.849  &  0.986 & 1.000 \\ 
5$\times 10^2$ & 8.06$\times 10^2$   &   0.120 &  0.443  & 0.857  &  0.990 & 1.000 \\
1$\times 10^3$ & 8.04$\times 10^2$   & 0.115 & 0.436 & 0.846 & 0.987 & 1\\
\hline
\hline
\end{tabular}
}
\end{center}

\captionof{table}{\label{tab:mnist-extra4} Performance of various $l_2$-regularized WRM models. The regularization multiplier is 0.5 for the first half of the neural net layers and 0.01 for the latter half. We show the smoothness upper bound ($\bar\gamma$) and test error vs. $\epsilon_{\rm adv}$ for $\|\cdot \|_{\infty}$-norm PGM attacks. }
\end{figure}
We now present analysis of principled experiments on the MNIST dataset. For
the chosen architecture, we find that it is difficult to satisfy the following
three objectives: $\gamma$ is larger than the bound given by
Corollary~\ref{corollary:softmax-smooth}, the model has high enough capacity
such that the test accuracy on clean, unperturbed data is on par with standard
benchmarks (less than $10\%$ test error for the MNIST dataset), and the WRM
model's performance differs appreciably from ERM's performance. Table
\ref{tab:mnist-extra1} tests ERM and WRM models trained with different
$\gamma$ with $\|\cdot \|_{\infty}$ PGM attacks of various
$\epsilon_{\rm adv}$. Although all models enjoy good test accuracy on clean
test examples, we see that when $\gamma$ is large enough to satisfy the bound
of Corollary~\ref{corollary:softmax-smooth}, WRM and ERM do not differ in
performance, as the perturbed examples during training are essentially the
same as the originals.

To ameliorate this issue, we regularize the weights to decrease the bound of
Corollary~\ref{corollary:softmax-smooth}. Tables~\ref{tab:mnist-extra2},~\ref{tab:mnist-extra3},
and~\ref{tab:mnist-extra4} present the same analysis for architectures with
different $l_2$-regularization schemes. We choose to regularize the earlier
layers (those closer to the inputs) more heavily than later layers (those
closer to the output), as the bound in
Corollary~\ref{corollary:softmax-smooth} scales with the norm of earlier layer
weights exponentially in the depth of the network. We see that although 
 
Overall, we see that high accuracy on clean test examples and appreciable
adversarial robustness hold comes at the expense of $\gamma < \bar{\gamma}$
(the top rows of all tables). For $\gamma \gtrsim \bar{\gamma}$, high accuracy
on clean test examples hold, but adversarial robustness seems similar to that
of ERM (the bottom rows of all tables). For the most regularized models in
Table~\ref{tab:mnist-extra4}, the large values of $\gamma$ achieves
appreciable adversarial robustness, although the regularization is too heavy
for good performance on clean examples. For this architecture, we are unable
unable to easily discover a paramtrization that satisfied all three
objectives.

\section{Proofs}

\subsection{Proof of Proposition~\ref{prop:duality}}
\label{sec:proof-duality}

For completeness, we provide an alternative proof to that given
in~\citet{BlanchetMu16} using convex analysis. Our proof is less general, requiring the cost function $c$ to be continuous and convex in its first
argument. The below general duality result gives
Proposition~\ref{prop:duality} as an immediate special case.  Recalling
\citet[Def.~14.27 and Prop.~14.33]{RockafellarWe98}, we say that a function
$g : X \times Z \to \R$ is a \emph{normal integrand} if for each $\alpha$, the
mapping
\begin{equation*}
  z \mapsto \{x \mid g(x, z) \le \alpha\}
\end{equation*}
is closed-valued and measurable.  We recall that if $g$ is continuous, then
$g$ is a normal integrand~\cite[Cor.~14.34]{RockafellarWe98};
therefore, $g(x, z) = \gamma c(x, z)
- \loss(\theta; x)$ is a normal integrand. We have the
following theorem.
\begin{theorem}
  Let $f, c$ be such that for any $\gamma \ge 0$, the function $g(x, z) =
  \gamma c(x, z) - f(x)$ is a normal integrand.  (For example, continuity of
  $f$ and closed convexity of $c$ is sufficient.)  For any $\rho > 0$ we
  have
  \begin{equation*}
    \sup_{P : W_c(P, Q)}
    \int f(x) dP(x)
    =
    \inf_{\gamma \ge 0}
    \left\{\int \sup_{x \in X}
    \left\{f(x) - \gamma c(x, z)\right\} dQ(z)
    + \gamma \rho \right\}.
  \end{equation*}
\end{theorem}
\begin{proof}
  First, the mapping $P \mapsto W_c(P, Q)$ is convex in the space of
  probability measures. As taking $P = Q$ yields $W_c(Q, Q) = 0$, Slater's
  condition holds and we may apply standard (infinite dimensional) duality
  results~\cite[Thm.~8.7.1]{Luenberger69} to obtain
  \begin{align*}
    \sup_{P : W_c(P, Q)}
    \int f(x) dP(x)
    & =
    \sup_{P : W_c(P, Q)}
    \inf_{\gamma \ge 0}
    \left\{ \int f(x) dP(x)
    - \gamma W_c(P, Q) + \gamma \rho \right\} \\
    & =
    \inf_{\gamma \ge 0}
    \sup_{P : W_c(P, Q)}
    \left\{ \int f(x) dP(x)
    - \gamma W_c(P, Q) + \gamma \rho \right\}.
  \end{align*}
  Now, noting that
  for any $M \in \joints(P, Q)$ we have
  $\int f dP = \iint f(x) dM(x, z)$, we have that the
  rightmost quantity in the preceding display satisfies
  \begin{equation*}
    \int f(x) dP(x)
    - \gamma \inf_{M \in \joints(P, Q)}
    \int c(x, z) dM(x, z)
    = \sup_{M \in \joints(P, Q)}
    \left\{\int [f(x) - \gamma c(x, z)]
    dM(x, z) \right\}.
  \end{equation*}
  That is, we have
  \begin{equation}
    \sup_{P : W_c(P, Q)}
    \int f(x) dP(x)
    = \inf_{\gamma \ge 0}
    \sup_{P, M \in \joints(P, Q)}
    \left\{\int [f(x) - \gamma c(x, z)]
    dM(x, z) + \gamma \rho \right\}.
    \label{eqn:basic-duality}
  \end{equation}

  Now, we note a few basic facts. First, because we have a joint
  supremum over $P$ and measures $M \in \joints(P, Q)$ in
  expression~\eqref{eqn:basic-duality}, we have that
  \begin{equation*}
    \sup_{P, M \in \joints(P, Q)}
    \int [f(x) - \gamma c(x, z)]
    dM(x, z)
    \le \int \sup_x [f(x) - \gamma c(x, z)] dQ(z).
  \end{equation*}
  We would like to show equality in the above. To that end,
  we note that if $\mathcal{P}$ denotes the space of regular conditional
  probabilities (Markov kernels) from $Z$ to $X$, then
  \begin{equation*}
    \sup_{P, M \in \joints(P, Q)}
    \int [f(x) - \gamma c(x, z)]
    dM(x, z)
    \ge \sup_{P \in \mathcal{P}}
    \int [f(x) - \gamma c(x, z)] dP(x \mid z) dQ(z).
  \end{equation*}
  Recall that a conditional distribution $P(\cdot \mid z)$ is regular if
  $P(\cdot \mid z)$ is a distribution for each $z$ and for each measurable
  $A$, the function $z \mapsto P(A \mid z)$ is measurable.  Let $\mathcal{X}$
  denote the space of all measurable mappings $z \mapsto x(z)$ from $Z$ to
  $X$.  Using the powerful measurability results of
  \citet[Theorem 14.60]{RockafellarWe98}, we have
  \begin{equation*}
    \sup_{x \in \mathcal{X}}
    \int [f(x(z)) - \gamma c(x(z), z)] dQ(z)
    = \int \sup_{x \in X}
    [f(x) - \gamma c(x, z)] dQ(z)
  \end{equation*}
  because $f - c$ is upper semi-continuous, and the latter
  function is measurable.
  Now, let $x(z)$ be any measurable function that is $\epsilon$-close
  to attaining the supremum above. Define
  the conditional distribution $P(\cdot \mid z)$ to be supported
  on $x(z)$, which is evidently measurable. Then
  using the preceding display, we have
  \begin{align*}
    \int [f(x) - \gamma c(x, z)] dP(x \mid z) dQ(z)
    & = \int [f(x(z)) - \gamma c(x(z), z)] dQ(z) \\
    & \ge \int \sup_{x \in X}
    [f(x) - \gamma c(x, z)] dQ(z) - \epsilon \\
    & \ge \sup_{P, M \in \joints(P, Q)}
    \int [f(x) - \gamma c(x, z)] dM(x, z) - \epsilon.
  \end{align*}
  As $\epsilon > 0$ is arbitrary, this gives
  \begin{equation*}
    \sup_{P, M \in \joints(P, Q)}
    \int [f(x) - \gamma c(x, z)] dM(x, z)
    = \int \sup_{x \in X} [f(x) - \gamma c(x, z)] dQ(z)
  \end{equation*}
  as desired, which implies both equality~\eqref{eqn:lagrangian}
  and completes the proof.
\end{proof}
\subsection{Proof of Lemma \ref{lemma:smoothness}}
\label{sec:proof-of-smoothness}

First, note that $z\opt(\theta)$ is unique and well-defined by the strong
convexity of $f(\theta, \cdot)$. For Lipschitzness of $z\opt(\theta)$, we
first argue that $z^{\star}(\theta)$ is continuous in $\theta$. For any
$\theta$, optimality of $z^{\star}(\theta)$ implies that
$\grady(\theta, z^{\star}(\theta))^T(z-z^{\star}(\theta))\le 0$.  By strong
concavity, for any $\theta_1, \theta_2$ and
$z^{\star}_1 = z^{\star}(\theta_1)$ and $z^{\star}_2 = z^{\star}(\theta_2)$,
we have
\begin{equation*}
  \frac{\lambda}{2} \norm{z^{\star}_1 - z^{\star}_2}^2 \le f(\theta_2, z^{\star}_2) - f(\theta_2, z^{\star}_1)
  ~~ \mbox{and} ~~
  f(\theta_2, z^{\star}_2) \le f(\theta_2, z^{\star}_1) + \grady(\theta_2, z^{\star}_1)^T(z^{\star}_2 - z^{\star}_1)
  - \frac{\lambda}{2} \norm{z^{\star}_1 - z^{\star}_2}^2.
\end{equation*}
Summing these inequalities 
gives
\begin{equation*}
  \lambda \norm{z^{\star}_1 - z^{\star}_2}^2
  \le \grady(\theta_2, z^{\star}_1)^T(z^{\star}_2 - z^{\star}_1)
  \le (\grady(\theta_2, z^{\star}_1) - \grady(\theta_1, z^{\star}_1))^T(z^{\star}_2 - z^{\star}_1),
\end{equation*}
where the last inequality follows because $\grady(\theta_1, z^{\star}_1)^T(z^{\star}_2 - z^{\star}_1) \le 0$.
Using a cross-Lipschitz condition from above and Holder's inequality, we obtain
\begin{equation*}
  \lambda \norm{z^{\star}_1 - z^{\star}_2}^2 \le \norm{\grady(\theta_2, z^{\star}_1) - \grady(\theta_1, z^{\star}_1)}_{\star}
  \norm{z^{\star}_1 - z^{\star}_2}
  \le \lipyx \norm{\theta_1 - \theta_2} \norm{z^{\star}_1 - z^{\star}_2},
\end{equation*}
that is,
\begin{equation}
  \label{eqn:y-x-continuous}
  \norm{z^{\star}_1 - z^{\star}_2} \le \frac{\lipyx}{\lambda} \norm{\theta_1 - \theta_2}.
\end{equation}

To see the second inequality, we show that $\bar{f}$ is differentiable with
$\nabla \bar{f}(\theta) = \gradx(\theta, z\opt(\theta))$. By using a variant
of the envelope (or Danskin's) theorem, we first show directional differentiability of
$\bar{f}$. Recall that we say $f$ is \emph{inf-compact} if for all
$\theta_0 \in \Theta$, there exists $\alpha > 0$ and a compact set
$C \subset \Theta$ such that
\begin{equation*}
  \emptyset \neq \left\{z \in \mathcal{Z}: f(\theta, z) \le \alpha \right\}
  \subset C
\end{equation*}
for all $\theta$ in some neighborhood of $\theta_0$~\citep{BonnansSh13}.
See~\citet[Theorem 4.13]{BonnansSh13} for a proof of the following result.
\begin{lemma}
  \label{lemma:envelope}
  Suppose that $f(\cdot, z)$ is
  differentiable in $\theta$ for all $z \in \mathcal{Z}$, and $f$, $\nabla_z f$
  are continuous on $\Theta \times \mathcal{Z}$. If $f$ is inf-compact, then
  $\bar{f}$ is directionally differentiable with
  \begin{equation*}
    \bar{f}'(\theta, d)
    = \sup_{z \in S(\theta)} \nabla_z f(\theta, z) ^\top d
  \end{equation*}
  where $S(\theta) = \argmin_z f(\theta, z)$.
\end{lemma}
\noindent Now, note that from Assumption~\ref{assumption:smoothness}, we have
\begin{equation*}
  |f(\theta, z) - f(\theta_0, z)
  - \nabla_\theta f(\theta_0, z)^\top (\theta - \theta_0)|
  \le L_{\theta\theta} \norm{\theta - \theta_0}
\end{equation*}
from which it is easy to see that $f$ is inf-compact. Applying
Lemma~\ref{lemma:envelope} to $\bar{f}$ and noting that $S(\theta)$ is unique
by strong convexity of $f(\theta, \cdot)$, we have that $\bar{f}$ is
directionally differentiable with
$\nabla \bar{f}(\theta) = \gradx(\theta, z\opt(\theta))$.  Since $\gradx$ is
continuous by Assumption~\ref{assumption:smoothness} and $z\opt(\theta)$ is
Lipschitz~\eqref{eqn:y-x-continuous}, we conclude that $\bar{f}$ is
differentiable.

Finally, we have
\begin{align*}
  \norm{\gradx(\theta_1, z^{\star}_1) - \gradx(\theta_2, z^{\star}_2)}_{\star}
  & \le 
    \norm{\gradx(\theta_1, z_1^{\star}) - \gradx(\theta_1, z_2^{\star})}_{\star}
    + \norm{\gradx(\theta_1, z_2^{\star}) - \gradx(\theta_2, z_2^{\star})}_{\star}\\
  & \le \lipxy \norm{z_1^{\star} - z_2^{\star}}
    + \lipx \norm{\theta_1 - \theta_2} \\
  & \le \left(\lipx + \frac{\lipxy\lipyx}{\lambda}\right)
    \norm{\theta_1 - \theta_2},
\end{align*}
where we have used inequality~\eqref{eqn:y-x-continuous} again.
This is the desired result.
\subsection{Proof of Theorem \ref{thm:convergence}}
\label{sec:proof-of-convergence}

\newcommand{\error}{\xi}

Our proof is based on that of \citet{GhadimiLa13}. For shorthand, let $f(\theta, z; z_0) = \loss(\theta; z) - \gamma c(z, z_0)$,
noting that we perform gradient steps with
\begin{equation*}
  g^t = \nabla_\theta f(\theta^t, \what{z}^t; z^t)
\end{equation*}
for $\what{z}^t$ an $\epsilon$-approximate maximizer of $f(\theta, z; z^t)$
in $z$, and $\theta^{t + 1} = \theta^t - \stepsize_t g^t$. We assume $\stepsize_t \le \frac{1}{L_{\phi}}$ in the rest of the proof, which is satisfied for the constant stepsize $\stepsize=\sqrt{\frac{\Delta_F}{L_{\phi} T\sigma^2}}$ and $T \ge \frac{L_{\phi}\Delta_{F}}{\sigma^2}$.   By a Taylor
expansion using the $L_{\phi}$-smoothness of the objective $F$, we have
\begin{align}
  \nonumber
  F(\theta^{t+1})
  & \le F(\theta^t)
    + \<\nabla F(\theta^t), \theta^{t+1} - \theta^t\>
    + \frac{L_{\phi}}{2} \ltwo{\theta^{t+1} - \theta^t}^2 \\
  & = F(\theta^t)
  - \stepsize_t \ltwo{\nabla F(\theta^t)}^2
  + \frac{L_{\phi} \stepsize_t^2}{2}
  \ltwo{g^t}^2
  + \stepsize_t \<\nabla F(\theta^t),
  \nabla F(\theta^t) - g^t\> \nonumber \\
  & =
  F(\theta^t)
  - \stepsize_t\left(1 -\half L_{\phi} \stepsize_t\right)
  \ltwo{\nabla F(\theta^t)}^2 \label{eqn:intermediate-progress} \\
  & \qquad \quad ~ + \stepsize_t \left(1 - L_{\phi} \stepsize_t\right)
  \<\nabla F(\theta^t), \nabla F(\theta^t) - g^t\>
  + \frac{L_{\phi} \stepsize_t^2}{2}
  \ltwo{g^t - \nabla F(\theta^t)}^2. \nonumber
\end{align}
Recalling the definition~\eqref{eqn:inner-sup} of
$\phi_\gamma(\theta; z_0) = \sup_{z \in \mathcal{Z}} f(\theta, z; z_0)$, we
define the potentially biased errors
$\delta^t = g^t - \nabla_\theta \phi_\gamma(\theta^t; z^t)$.  Substituting the
this into the progress guarantee~\eqref{eqn:intermediate-progress}, we have
\begin{align*}
  F(\theta^{t+1})
  & \le F(\theta^t)
    - \stepsize_t\left(1 -\half L_{\phi} \stepsize_t \right)
    \ltwo{\nabla F(\theta^t)}^2
    + \stepsize_t \left(1 - L_{\phi} \stepsize_t\right)
  \<\nabla F(\theta^t), \nabla F(\theta^t) - \nabla_\theta
  \phi_\gamma(\theta; z^t)\> \\
  & \quad
    - \stepsize_t \left(1 - L_{\phi} \stepsize_t\right)
    \<\nabla F(\theta^t), \delta^t\>
    + \frac{L_{\phi} \stepsize_t^2}{2}
    \ltwo{\nabla_\theta \phi_\gamma(\theta; z^t) + \delta^t
    - \nabla F(\theta^t)}^2 \\
  & = F(\theta^t)
    - \stepsize_t\left(1 -\half L_{\phi} \stepsize_t\right)
    \ltwo{\nabla F(\theta^t)}^2
    + \stepsize_t \left(1 - L_{\phi} \stepsize_t\right)
  \<\nabla F(\theta^t), \nabla F(\theta^t) - \nabla_\theta
  \phi_\gamma(\theta; z^t)\> \\
  & \quad
    - \stepsize_t \left(1 - L_{\phi} \stepsize_t\right)
    \<\nabla F(\theta^t), \delta^t\> \\
  & \qquad
    + \frac{L_{\phi} \stepsize_t^2}{2}
    \left( \ltwo{\delta^t}^2 + 
    \ltwo{\nabla_\theta \phi_\gamma(\theta; z^t) 
    - \nabla F(\theta^t)}^2
    + 2 \<\nabla_\theta \phi_\gamma(\theta; z^t)  - \nabla F(\theta^t),
    \delta^t\>
    \right).
\end{align*}
Using $\pm\<a, b\> \le \half \ltwo{a}^2 + \half \ltwo{b}^2$ in the preceding
display, we get
\begin{align}
  F(\theta^{t+1})
  & \le F(\theta^t)
    - \frac{\stepsize_t}{2}
    \ltwo{\nabla F(\theta^t)}^2
    + \stepsize_t \left(1 - L_{\phi} \stepsize_t\right)
  \<\nabla F(\theta^t), \nabla F(\theta^t) - \nabla_\theta
  \phi_\gamma(\theta; z^t)\> \nonumber \\
  & \quad
    + \frac{\stepsize_t \left(1 + L_{\phi} \stepsize_t\right)}{2}
    \ltwo{\delta^t}^2
    + L_{\phi} \stepsize_t^2
    \ltwo{\nabla_\theta \phi_\gamma(\theta; z^t)
    - \nabla F(\theta^t)}^2
    \label{eq:final-progress}
\end{align}
Letting $z_\star^t = \argmax_z f(\theta^t, z; z^t)$, note that the error
$\delta^t$ satisfies
\begin{align*}
  \ltwo{\delta^t}^2
  = \ltwo{\nabla_\theta \phi_\gamma(\theta^t; z^t)
    - \nabla_\theta f(\theta, \what{z}^t; z^t)}^2
  & = \ltwo{\nabla_\theta \loss(\theta, z_\star^t)
    - \nabla_\theta \loss(\theta, \what{z}^t)}^2 \\
  & \le \lipxy^2 \ltwos{\what{z}^t - z_\star^t}^2
  \le \frac{2\lipxy^2}{\lambda}\epsilon,
\end{align*}
where the final inequality uses the $\lambda = \gamma - \lipy$
strong-concavity of $z \mapsto f(\theta, z; z_0)$.  For shorthand, let
$\what{\epsilon} = \frac{2 \lipxy^2}{\gamma - \lipy} \epsilon$. Taking conditional expectations in the bound~\eqref{eq:final-progress} and
using
$\E[\nabla_\theta \phi_\gamma(\theta^t; z^t) \mid \theta^t] = \nabla
F(\theta^t)$, we have
\begin{align*}
  \label{eqn:expectation-single-step}
  \E[F(\theta^{t + 1}) - F(\theta^t) \mid \theta^t]
  \le & - \frac{\stepsize_t}{2} 
  \ltwo{\nabla F(\theta^t)}^2
  + \frac{\stepsize_t(1+L_{\phi}\stepsize_t)}{2} \what{\epsilon}
  + L_{\phi} \stepsize_t^2 \ltwo{\nabla_\theta \phi_\gamma(\theta; z^t)
    - \nabla F(\theta^t)}^2\\
    & \le - \frac{\stepsize_t}{2} 
  \ltwo{\nabla F(\theta^t)}^2
  + \stepsize_t \what{\epsilon}
  + L_{\phi} \stepsize_t^2 \ltwo{\nabla_\theta \phi_\gamma(\theta; z^t)
    - \nabla F(\theta^t)}^2,
\end{align*}
where we use the fact that $\stepsize_t \le \frac{1}{L_{\phi}}$. For a fixed stepsize $\stepsize$, taking the total expectation yields
\begin{equation*}
  \E \left [  \ltwo{\nabla F(\theta^t)}^2 \right ] - 2\what{\epsilon}
  \le \frac{2}{\stepsize}\E[F(\theta^{t}) - F(\theta^{t+1})] + 2L_{\phi}\stepsize\sigma^2
\end{equation*}
since we have $\E[\ltwo{\nabla \phi_\gamma(\theta; Z) - \nabla F(\theta)}^2] \le \sigma^2$ by assumption. Summing over $t$, we have
\begin{align*}
\frac{1}{T}  \sum_{t=0}^{T-1}\E \left [ \ltwo{\nabla F(\theta^t)}^2 \right ]
-2\what{\epsilon} & \le \frac{2}{\stepsize T}\left (F(\theta^0)- \E[F(\theta^T)] \right) + 2L_{\phi}\stepsize\sigma^2\\
& \le \frac{2\Delta_{F}}{\stepsize T} + 2L_{\phi}\stepsize\sigma^2,
\end{align*}
where the latter inequality holds since $\inf_{\theta} F(\theta) \le F(\theta^{T})$. Plugging in $\stepsize = \sqrt{\frac{\Delta_F}{L_{\phi}T \sigma^2}}$ gives the result. 

\subsection{Proof of Lemma~\ref{thm:nphard}}
\label{section:proof-of-lemma-nphard}

First, we introduce the decision reformulation of the problem: for some
$b$, we ask whether there exists some $u$ such that $\loss(\theta; z + u)
\ge b$. The decision reformulation for an NPO problem is in NP, as a
certificate for the decision problem can be verified in polynomial
time. By appropriate scaling of $\theta$, $v$, and $w$,
\citet{ASkatz2017reluplex} show that 3-SAT Turing-reduces to this decision
problem: given an oracle $D$ for the decision problem, we can solve an
arbitrary instance of 3-SAT with a polynomial number of calls to
$D$. The decision problem is thus NP-complete.

Now, consider an oracle $O$ for the optimization problem. The decision
problem Turing-reduces to the optimization problem, as the decision
problem can be solved with one call to $O$. Thus, the optimization problem
is NP-hard.

\subsection{Proof of Theorem~\ref{theorem:robustness}}
\label{section:proof-robustness}

We first show the bound~\eqref{eqn:robustness-any-rho}. From the duality
result~\eqref{eqn:constrained}, we have the deterministic result that
\begin{align*}
  \sup_{P : W_c(P, Q) \le \rho} \E_Q[\loss(\theta; Z)]
  \le \gamma \rho  + \E_{Q}[\phi_{\gamma}(\theta; Z)]
\end{align*}
for all $\tol > 0$, distributions $Q$, and $\gamma \ge 0$. Next, we show that
$\E_{\emp}[\phi_{\gamma}(\theta; Z)]$ concentrates around its population
counterpart at the usual rate~\citep{BoucheronBoLu05}.

First, we have that
\begin{equation*}
  \phi_\gamma(\theta; z) \in [-\objbd, \objbd],
\end{equation*}
because
$-\objbd \le \loss(\theta; z) \le \phi_\gamma(\theta; z) \le \sup_z
\loss(\theta; z) \le \objbd$.  Thus, the functional
$\theta \mapsto F_n(\theta)$ satisfies bounded
differences~\cite[Thm.~6.2]{BoucheronLuMa13}, and applying standard results on
Rademacher complexity~\citep{BartlettMe02} and entropy
integrals~\cite[Ch.~2.2]{VanDerVaartWe96} gives the result.

To see the second result~\eqref{eqn:robustness-good-rho}, we substitute
$\tol = \what{\tol}_n$ in the bound~\eqref{eqn:robustness-any-rho}. Then, with
probability at least $1-e^{-t}$, we have
\begin{equation*}
  \sup_{P: W_c(P, P_0) \le \what{\tol}_n(\theta)} \E_P[\loss(\theta; Z)]
  \le \gamma \what{\tol}_n(\theta) + \E_{\emp}[\phi_{\gamma}(\theta; Z)]
  + \epsilon_n(t).
\end{equation*}
Since we have 
\begin{align*}
  \sup_{P: W_c(P, \emp) \le \what{\tol}_n(\theta)} \E_P[\loss(\theta; Z)]
  = \E_{\emp}[ \phi_{\gamma}(\theta; Z)] + \gamma \what{\tol}_n(\theta).
\end{align*}
from the strong duality in Proposition~\ref{prop:duality}, our second
result follows.

\subsection{Proof of Corollary~\ref{corollary:lipschitz}}
\label{section:proof-lipschitz}

The result is essentially standard~\citep{VanDerVaartWe96}, which we now give
for completeness. Note that for
$\fclass = \{\loss(\theta; \cdot) : \theta \in \Theta\}$, any
$(\epsilon, \norm{\cdot})$-covering $\{\theta_1, \ldots, \theta_\covnum\}$ of
$\Theta$ guarantees that
$\min_i |\loss(\theta; \statval) - \loss(\theta_i; \statval)| \le L \epsilon$
for all $\theta, \statval$, or
\begin{equation*}
  \covnum(\fclass, \epsilon, \linfstatnorm{\cdot})
  \le \covnum(\Theta, \epsilon / L, \norm{\cdot})
  \le \left(1 + \frac{\diam(\Theta) L}{\epsilon}\right)^d,
\end{equation*}
where
$\diam(\Theta) = \sup_{\theta, \theta' \in \Theta} \norm{\theta - \theta'}$.
Noting that $|\loss(\theta; Z)| \le L \diam(\Theta) + M_0 \eqdef \objbd$, we
have the result.

\subsection{Proof of Theorem~\ref{theorem:dist-concentration}}
\label{sec:proof-of-dist-concentration}

Define
\begin{align*}
  P_n^*(\theta)
  & \defeq \argmax_{P} \left\{ \E_P[\loss(\theta; Z)]
    - \gamma W_c(P, \emp) \right\},  \\
  P^*(\theta)
  & \defeq \argmax_{P} \left\{ \E_P[\loss(\theta; Z)]
    - \gamma W_c(P, P_0) \right\}.
\end{align*}
First, we show that $P^*(\theta)$ and $P_n^*(\theta)$ are attained for all
$\theta \in \Theta$. We omit the dependency on $\theta$ for notational
simplicity and only show the result for $P^*(\theta)$ as the case for
$P_n^*(\theta)$ is symmetric. Let $P^{\epsilon}$ be an $\epsilon$-maximizer,
so that
\begin{align*}
  \E_{P^{\epsilon}}[\loss(\theta; Z)] - \gamma W_c(P^{\epsilon}, P_0)
  & \ge \sup_P \left\{
    \E_{P}[\loss(\theta; Z)] - \gamma W_c(P_n, P_0) \right\} - \epsilon.
\end{align*}
As $\mathcal{Z}$ is compact, the collection $\{ P^{1/k}\}_{k \in \N}$ is a
uniformly tight collection of measures. By Prohorov's theorem~\cite[Ch
  1.1, p. 57]{Billingsley99}, (restricting to a subsequence if necessary),
there exists some distribution $P^*$ on $\mathcal{Z}$ such that
$P^{1/k} \cd P^{*}$ as $k \to \infty$.
Continuity properties of Wasserstein
distances~\citep[Corollary 6.11]{Villani09} then imply that
\begin{equation}
  \label{eqn:wass-conti-in-weak-topology}
  \lim_{k \to \infty}
  W_c(P^{1/k}, P_0) = W_c(P^{*}, P_0).
\end{equation}
Combining~\eqref{eqn:wass-conti-in-weak-topology} and the monotone
convergence theorem, we obtain
\begin{align*}
  \E_{P^{*}}[\loss(\theta; Z)] - \gamma W_c(P^{*}, P_0)
  & = \lim_{k \to \infty}
  \left\{
  \E_{P^{1/k}}[\loss(\theta; Z)] - \gamma W_c(P^{1/k}, P_0)
  \right\}  \\
  & \ge \sup_P \left\{
    \E_{P}[\loss(\theta; Z)] - \gamma W_c(P, P_0) \right\}.
\end{align*}
We conclude that $P^*$ is attained for all $P_0$.

Next, we show the concentration result~\eqref{eqn:dist-concentration}. 
Recall the definition~\eqref{eqn:transport-map} of the
transportation mapping
\begin{equation*}
  T(\theta, z) \defeq \argmax_{z' \in \mathcal{Z}}
  \left\{ \loss(\theta; z') - \gamma c(z', z) \right\},
\end{equation*}
which is unique and well-defined under our strong concavity assumption that
$\gamma > \lipy$, and smooth (recall Eq.~\eqref{eqn:transport-smoothness})
in $\theta$. %
Then by Proposition~\ref{prop:duality} (or by using a variant of Kantorovich
duality~\citep[Chs.~9--10]{Villani09}), we have
\begin{align*}
  \E_{P_n^*(\theta)}[\loss(\theta; Z)
  = \E_{\emp}[\loss(\theta; T(\theta; Z))]
  & ~~\mbox{and}~~
    \E_{P^*(\theta)}[\loss(\theta; Z)
  = \E_{P_0}[\loss(\theta; T(\theta; Z))] \\
  W_c(P_n^*(\theta), \emp) = \E_{\emp}[c(T(\theta; Z), Z)]
  & ~~\mbox{and}~~
  W_c(P^*(\theta), P_0) = \E_{P_0}[c(T(\theta; Z), Z)].
\end{align*}

We now proceed by showing the uniform convergence of
\begin{equation*}
  \E_{\emp}[c(T(\theta; Z), Z)]
  ~~~ \mbox{to} ~~~
  \E_{P_0}[c(T(\theta; Z), Z)]
\end{equation*}
under both cases (i), that $c$ is Lipschitz, and (ii), that
$\loss$ is Lipschitz in $z$,
using a covering argument on $\Theta$.
Recall inequality~\eqref{eqn:transport-smoothness}
(i.e.\ Lemma~\ref{lemma:smoothness}),
which is that
\begin{equation*}
  \norm{T(\theta_1; z) - T(\theta_2; z)}
  \le \frac{\lipyx}{\hinge{\gamma-\lipy}} \norm{\theta_1 - \theta_2}.
\end{equation*}
We have the following lemma.
\begin{lemma}
  \label{lemma:c-is-not-bad}
  Assume the conditions of Theorem~\ref{theorem:dist-concentration}.
  Then for any $\theta_1, \theta_2 \in \Theta$,
  \begin{equation*}
    \left|c(T(\theta_1; z), z)
    - c(T(\theta_2; z), z)\right|
    \le \frac{\lipc \lipyx}{\hinge{\gamma - \lipy}}
    \norm{\theta_1 - \theta_2}.
  \end{equation*}
\end{lemma}
\begin{proof}
  In the first case, that
  $c$ is $\lipc$-Lipschitz in its first argument, this is trivial:
  we have
  \begin{equation*}
    | c(T(\theta_1; z), z) - c(T(\theta_2; z), z) |
    \le \lipc \norm{T(\theta_1; z) - T(\theta_2; z)}
    \le \frac{\lipc \lipyx}{\hinge{\gamma-\lipy}} \norm{\theta_1 - \theta_2}
  \end{equation*}
  by the smoothness inequality~\eqref{eqn:transport-smoothness}
  for $T$.

  In the second case, that $z \mapsto \loss(\theta, z)$ is
  $\lipc$-Lipschitz, let
  $z_i = T(\theta_i; z)$ for shorthand. Then we have
  \begin{align*}
    \gamma c(z_2, z) - \gamma c(z_1, z)
    & = \gamma c(z_2, z) - \loss(\theta_2, z_2)
    + \loss(\theta_2, z_2) - \gamma c(z_1, z) \\
    & \le \gamma c(z_1, z) - \loss(\theta_2, z_1)
    + \loss(\theta_2, z_2) - \gamma c(z_1, z)
    = \loss(\theta_2, z_2) - \loss(\theta_2, z_1),
  \end{align*}
  and similarly,
  \begin{align*}
    \gamma c(z_2, z) - \gamma c(z_1, z)
    & = \gamma c(z_2, z) - \loss(\theta_1, z_1)
    + \loss(\theta_1, z_1) - \gamma c(z_1, z) \\
    & \ge \gamma c(z_2, z) - \loss(\theta_1, z_1)
    + \loss(\theta_1, z_2) - \gamma c(z_2, z)
    = \loss(\theta_1, z_2) - \loss(\theta_1, z_1).
  \end{align*}
  Combining these two inequalities and using that
  \begin{equation*}
    |\loss(\theta, z_2) - \loss(\theta, z_1)|
    \le \gamma \lipc \norm{z_2 - z_1}
  \end{equation*}
  for any $\theta$ gives the result.
\end{proof}

Using Lemma~\ref{lemma:c-is-not-bad} we obtain that
$\theta \mapsto |\E_{\emp}[c(T(\theta; Z), \theta)] - \E_{P_0}[c(T(\theta; Z),
Z)]|$ is $2 \lipc \lipyx / \hinge{\gamma-\lipy}$-Lipschitz.  Let
$\Theta_{\rm cover} = \{ \theta_1, \cdots, \theta_N\}$ be a
$\frac{\hinge{\gamma-\lipy} t}{4\lipc \lipyx}$-cover of $\Theta$ with respect
to $\norm{\cdot}$. From Lipschitzness of
$|\E_{\emp}[c(T(\theta; Z), Z)] - \E_{P_0}[c(T(\theta; Z), Z)]|$, we have that
if for all $\theta \in \{\Theta_{\rm cover}\}$,
\begin{align*}
  |\E_{\emp}[c(T(\theta; Z), Z)] - \E_{P_0}[c(T(\theta; Z), \theta)]|
  \le \frac{t}{2},
\end{align*}
then it follows that
\begin{equation*}
  \sup_{\theta \in \Theta}
  |\E_{\emp}[c(T(\theta; Z), Z)] - \E_{P_0}[c(T(\theta; Z), Z)]|
  \le t.
\end{equation*}

Under the first assumption $(i)$, we have
$|c(T(\theta; Z), Z)| \le 2\lipc \zbd$. Applying Hoeffding's inequality, for
any fixed $\theta \in \Theta$
\begin{equation*}
  \P\left( |\E_{\emp}[c(T(\theta; Z), Z)] - \E_{P_0}[c(T(\theta; Z), Z)]|
    \ge \frac{t}{2} \right)
  \le 2\exp\left(-\frac{nt^2}{32\lipc^2\zbd^2}\right).
\end{equation*}
Taking a union bound over $\theta_1, \cdots, \theta_N$, we conclude that
\begin{equation*}
  \P\left( \sup_{\theta \in \Theta}
    |\E_{\emp}[c(T(\theta; Z), Z)] - \E_{P_0}[c(T(\theta; Z), Z)]| \ge t
  \right)
  \le 2 N \left( \Theta, \frac{\hinge{\gamma-\lipy} t}{4\lipc \lipyx},
    \norm{\cdot} \right)
  \exp\left( - \frac{nt^2}{32\lipc^2\zbd^2} \right)  
\end{equation*}
which was our desired result~\eqref{eqn:dist-concentration}.

Under the second assumption $(ii)$, we
have from the definition of the transport map $T$ 
\begin{equation*}
  \gamma c(T(\theta; z), z)
  \le \loss(\theta; z)
  \le \objbd
\end{equation*}
and hence $|c(T(\theta; Z), Z)| \le \objbd / \gamma$. The result for the
second case follows from an identical reasoning.

\subsection{Proof of Proposition~\ref{proposition:deep-net-smooth}}
\label{section:proof-of-deep-net-smooth}

From the chain rule, we can write down the Jacobian $J_xF_l(\theta; x)$ of
$x \mapsto F_{l}(\theta; x)$ recursively.
\begin{lemma}
  \label{lemma:chain-rule}
  If $\sigma_l$ is differentiable for all $l = 1, \ldots, L$, then
  $x \mapsto F_l(\theta; x)$ is differentiable with
  \begin{align*}
    J_x F_l(\theta; x)
    = \nabla \sigma_l(\theta_l \cdot F_{l-1}(\theta; x)) \cdot \theta_l \cdot
    J_x F_{l-1}(\theta; x)
  \end{align*}
  for all $l = 1, \ldots, L$. Using the convention
  $\prod_{l=1}^L A_l = A_L \cdots A_1$ for matrix products, 
  \begin{align*}
    J_x F_l(\theta; x) = \prod_{l=1}^L \nabla
    \sigma_l(\theta_l \cdot F_{l-1}(\theta; x)) \cdot \theta_l.
  \end{align*}
\end{lemma}

To prove the first result of the proposition, we proceed by induction.  For $l = 1$ and any
$x, x' \in \mathcal{X}$,
\begin{align*}
  \ltwo{F_1(\theta; x) - F_1(\theta; x')}
  & = \ltwo{\sigma_1(\theta_1 \cdot x) - \sigma_1(\theta_1 \cdot x')} \\
  & \le L_1^{0} \opnorm{\theta_1} \ltwo{x-x'}.
\end{align*}
By induction, we conclude that for $l \ge 2$,
\begin{align*}
  \ltwo{F_l(\theta; x) - F_l(\theta; x')}
  & = \ltwo{\sigma_l(\theta_l \cdot F_{l-1}(\theta; x))
  - \sigma_l(\theta_l \cdot F_{l-1}(\theta; x'))} \\
  & \le L_l^0\opnorm{\theta_l}
    \ltwo{F_{l-1}(\theta; x) - F_{l-1}(\theta; x')} \\
  & \le \ltwo{x - x'} \prod_{j=1}^l L_j^0 \opnorm{\theta_j}.
\end{align*}

To show that $JF_l(\theta; \cdot)$ is Lipschitz with respect to the operator
norm, note from Lemma~\ref{lemma:chain-rule} that
\begin{align}
  J_x F_l(\theta; x) -   J_x F_l(\theta; x')
  & = \prod_{j=1}^l \left(
    \nabla \sigma_j(\theta_j \cdot F_{j-1}(\theta; x))
    - \nabla \sigma_j(\theta_j \cdot F_{j-1}(\theta; x'))
    \right)
    \cdot \theta_j \nonumber \\
  & = \sum_{j=1}^l
    \underbrace{
    \left( \prod_{k=j+1}^l \nabla \sigma_k(\theta_k \cdot F_{k-1}(\theta; x'))
    \cdot \theta_k \right)}_{(a)} \nonumber \\
  & \qquad \qquad  \qquad \cdot
    \underbrace{\left(
    \nabla \sigma_j(\theta_j \cdot F_{j-1}(\theta; x)) 
    - \nabla \sigma_j(\theta_j \cdot F_{j-1}(\theta; x')) 
    \right)}_{(b)}  \nonumber \\
  & \qquad \qquad \qquad \qquad \qquad
    \cdot \underbrace{\theta_j \cdot J_x F_{j-1}(\theta; x)}_{(c)}
    \label{eqn:telescope}
\end{align}
where the last equality followed from telescoping summands. Here, we let
$F_0(\theta; x)=x$ notational convenience and define the product $\prod_{a}^b\cdot=1$ when $a>b$.

We now bound the operator norms of the three terms $(a)$, $(b)$, and $(c)$.
From Assumption~\ref{assumption:activation} and the submultiplicativity of the
operator norm, term $(a)$ can be bounded by
\begin{align*}
 \opnorm{\prod_{k=j+1}^l \nabla \sigma_k(\theta_k \cdot F_{k-1}(\theta; x'))
  \cdot \theta_k}
  \le \prod_{k=j+1}^{l} L_k^{0} \opnorm{\theta_k}.
\end{align*}
From Assumption~\ref{assumption:activation}, we can bound term $(b)$ by
\begin{align*}
  \opnorm{\nabla \sigma_j(\theta_j \cdot F_{j-1}(\theta; x)) 
  - \nabla \sigma_j(\theta_j \cdot F_{j-1}(\theta; x')) }
  & \le L_j^1 \ltwo{\theta_j
    \cdot (F_{j-1}(\theta; x) - F_{j-1}(\theta; x'))} \\
  & \le L_j^1 \opnorm{\theta_j}
    \left( \prod_{k=1}^{j-1} L_k^0\opnorm{\theta_k}\right)
    \ltwo{x-x'}
\end{align*}
where the last inequality follows from the first part of the proof.
Lastly, term $(c)$ is bounded by from the first part of the proof
\begin{align*}
  \opnorm{\theta_j \cdot J_x F_{j-1}(\theta; x)}
  \le \opnorm{\theta_j} \prod_{k=1}^{j-1} L_k^0 \opnorm{\theta_k}.
\end{align*}

Collecting these bounds and applying them in the
representation~\eqref{eqn:telescope}, triangle inequality and
submultiplicativity of the operator norm yields
\begin{align*}
  \opnorm{J_x F_l(\theta; x)- J_x F_l(\theta; x')}
  & \le \sum_{j=1}^l \left( \prod_{k=j+1}^{l} L_k^{0} \opnorm{\theta_k}\right)
  \cdot L_j^1 \cdot \opnorm{\theta_j} \cdot
  \left( \prod_{k=1}^{j-1} L_k^{0} \opnorm{\theta_k}\right) \\
  & \qquad \qquad \qquad \cdot \ltwo{x-x'} \cdot \opnorm{\theta_j}
    \cdot \left( \prod_{k=1}^{j-1} L_k^0 \opnorm{\theta_k} \right) \\
  & = \beta_l(\theta) \ltwo{x-x'}.
\end{align*}

\subsection{Proof of Corollary~\ref{corollary:softmax-smooth}}
\label{section:proof-of-softmax-smooth}
Denote the softmax function $\sigma_y: \R^\numclass \to \R$
\begin{equation*}
  \sigma_y(x) \defeq - \log p_y(x)
  = - \log \frac{\exp(x_y)}{\sum_{j=1}^\numclass \exp(x_j)}.
\end{equation*}
\noindent We consider calculating $L^0$ and $L^1$ for this map. Note that 
\begin{align*}
\nabla\sigma_y(x) = p(x)-e_y, \nabla^2\sigma_y(x) = \diag{p(x)} - p(x)p(x)^T,
\end{align*}
where $p(x)=\sum_{k=1}^Kp_k(x)e_k$. Then,
\begin{equation*}
L^0 = \sup_{x}\ltwo{\nabla\sigma_y(x)} = \sup_x \ltwo{p(x) - e_y}= \sqrt{2},
\end{equation*}
where the last equality results from the fact that $g(z):=\ltwo{z-e_y}$ is convex in $z$ and therefore $\sup_{z \in \mathcal{Z}} g(z)$ is attained at a corner of the probability simplex $\mathcal{Z}:=\{z | z \ge 0, \sum_i z_i=1\}$, namely any corner other than $e_y$. 

The eigenvalues of $\diag{p(x)}$ are $p_i(x)$ and therefore satisfy $\lambda(\diag{p(x)}) \in [0,1]$. The eigenvalues of $p(x)p(x)^T$ satisfy $\lambda(p(x)p(x)^T) \in \{0, \ltwo{p(x)}^2\}$. Then, by Weyl's inequality, we have $\lambda(\nabla^2\sigma_y(x)) \in [-\ltwo{p(x)}^2, 1]$. Again using convexity of $\ltwo{\cdot}$, we have $\sup_x\ltwo{p(x)}=1$, as the supremum is attained at a corner of the probability simplex. Then $\lambda(\nabla^2\sigma_y(x)) \in [-1,1]$, whereby $L^1=1$.

An equivalent recursive way of expressing the Lipschitz constants $\alpha_l(\theta)$ and $\beta_l(\theta)$~\eqref{eq:lipschitzconsts} is
 \begin{align*}
 \alpha_{l+1}(\theta) &= L_{l+1}^0 \opnorm{\theta_{l+1}}\alpha_l(\theta)\\
   \beta_{l+1}(\theta) &= \frac{\alpha_{l+1}(\theta)}{\alpha_l(\theta)}\beta_l(\theta)
                         + \frac{L_{l+1}^1}{(L_{l+1}^0)^2}\alpha_{l+1}(\theta)^2,
 \end{align*}
 with $\beta_0(\theta)=\alpha_0(\theta)=1$. Then, considering an $(L+1)$-layer network where $\theta_{L+1}=I$ and $\sigma_{L+1}$ is the softmax function, we have 
\begin{align*}
\beta_{L+1}(\theta) &= \sqrt{2}\beta_L(\theta) + \alpha_L(\theta)^2.
\end{align*}

\section{Proximal algorithm for $\linf{\cdot}$-norm robustness}
\label{section:prox}

\newcommand{\prox}[1]{\mbox{prox}_{#1}}

In this section, we give a efficient training algorithm that learns to defend
against $\linf{\cdot}$-norm perturbations. For simplicity, we assume
$\mathcal{Z} = \R^m$ for the rest of this section. Let $\theta \in \Theta$ be some
fixed model, $z^0 \in \mathcal{Z}$ a natural example\footnote{We depart from our convention of denoting original datapoints as $z_0$ to ease forthcoming notation.} and define
$f(z) \defeq \loss(\theta; z)$ to ease notation. Concretely, we are interested
in solving the optimization problem
\begin{equation*}
  \maximize_{z} f(z) - \frac{\alpha}{2} \linf{z - z^0}^2
\end{equation*}
Note that this is equivalent to computing the surrogate loss
$\phi_{\gamma}(\theta; z^0) = \sup_{z \in \mathcal{Z}} \{ \loss(\theta; z) - \gamma
c(z, z^0) \}$ for $c(z, z^0)=\linf{z-z^0}^2$ and $\alpha=2\gamma$. Our following
treatment can easily be modified for the supervised learning scenario
$c((x, y), (x^0, y^0)) = \linf{x - x^0}^2 + \infty \cdot \indic{y =
  y^0}$ with the convention that $\infty \cdot 0 = 0$.  To make our notation
consistent with the optimization literature, we consider the minimization
problem
\begin{equation}
  \label{eqn:opt-problem}
  \minimize_{z} -f(z) + \frac{\alpha}{2} \linf{z - z^0}^2.
\end{equation}

A simple gradient descent algorithm applied to the
problem~\eqref{eqn:opt-problem} may be slow to converge in
practice. Intuitively, this is because the subgradient of
$z \mapsto \half \linf{z - z^0}^2$ is given by $\linf{z - z^0} \cdot s$ where
$s$ is a $m$-dimensional vector taking values in $[-1, 1]$ whose coordinates
are non-zero only when $|z_j - z_{0, j}| = \linf{z - z^0}$. Hence, at any
given iteration of gradient descent, the $\linf{\cdot}$-norm penalty term only
gets accounted for by at most a few coordinates.

To remedy this issue, we consider a proximal algorithm for solving the
problem~\eqref{eqn:opt-problem} (see, for example,~\citet{ParikhBo13} for an
comprehensive review of proximal algorithms). For a function
$g: \mathcal{Z} \to \R$ and a positive number $\lambda > 0$, the proximal operator
for $\lambda g$ is defined by
\begin{equation*}
  \prox{\lambda g}(v)
  \defeq \argmin_{z} \left\{ g(z) + \frac{1}{2\lambda} \ltwo{z - v}^2 \right\}.
\end{equation*}
Then, the proximal algorithm on the problem~\eqref{eqn:opt-problem} consists
of two steps at each iteration $t$: $(i)$ for the smooth function $-f(z)$,
take a gradient descent step at the current iterate $z^t$ ($z^{t+\half}$
below) and $(ii)$ for the non-smooth function $\linf{z - z^0}^2$, take a
proximal step for the function
$\frac{\lambda^t \alpha}{2} \linf{\cdot - z^0}^2$ at $z^{t+\half}$ ($z^{t+1}$
below):
\begin{equation}
  \label{eqn:prox}
  z^{t+\half} = z^t + \lambda^t \nabla f(z^t),
  ~~~~~~
  z^{t+1} = \prox{\frac{\lambda^t \alpha}{2} \linf{\cdot -
      z^0}^2}\left(z^{t+\half}\right).
\end{equation}
The following proposition shows that we can compute the proximal step
$z^{t+1}$ efficiently, simply by sorting the vector $|z^{t+\half} - z^0|$.  We
denote by $v^t$, the sorted vector of $|z^{t+\half} - z^0|$ in
\textbf{decreasing} order. In the proposition, we use the notation
$\hinge{\cdot} = \max(\cdot, 0)$.
\begin{proposition}
  \label{prop:prox}
  Define the scale parameter $\beta^t > 0$ by
  \begin{equation}
    \label{eqn:prox-scale}
    \beta^t \defeq \frac{1}{1 + \alpha \lambda^t j^t} \sum_{i=1}^{j^t} v^t_i
    ~~\mbox{where}~~
    j^t \defeq \max \left\{
      j \in [m]: \sum_{i = 1}^{j-1} v_i
      - \left(\frac{1}{\alpha \lambda^t} + (j-1)\right) v_j
      < 0
    \right\}.
  \end{equation}
  Then, $z^{t+1}$ in the proximal update~\eqref{eqn:prox} is given by
  \begin{equation}
    \label{eqn:prox-actual}
    z^{t+1} = z^{t+\half} - \hinge{|z^{t+\half} - z^0| - \beta^t}
    \sign\left(z^{t+\half} - z^0\right).
  \end{equation}
\end{proposition}
\noindent See Section~\ref{section:proof-of-prox} for the proof of the
proposition. From the proposition, we obtain the proximal procedure in
Algorithm~\ref{alg:prox} that can be used to (heuristically) solve for the
approximate maximizer of $\loss(\theta;z) - \gamma c(z, z^0)$ in
Algorithm~\ref{alg:thealg}. Roughly speaking, ignoring the truncation term in
the proximal update~\eqref{eqn:prox-actual}, we have
\begin{equation*}
  z^{t+1} \approx
  z^0 + \beta^t \sign(z^t + \lambda^t \nabla f(z^t) - z^0).
\end{equation*}
Here, we move towards the sign of $z^t + \lambda^t \nabla f(z^t) - z^0$
modulated by the term $\beta^t$, as opposed to just the sign of
$\nabla f(z^t)$ for the iterated fast sign gradient
method~\citep{ASgoodfellow2015explaining,ASkurakin2016adversarial}.

\begin{algorithm}[t]
  \caption{\label{alg:prox} Proximal Algorithm for Maximizing
    $f(z) - \frac{\alpha}{2} \linf{z-z^0}^2$}
  \begin{algorithmic}[]
    \State \textsc{Input:} Stepsizes $\lambda^t$
    \State \textbf{for} {$t=0, \ldots, T-1$} \textbf{do}
    \State~~~ $z^{t+\half} \gets z^t + \lambda^t \nabla f(z^t)$
    \State~~~ $v^t \gets \mbox{sort}(|z^{t+\half} - z^0|, \mbox{dec})$
    \State~~~ Compute $\beta^t$ as in~\eqref{eqn:prox-scale}
    \State~~~ $z^{t+1} \gets z^{t+\half} - \hinge{|z^{t+\half} - z^0| - \beta^t}
    \sign\left(z^{t+\half} - z^0\right)$
  \end{algorithmic}
\end{algorithm}

\subsection{Proof of Proposition~\ref{prop:prox}}
\label{section:proof-of-prox}

In this proof, we drop the subscript on the iteration $t$ to ease notation. We
assume without loss of generality that $z^{t+\half} - z^0 \neq 0$. For some
convex, lower semi-continuous function $g: \R^m \to \R$, let
$g^*(s) = \sup_{s} \{ s^\top t - g(t)\}$ be the Fenchel conjuagte of $g$.
From the Moreau decomposition~\citep[Section 2.5]{ParikhBo13}, we have
\begin{equation*}
  \prox{g}(w) + \prox{g^*}(w) = w
\end{equation*}
for any $w \in \R^m$. Noting that the conjugate of
$z \mapsto \frac{\alpha \lambda}{2} \linf{z - z_0}^2$ is given by
$z \mapsto z^\top z_0 + \frac{1}{2 \alpha \lambda} \lone{z}^2$, we have
\begin{equation*}
  \prox{\frac{\alpha \lambda}{2} \linf{\cdot - z_0}^2}(w)
  = w -
  \prox{\< z^{0}, \cdot\> + \frac{1}{2 \alpha \lambda} \lone{\cdot}^2}(w)
    = w -
  \prox{\frac{1}{2 \alpha \lambda} \lone{\cdot}^2}(w - z^0)
\end{equation*}
Let us denote the sorted vector (in decreasing order) of $|w - z^0|$ by
$v$. Then, in light of the preceeding display, it suffices to show that
\begin{equation}
  \label{eqn:prox-l1-norm-squared}
  \prox{\frac{1}{2 \alpha \lambda} \lone{\cdot}^2}(w - z^0)
  = \hinge{|w - z| - \beta\opt} \sign \left(w - z^0\right)
\end{equation}
where $\beta\opt$ is defined as in~\eqref{eqn:prox-scale}.  To show that
equality~\eqref{eqn:prox-l1-norm-squared} holds, note that the first order
optimality conditions for
\begin{equation*}
  \prox{\frac{1}{2\alpha \lambda} \lone{\cdot}^2}(w - z^0)
  = \argmin_{z} \left\{
    \frac{1}{2} \lone{z}^2
    +  \frac{\alpha \lambda}{2} \ltwo{z - (w-z^0)}^2
  \right\}
\end{equation*}
is given by
\begin{subequations}
  \label{eqn:kkt}
  \begin{align}
    \lone{z} \sign(z_i) + \alpha \lambda (z_i - w_i + z_i^0) = 0
    ~~~ & \mbox{if} ~~ |z_i| \neq 0 \\
    \lone{z} [-1, 1] - \alpha \lambda (w_i - z_i^0) \ni 0
    ~~~ & \mbox{if} ~~ |z_i| = 0.
  \end{align}
\end{subequations}
Now, we use the following elementary lemma.
\begin{lemma}
  \label{lemma:prox}
  For $0 \neq v \ge 0$ with decreasing coordinates, the solution to the
  equation
  \begin{equation*}
    \sum_{i: v_i > \beta} (v_i - \beta) = \alpha \lambda \beta
  \end{equation*}
  exists and is given by
  \begin{equation*}
    \beta\opt \defeq \frac{1}{1 + \alpha \lambda j\opt} \sum_{i=1}^{j} v_i
    ~~\mbox{where}~~
    j\opt \defeq \max \left\{
      j \in [m]: \sum_{i = 1}^{j-1} v_i
      - \left(\frac{1}{\alpha \lambda} + (j-1)\right) v_j < 0
    \right\}.
  \end{equation*}

\end{lemma}
\begin{proof-of-lemma}
  First, note that
  $\beta \mapsto \sum_{i: v_i > \beta} (v_i - \beta) - \alpha \lambda \beta
  \eqdef h(\beta)$ is decreasing. Noting that $\lone{v} > 0$ and
  $- \alpha \lambda \linf{v} < 0$, there exists $\beta'$ such that
  $h(\beta') = 0$ and $\beta' \in (0, \linf{v})$. Since $v_i$'s are decreasing
  and nonnegative, there exists $j'$ such that $v_{j'} > \beta' \ge v_{j'+1}$
  (we abuse notation and let $v_{m+1} \defeq 0$). Then, we have
  \begin{equation*}
    \sum_{i=1}^{j'-1} (v_i - v_{j'}) - \alpha \lambda v_{j'} < 0
    \le  \sum_{i=1}^{j'} (v_i - v_{j'+1}) - \alpha \lambda v_{j'+1}.
  \end{equation*}
  That is, $j' = j\opt$. Solving for $\beta'$ in
  \begin{equation*}
    0 = h(\beta') = \sum_{i=1}^{j\opt} v_i
    - \left( \alpha\lambda + j\opt\right) \beta',
  \end{equation*}
  we obtain $\beta' = \beta\opt$ as claimed.
\end{proof-of-lemma}

Now, define
\begin{equation*}
  z\opt = \hinge{|w - z^0| - \beta\opt} \sign\left(w - z^0\right).
\end{equation*}
Then, we have from Lemma~\ref{lemma:prox} that
\begin{equation*}
  \lone{z\opt} = \sum_{i: |w_i - z^0_i| > \beta\opt} (|w_i - z^0_i| - \beta\opt)
  = \sum_{i=1}^{j\opt} (v_i - \beta\opt) = \alpha \lambda \beta\opt.
\end{equation*}
If $z\opt_i > 0$, then $\sign(z\opt_i) = \sign(w_i - z^0_i)$ so that
\begin{equation*}
  \lone{z\opt} \sign(z_i) + \alpha \lambda (z\opt_i - w_i + z^0_i)
  = 0.
\end{equation*}
If $z\opt_i = 0$, then $|w_i - z_i^0| \le \beta\opt$ and
\begin{equation*}
  \lone{z\opt} [-1, 1] - \alpha \lambda (w_i - z_i^0)
  = \alpha \lambda \beta\opt [-1, 1] - \alpha \lambda (w_i - z_i^0)
  \ni 0.
\end{equation*}
Hence, $z\opt$ satisfies the optimality condition~\eqref{eqn:kkt} as desired.

\fi

\end{document}